\documentclass[twoside,11pt]{article}
\usepackage{blindtext}

\usepackage{jmlr2e}

\usepackage{zy_notations}
\usepackage{hhline}
\usepackage{makecell}

\usepackage{lastpage}
\ShortHeadings{Differentially Private Bootstrap}{Wang, Cheng, and Awan}
\firstpageno{1}

\begin{document}

\title{Differentially Private Bootstrap: \\New Privacy Analysis and Inference Strategies}

\author{\name Zhanyu Wang \email zhanyu.wang.purdue@gmail.com \\
       \addr Department of Statistics \\
       Purdue University\\
       West Lafayette, IN 47906, USA
       \AND
       \name Guang Cheng \email guangcheng@ucla.edu \\
       \addr Department of Statistics\\
       University of California, Los Angeles \\
       Los Angeles, CA 90095, USA
       \AND
       \name Jordan Awan \email jaa557@pitt.edu \\
       \addr   Department of Statistics\\
       Purdue University\\
       West Lafayette, IN 47906, USA\\
       and\\
       Department of Statistics \\
       University of Pittsburgh\\
       Pittsburgh, PA 15260, USA
       }

\editor{}

\maketitle

\begin{abstract}
Differentially private (DP) mechanisms protect individual-level information by introducing randomness into the statistical analysis procedure. Despite the availability of numerous DP tools, there remains a lack of general techniques for conducting statistical inference under DP. We examine a DP bootstrap procedure that releases multiple private bootstrap estimates to infer the sampling distribution and construct confidence intervals (CIs). Our privacy analysis presents new results on the privacy cost of a single DP bootstrap estimate, applicable to any DP mechanism, and identifies some misapplications of the bootstrap in the existing literature. {For the composition of the DP bootstrap, we present a numerical method to compute the exact privacy cost of releasing multiple DP bootstrap estimates, and} using the Gaussian-DP (GDP) framework \citep{dong2021gaussian}, we show that the release of $B$ DP bootstrap estimates from mechanisms satisfying $(\mu/\sqrt{(2-2/\mathrm{e})B})$-GDP asymptotically satisfies $\mu$-GDP as $B$ goes to infinity. {Then, we perform private statistical inference by post-processing the DP bootstrap estimates. We prove that our point estimates are consistent, our standard CIs are asymptotically valid, and both enjoy optimal convergence rates. To further improve the finite performance,} we use deconvolution with DP bootstrap estimates to accurately infer the sampling distribution. We derive CIs for tasks such as population mean estimation, logistic regression, and quantile regression, and we compare them to existing methods using simulations and real-world experiments on 2016 Canada Census data. Our private CIs achieve the nominal coverage level and offer the first approach to private inference for quantile regression. % \footnote{Code is available at \url{https://github.com/Zhanyu-Wang/Differentially_Private_Bootstrap}.}
\end{abstract}


\section{Introduction}
In the big data era, individual privacy protection becomes more critical than ever because personal information is collected and used in many different ways;
while the intention of the data collection is usually to improve the user experience or, more generally, to benefit society, there have also been rising concerns about malicious applications of these data.
To protect individuals against arbitrary attacks on their data, \citet{dwork2006calibrating} proposed \textit{differential privacy} (DP) which has become the state-of-the-art framework in privacy protection. 

DP is a probabilistic framework that measures the level of privacy protection of a mechanism. A mechanism is a \textit{randomized algorithm}, i.e., its output is a realization of a random variable following a distribution determined by the mechanism and its input. A mechanism satisfies DP if changing any individual in the input results in an output with a distribution similar to the original output distribution. Starting from the definition of $(\ep,\delta)$-DP \citep{dwork2010boosting},
% \footnote{For easy reference, we include the definition of $(\ep,\delta)$-DP in Definition \ref{def:DP} in Appendix \ref{sec:append_background}.}, 
there have been many variants of DP definitions serving different needs. For our results, we use $f$-DP \citep{dong2021gaussian}, a hypothesis-testing perspective of DP formally defined in Definitions \ref{def:tradeoff} and \ref{def:fdp}, as it is the most informative DP notion satisfying the post-processing property \citep[Theorem 2]{dong2021gaussian}.
% DP mechanisms protect the sensitive information in the input by introducing additional randomness into the mechanism.} 
For statistical analysis under DP, a great deal of prior work focused on producing private point estimates of a parameter, e.g., the sample mean, sample median, and the maximum of the data. In contrast, while some prior work aims to quantify the uncertainty of a DP procedure, their techniques are usually restricted to specific settings, and there is still a lack of general-purpose methods (see related work for some notable exceptions.)

One of the most widely used methods to approximate a sampling distribution is the bootstrap method \citep{efron1979bootstrap}, which can be used to quantify the uncertainty of an estimator in many statistical ways, such as by producing a non-parametric confidence interval (CI).
Although the bootstrap has been studied very well in Statistics, it is still an open question of how to build and analyze a DP bootstrap for private statistical inference. \citet{brawner2018bootstrap} were the first to propose and analyze a DP bootstrap procedure and used it to produce a CI based on the private bootstrap outputs. However, a key step in their privacy proof is incorrect\footnote{They use a subsampling result for zCDP \citep{bun2016concentrated} in Section 6.2, while there is no such result for zCDP as  demonstrated on page 75 in \citet{bun2018composable}.}, and we show in Section \ref{sec:privacy_analysis} that their stated privacy guarantee does not hold. Furthermore, {there is no theoretical analysis on the coverage and width of their CI.} \citet{balle2018privacy} developed the state-of-the-art analysis of resampling for $(\ep,\delta)$-DP, both with and without replacement. As a particular case, their results can be used to analyze the privacy guarantees of mechanisms with their input being bootstrap samples. However, \citet{balle2018privacy} did not consider the cumulative privacy cost of multiple samples of the resampling methods, which restricts the usage of their results on the bootstrap. 
Moreover, we show that it is necessary to develop a new method for conducting statistical inference with the DP bootstrap samples while \citet{balle2018privacy} only focused on the privacy analysis.

\paragraph{Our contributions} In this paper, we obtain a tight privacy analysis of a DP bootstrap method and develop inference strategies on the sampling distribution. Specifically,
we derive the privacy guarantee of the DP bootstrap, which generalizes the result by \citet{balle2018privacy} from $(\ep,\delta)$-DP to $f$-DP, and an aspect of our proof strategy applies to any mixture of DP mechanisms where the bootstrap is a special case. Our result also identifies misuses of resampling with replacement in the literature. 
{For the cumulative privacy cost of releasing multiple DP bootstrap estimates, we present a numerical method for the finite composition and} derive the asymptotic composition result via the central limit theorem (CLT) for $f$-DP \citep{dong2021gaussian}. 
Then, we perform private statistical inference by post-processing the DP bootstrap estimates. While our privacy analysis of the DP bootstrap is applicable to arbitrary mechanisms, our statistical inference results are focused on the Gaussian mechanism. We prove that our point estimates are consistent, our DP CIs are asymptotically valid, and both enjoy optimal convergence rates. To improve finite-sample performance,
we use deconvolution on the DP bootstrap results to obtain a private estimate of the non-private sampling distribution and develop CIs. {Our simulations show the advantage of our deconvolution method in terms of the coverage of the CIs compared to \citet{du2020differentially}, and the CI width compared to \citet{brawner2018bootstrap}.} We also conduct numerical experiments on the 2016 Canada Census Public Use Microdata, which reveals the dependence between individuals' income and shelter cost under DP guarantees by building CIs for the slope parameters of logistic regression and quantile regression. To the best of our knowledge, DP bootstrap is the first tool that can be used to perform valid private statistical inference for parameters in quantile regression. 
% difference between two provinces and a positive correlation between income and shelter cost under the DP guarantee.

\paragraph{Related work}

For DP CIs, 
\citet{dorazio2015differential} presented algorithms for releasing DP estimates of causal effects, particularly the difference between means estimators along with their standard errors and CIs. 
\citet{sheffet2017differentially} presented a DP estimator for the $t$-values in ordinary least squares and derived the CI based on the $t$-values. 
\citet{karwa2017finite} gave DP CIs for the population mean of a normal distribution along with finite sample coverage and lower bounds on the size of the DP CI. 
\citet{brawner2018bootstrap} gave the first attempt to use the DP bootstrap to obtain a private CI. 
\citet{wang2019differentially} developed algorithms for generating DP CIs for DP estimates from objective and output perturbation mechanisms in the empirical risk minimization framework.
{\citet{awan2020differentially} developed DP uniformly most powerful hypothesis tests and DP CIs for Bernoulli data.}
\citet{du2020differentially} proposed and compared different methods (including NoisyVar, which we discuss in our experimental sections) to build DP CIs for the mean of normally distributed data. 
\citet{covington2025unbiased} used the bag of little bootstraps (BLB) and the CoinPress algorithm \citep{biswas2020coinpress} to privately estimate the parameters' sampling distribution {through its mean and covariance}. 
 \citet{chadha2024resampling} also used BLB to produce private confidence sets with consistency results and asymptotic rate analysis.
\citet{drechsler2022nonparametric} proposed several strategies to compute non-parametric DP CIs for the median. 
\citet{awan2024simulation} used a simulation-based method to produce finite-sample CIs and hypothesis tests from DP summary statistics in parametric models.

Other literature involving the idea of bootstrap in DP includes \citet{ferrando2022parametric} and \citet{o2019bootstrap}. \citet{ferrando2022parametric} used parametric bootstrap through resampling data from a distribution parametrized by estimated parameters, while we use non-parametric bootstrap through resampling data from the empirical data distribution. \citet{o2019bootstrap} proposed a relaxed definition of differential privacy based on bootstrap, but they did not use bootstrap to perform statistical inference as we do.

\paragraph{Organization}
The remainder of this paper is organized as follows. 
In Section \ref{sec:background}, we review the definition of $f$-DP and results used in our DP bootstrap analysis.
In Section \ref{sec:privacy_analysis}, we provide our privacy guarantee for DP bootstrap along with {a numerical method and} a central limit theorem result for the cumulative privacy cost of multiple DP bootstrap outputs. 
In Section \ref{sec:inference}, we propose two methods for performing non-parametric statistical inference {using the DP bootstrap: one uses the asymptotic distributions of two point estimates, and the other uses deconvolution.
In Section \ref{sec:simulation}, we use simulation to show that our CIs have better coverage than NoisyVar \citep{du2020differentially}, and better width than the method by \citet{brawner2018bootstrap}.} 
In Section \ref{sec:realworld}, we analyze the dependence between market income and shelter cost in Ontario by building DP CIs for the slope parameters of logistic regression and quantile regression with the 2016 Canadian census dataset. We compare DP bootstrap with DP-CI-ERM \citep{wang2019differentially} only in logistic regression, {as the DP-CI-ERM method is inapplicable to quantile regression.}
In Section \ref{sec:conclusion}, we discuss the implications of our work and highlight some directions for future research. Proofs and technical details are deferred to the appendix and supplementary materials.

\section{Background in differential privacy}\label{sec:background}
In this section, we provide existing results related to our DP bootstrap method. First, we discuss the definitions of $f$-DP and $(\ep,\delta)$-DP, and introduce the Gaussian mechanism to guarantee DP. Then we show previous $f$-DP results on subsampling, composition, and group privacy. {Finally, we provide the connection between $(\ep,\delta)$-DP and $f$-DP.}

For a set $\cX$, a dataset $D\in \cX^n$ with cardinality $|D|=n$ is a finite collection of elements from $\cX$.
We define $D_1\simeq_k D_2$ if $|D_1|=|D_2|$ and they differ in $k$ entries. We call $D_1$ and $D_2$ as \textit{neighboring datasets} if $D_1\simeq_1 D_2$, and we also write it as $D_1\simeq D_2$.
A \textit{mechanism} $\cM:\cX^n\rightarrow \cY$ is a randomized algorithm taking a dataset as input and outputting a value in $\cY$.
\textit{Differential privacy} measures how much the outputs of the mechanism differ when the inputs are two neighboring datasets. 
% {We write $f\circ g(X):=f(g(X))$ for any $f$ and $g$.} 

In this paper, we use the $f$-DP framework, which measures the difference between two distributions by hypothesis testing.
With an observation $\cM(D_{\text{in}})=X$, we consider a hypothesis test between $H_0: X\sim\cM(D)$ and $H_1: X\sim\cM(D')$.
Intuitively, the harder this test is, the stronger the privacy guarantee that the mechanism {$\cM$} provides.  \citet{wasserman2010statistical} first used this hypothesis testing framework in DP, and \citet{dong2021gaussian} generalized it to the $f$-DP framework using the \textit{tradeoff function} which characterizes the difficulty of this test (also known as the receiver operating characteristic (ROC) curve.)

\begin{definition}[tradeoff function \citep{dong2021gaussian}]\label{def:tradeoff}
Consider the hypothesis test $H_0:X\sim P$ versus $H_1:X\sim Q$. For any rejection rule $\phi(X)$, we use $\alpha_\phi$ to denote the type I error and $\beta_{\phi}$ to denote the type II error. The tradeoff function $T_{P,Q}(\alpha): [0,1]\rightarrow [0,1]$ is defined to be
$T_{P,Q}(\alpha) := \inf_{\phi}\{\beta_\phi \;|\;  \alpha_{\phi} \leq \alpha \}$.
% , where the infimum is over all possible rejection rules $\phi$. 
For a tradeoff function $f$, we denote its inverse by $f^{-1}(x):=\inf\{\alpha\in[0,1]: f(\alpha)\leq x\}$. We define that $f$ is symmetric if $f=f^{-1}$.
\end{definition}
        
\begin{figure}[t]
        \centering
        \includegraphics[width=0.5\textwidth]{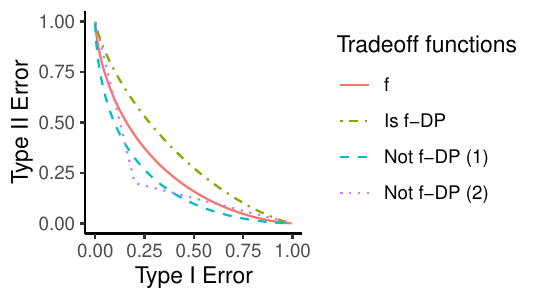}
        \captionof{figure}{An illustration of tradeoff functions and $f$-DP \citep{dong2021gaussian}.}\label{fig:not-fdp}
\end{figure}

Lower tradeoff functions indicate less privacy since an adversary can distinguish one distribution from the other with smaller type II error at a given type I error. The upper bound of the tradeoff function is $T(\alpha)=1-\alpha$, since it corresponds to the family of random rejection rules which reject $H_0$ with probability $\alpha$ for any observation $X$.

\begin{definition}[$f$-DP \citep{dong2021gaussian}]\label{def:fdp}
Let $f$ be a tradeoff function. We write \\
$T_{\cM(D),\cM(D')} \geq f$ if $T_{\cM(D),\cM(D')}(\alpha) \geq f(\alpha)~\forall \alpha\in[0,1]$. A mechanism $M$ is said to be $f$-differentially private {($f$-DP)} if $T_{\cM(D),\cM(D')} \geq f$ for any datasets $D,D'$ with $D\simeq D'$.
\end{definition}

Intuitively, for a mechanism $\cM$ satisfying $f$-DP where $f=T_{P,Q}$, testing $H_0: X\sim \cM(D)$ versus $H_1: X\sim \cM(D')$ is at least as hard as testing $H_0:X\sim P$ versus $H_1:X\sim Q$ when $D\simeq D'$. 
We visualize the definition of $f$-DP in Figure \ref{fig:not-fdp}. Among all tradeoff functions, an important subclass is $G_\mu(\alpha)=T_{\cN(0,1), \cN(\mu,1)}(\alpha)$. $G_\mu$-DP is also called $\mu$-Gaussian DP (GDP), which is shown to be the limit of many DP procedures under composition by \citet{dong2021gaussian}. Another important subclass is $f_{\ep,\delta}(\alpha) = \max\{0, 1-\delta-\ee^\ep\alpha, \ee^{-\ep}(1-\delta-\alpha)\}$ since $f_{\ep,\delta}$-DP is equivalent to $(\ep,\delta)$-DP. 
{\begin{definition}[$(\ep,\delta)$-DP: \citealp{dwork2006calibrating}]\label{def:DP}
	A mechanism  $\cM:\cX^n \rightarrow \cY $ is $(\ep, \delta)$-differentially private {($(\ep,\delta)$-DP)} if for any neighboring datasets $D, D'\in \cX^n$, and every {measurable} set $S \subseteq \cY$, 
	the following inequality holds:
	$ \pr[\cM(D)\in S] \leq \ee^{\ep}\pr[\cM(D')\in S] + \delta. $
The \textit{privacy profile} \citep{balle2018privacy} of $\cM$ is a map $\delta(\ep)$ which represents the smallest $\delta$ such that $\cM$ is $(\ep,\delta)$-DP.
\end{definition}
}

\paragraph{The Gaussian mechanism}
Let the $\ell_2$-\textit{sensitivity} of a function  $g:\cX^n \rightarrow \RR^d$ be $\ds\Delta(g)=\sup\limits_{D_1\simeq D_2}\|g(D_1)-g(D_2)\|_2$.
For any $g:\cX^n \rightarrow \RR^d$, the \textit{Gaussian mechanism} on $g$ adds Gaussian noises to the output of $g$: $\cM_{\text{G}}(D, g, \sigma)=g(D)+\xi$ where $\xi\sim\cN(\mu=0,\Sigma=\sigma^2 I_{d\times d})$.
\citet{dong2021gaussian} proved that $\cM_{\text{G}}(D, g, \sigma)$ satisfies $\mu$-GDP if $\sigma^2={(\Delta(g) / \mu)^2}$.

% \paragraph{} 
\begin{restatable}[Subsampling \citep{dong2021gaussian}]{proposition}{propfdpsubsamling}\label{prop:fdp_subsamling}
Let $D \in \cX^n$ be a dataset and $D':=\text{Sample}_{m,n}(D) \in \cX^m$ be chosen uniformly at random among all the subsets of $D$ with size $m\leq n$ {(sampling without replacement)}. 
For $\cM: \cX^m \rightarrow \cY$, $\cM \circ \text{Sample}_{m,n}:  \cX^n \rightarrow \cY$ is the subsampled mechanism. 
For $0 \le p \le 1$, let $f_p:= pf + (1-p)\Id$ and $C_p(f) := \min\{f_p, f_p^{-1}\}^{**}$ where $\Id(x) = 1-x$, $f^*(y)=\sup_{-\infty<x<\infty} xy-f(x)$, and $f^{**}=(f^*)^*$. 
% \citet{dong2021gaussian} showed that 
If $\cM$ satisfies $f$-DP and $p=m/n$, then $\cM\circ\text{Sample}_{m,n}$ satisfies $C_p(f)$-DP.
\end{restatable} 

\begin{restatable}[Composition \citep{dong2021gaussian}]{proposition}{propfdpcomposition}\label{prop:fdp_composition}
The composition property of DP quantifies the cumulative privacy cost of several DP outputs.
If $f = T_{P, Q}$ and $g = T_{P', Q'}$, their tensor product is defined as $f\otimes g:= T_{P\times P', Q\times Q'}$ where $P\times P'$ is the product measure of $P$ and $P'$, and $Q\times Q'$ is the product measure of $Q$ and $Q'$.
If $\cM_i$ satisfies $f_i$-DP for $i=1,\ldots,k$, then $\cM=(\cM_1,\ldots,\cM_k)$ satisfies $f_1\otimes\ldots\otimes f_k$-DP. {We define $f^{\otimes k}=f_1\otimes\ldots\otimes f_k$ if $f_i=f$ for $i=1,\ldots,k$. For the GDP guarantee, we have $G_{k\mu}=G_{\mu}^{\otimes k}$.}
\end{restatable} 

\paragraph{Group privacy}
While $f$-DP guarantees protection of the privacy of each individual, it can be generalized to give a privacy guarantee for groups of size $k$. We say a mechanism $\cM$ satisfies $f_k$-DP for groups of size $k$ if $T(D_1,D_2)\geq f_k$ for all  $D_1$ and $D_2$ with $D_1\simeq_k D_2$. If a mechanism is $\mu$-GDP, then it is $k\mu$-GDP for groups of size $k$ \citep{dong2021gaussian}. 

{At the end of this section, we state the primal-dual result between $(\ep,\delta)$-DP and $f$-DP in Proposition \ref{prop:primal-dual}. In the next section, we use it to transform the results by \citet{balle2018privacy} from $(\ep,\delta)$-DP  to $f$-DP, and we use it in Proposition \ref{prop:numerical-composition} for the numerical computation of composition, 
since Proposition \ref{prop:fdp_composition} involves high-dimensional testing which is challenging.% ) 
\begin{proposition}[primal-dual view between $(\ep,\delta)$-DP and $f$-DP: \citealp{dong2021gaussian}]\label{prop:primal-dual}
     A mechanism is $(\ep_i, \delta_i)$-DP for all $i\in I$ if and only if it is $f$-DP with $f=\sup_{i\in I} f_{\ep_i, \delta_i}$ where 
     $f_{\ep_i, \delta_i} = \max\{0,1-\delta-\ee^\ep\alpha, \ee^{-\ep}(1-\delta-\alpha)\}$.
    % (Primal to Dual) 
    For a symmetric tradeoff function $f$, a mechanism is $f$-DP if and    only if it is $(\ep, \delta)$-DP for all $\ep \geq 0$ with $\delta(\ep)=1+f^*(-\e^\ep)$ where $f^*(y)=\sup_{-\infty<x<\infty} xy-f(x)$, {also known as the convex conjugate of $f$}.
\end{proposition}
\vspace{-10pt}
}
\section{Privacy analysis of bootstrap resampling}\label{sec:privacy_analysis}
In this section, {we provide novel privacy guarantees for the release of DP bootstrap estimates by proving new results of resampling and composition.} First, we apply the privacy guarantee for sampling with replacement by \citet{balle2018privacy} to bootstrap and convert it from $(\ep, \delta)$-DP to $f$-DP. However, the resulting formula is computationally intractable. We then give a new $f$-DP bound with a direct proof, which agrees with the result of \citet{balle2018privacy}, but is more transparent and computationally friendly. {For private statistical inference based on many DP bootstrap estimates, we provide a numerical method to compute the exact composition result in $f$-DP, along with a user-friendly asymptotic GDP guarantee.} Proofs of the results in this section are provided in Appendix \ref{sec:append_proof}.

\subsection{\texorpdfstring{$f$}{f}-DP guarantee with one bootstrap sample as input}\label{sec:f_DP}
Bootstrap sampling is denoted by a randomized mapping $\mathtt{boot}_n:\cX^n\rightarrow\cX^n$, where $D=(x_{1},\ldots,x_{n})$ is a database, $\mathtt{boot}_n(D)=(x_{i_1},x_{i_2},\ldots,x_{i_n})$ is a randomly generated dataset where ${i_k}\iid \mathrm{Uniform}(\{1,2,\ldots,n\}),\ k=1,2,\ldots,n$. 
Let $p_{i,n}={n\choose i}(1/n)^i(1-1/n)^{n-i}$ which is the probability that a given entry of $D$ is included {$i$}  times in $\mathtt{boot}_n(D)$. We also write $\mathtt{boot}_n$ and $p_{i,n}$ as $\mathtt{boot}$ and $p_i$ respectively when $n$ is known from the context.

We obtain an $f$-DP result using the primal-dual conversion \citep{dong2021gaussian} on the result by \citet{balle2018privacy}. The conversion formula and the original result from \citet{balle2018privacy} are included as Proposition \ref{prop:primal-dual} and Theorem \ref{thm:balle} in the appendix.

\begin{restatable}{proposition}{propepdefboot}\label{prop:epde_f_boot}
{For $i=1,\ldots, n$, let $f_i$ be symmetric tradeoff functions.} 
For $\cM$ satisfying group privacy $f_{i}$-DP with group size $i$, $\cM\circ\mathtt{boot}$ satisfies $f_{\cM \circ \mathtt{boot}}$-DP where $f_{\cM \circ \mathtt{boot}}=C_{1-p_0}\l(\l(\sum_{i=1}^n \frac{p_i}{1-p_0} f_{i}^*\r)^*\r)$, and $C_{1-p_0}(\cdot)$ and $f^*(\cdot)$ are introduced in Proposition \ref{prop:fdp_subsamling}.
\end{restatable} 

Although this representation of $f_{\cM \circ \mathtt{boot}}$ is {seemingly} simple, it is hard to compute or visualize this tradeoff function because evaluating $f_{\cM \circ \mathtt{boot}}(\alpha)$ requires solving {over $n$} optimization problems for each $\alpha$. It is also hard to derive composition results from this $f_{\cM \circ \mathtt{boot}}(\alpha)$, which is crucial for using the bootstrap for statistical inference since multiple bootstrap samples will be used.
Due to the intractability of Proposition \ref{prop:epde_f_boot}, we prove a new $f$-DP result for $\cM\circ\mathtt{boot}$ using the \textit{mixture of tradeoff functions}. The fundamental idea of our proof is to view the bootstrap as a mixture distribution and decompose the overall tradeoff function into a mixture of tradeoff functions where each one is easy to obtain.

\begin{restatable}[Mixture of tradeoff functions]{definition}{defmixtradeoff}\label{def:mix_tradeoff}
For $i=1,2,\ldots,k$, let $f_i$ be tradeoff functions and $p_i\in(0,1]$ satisfying $\sum_{i=1}^k p_i=1$. We write $\underline{f}=(f_1,\ldots,f_k)$ and $\underline{p}=(p_1,\ldots,p_k)$. For a constant $C \in (-\infty,0]$, define $A_i(C):=\{\alpha_i | C\in \partial f_i(\alpha_i)\}$ where $\partial f_i(\alpha_i)$ is the sub-differential of $f_i$ at $\alpha_i$, and $A(C):=\{\sum_{i=1}^k p_i \alpha_i | \alpha_i\in A_i (C)\}$. The mixture of tradeoff functions, $\mathrm{mix}(\underline{p}, \underline{f})$, is defined as follows: For each $\alpha\in(0,1)$, we find $C$ such that $\alpha\in A(C)$ and $\alpha_i\in A_i(C)$ for $i=1,2,\ldots,k$ where $\sum_{i=1}^k p_i \alpha_i=\alpha$, then we define $\mathrm{mix}(\underline{p}, \underline{f})(\alpha)=\sum_{i=1}^k p_i f_i(\alpha_i)$; we define $\mathrm{mix}(\underline{p}, \underline{f})(0)=\sum_{i=1}^k p_i f_i(0)$ and $\mathrm{mix}(\underline{p}, \underline{f})(1)=0$.
\end{restatable}
{
\begin{remark}\label{rmk:mix_tradeoff}
If for $i=1,\ldots,k$, $f_i$ has derivative $f_i'$ which is monotonically increasing for every $\alpha$ in $[0,1]$, we can simplify Definition \ref{def:mix_tradeoff} as $\mathrm{mix}(\underline{p}, \underline{f})=(\sum_{i=1}^k (p_i f_i \circ (f_i')^{-1})) \circ (\sum_{i=1}^k p_i (f_i')^{-1})^{-1}$ since $\sum_{i=1}^k p_i (f_i')^{-1}$ maps the slope $C$ to the type I error, and $\sum_{i=1}^k (p_i f_i \circ (f_i')^{-1})$ maps the slope $C$ to the type II error.
\end{remark}
}
\begin{figure}[t]
        \centering
        \includegraphics[width=0.75\textwidth]{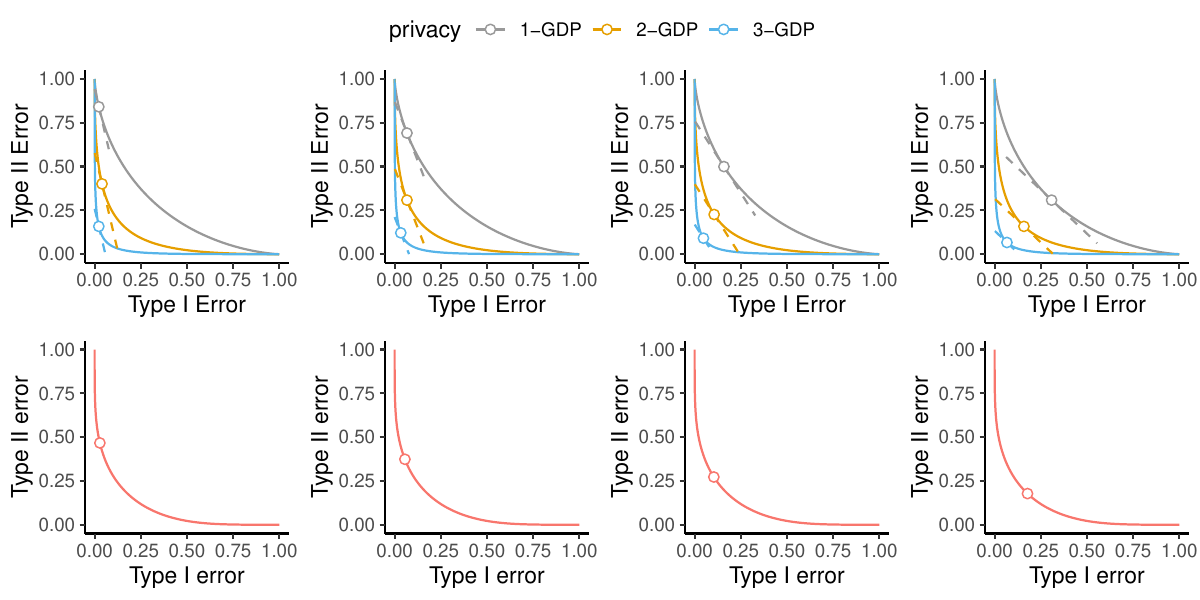}
        \captionof{figure}{An illustration of the mixture of tradeoff functions. In the top row, the solid curves are the tradeoff functions $f_1, f_2, f_3$ corresponding to 1-GDP, 2-GDP, and 3-GDP, respectively, and the three dashed lines are the tangent lines with matched slopes. The mixture of $\underline{f}=(f_1, f_2, f_3)$ with corresponding weights $\underline{p}=(p_1,p_2,p_3)=(\frac{1}{3},\frac{1}{3},\frac{1}{3})$ is $\mathrm{mix}(\underline{p}, \underline{f})$ shown as the curves in the figures on the bottom row. Each circle dot on the mixture curve is the average of the circle dots on $f_1, f_2, f_3$ weighted by $p_1, p_2, p_3$. Note that $\mathrm{mix}(\underline{p}, \underline{f})$ is neither $\frac{1}{3}(f_1+f_2+f_3)$, nor is it a member of the GDP family.}\label{fig:mix-fdp}
\end{figure}

Intuitively, as illustrated in Figure \ref{fig:mix-fdp}, by matching the slopes of each tradeoff function $f_i$, we minimize the overall type II error given a fixed type I error, {and $\mathrm{mix}(\underline{p}, \underline{f})$ is well-defined; see Lemma \ref{lem:mix_welldef} in the appendix.} In Theorem \ref{thm:mixlow}, we show that $\mathrm{mix}(\underline{p}, \underline{f})$ always gives a lower bound on the privacy cost of an arbitrary mixture mechanism. Note that this general result applies to any mixture of DP mechanisms.

\begin{restatable}{theorem}{thmmixlow}\label{thm:mixlow}
For $i=1,2,\ldots,k$, let $f_i$ be tradeoff functions and $\cM_i:\cX^n\rightarrow \cY_i, ~\cY_i\subset \cY$ be mechanisms satisfying $f_i$-DP.
Let $\cM: \cX^n\rightarrow \cY$ be a mixture mechanism which randomly selects one mechanism from $k$ mechanisms, $\{\cM_i\}_{i=1}^k$, with corresponding probabilities $\{p_i\}_{i=1}^k$ where $\sum_{i=1}^k p_i = 1$, and the output of $\cM$ will be the output of $\cM_i$ if $\cM_i$ is selected.  Then $\cM$ satisfies $f$-DP with $f=\mathrm{mix}(\underline{p}, \underline{f})$ where  $\underline{f}=(f_1,\ldots,f_k)$ and $\underline{p}=(p_1,\ldots,p_k)$.
\end{restatable}

The proof of Theorem \ref{thm:mixlow} is based on using different rejection rules when a different $\cM_i$ is selected to allocate the overall type I error to each rejection rule to minimize the overall type II error. We can combine such rejection rules as one if $\cY_i$ are disjoint; therefore, Theorem \ref{thm:mixlow} is not improvable. 
For the DP bootstrap setting, since each $\cM_i$ maps to the same $\cY_i$, we can leverage this fact to strengthen the privacy guarantee in Theorem \ref{thm:fdp_single_boot}.
We separately consider the case that a given entry of $D$ is \textit{not included} in $\mathtt{boot}(D)$ to improve our bound.
% -- in this case, we use a similar idea to the subsampling {result of $f$-DP to improve our bound.}

\begin{restatable}{theorem}{fdpsingleboot}\label{thm:fdp_single_boot}
Let $f_i$ be symmetric tradeoff functions for $i=1,2,\ldots, n$. For a mechanism $\cM$ satisfying group privacy $f_{i}$-DP with group size $i$, $\cM\circ\mathtt{boot}$ satisfies $f_{\cM \circ \mathtt{boot}}$-DP where $f_{\cM \circ \mathtt{boot}} = C_{1-p_0}(\mathrm{mix}(\underline{p}, \underline{f}))$, $\underline{f}=(f_1,\ldots,f_n)$, and $\underline{p}=(\frac{p_1}{1-p_0},\ldots,\frac{p_n}{1-p_0})$.
\end{restatable}

{\begin{remark}
If $\cM$ satisfies $f_1$-DP, \citet{dong2021gaussian} proved that $\cM$ also satisfies $[1-(1-f_1)^{\circ k}]$-DP for groups of size $k$ where $f^{\circ k}$ denote $f$ composed with itself for $k$ times, e.g., $f^{\circ 3}(x) = f(f(f(x)))$. We can use this result in Theorem \ref{thm:fdp_single_boot} if no tighter result is known.
\end{remark}
}

\begin{remark}
The result by \citet{balle2018privacy} can be derived from our $f$-DP result in Theorem \ref{thm:fdp_single_boot}; see Proposition \ref{prop:our_to_epdelta} in the appendix. \citet{balle2018privacy} showed that their results were based on a novel advanced joint convexity property used in $(\ep,\delta)$-DP; similarly, our Theorem \ref{thm:mixlow} and \ref{thm:fdp_single_boot} reveal the advanced joint convexity property for $f$-DP, {and results related to Theorems \ref{thm:mixlow} and \ref{thm:fdp_single_boot} were independently proven by \citet[Lemma 3.1, Lemma 6.3]{NEURIPS2023_acb3e200}.} While \citet{balle2018privacy} provided specific settings attaining their bound, which also prove the optimality of our result, it is unknown for general settings how to construct rejection rules achieving each pair of the type I error and type II error given by our tradeoff function. Therefore, it may be possible to further improve the privacy analysis for the DP bootstrap. Nevertheless, our bound is fairly tight for {the specific settings that we show in Figure \ref{fig:mix_gdp1}a, and it} should suffice for most purposes. 
\end{remark}

Our result in Theorem \ref{thm:fdp_single_boot} is easier to compute when $\partial f_i$ are all known, since for any $C$, we can immediately obtain its corresponding $\alpha$ and $\beta$. For a given $\alpha$ or $\beta$, we can use the bisection method to search for $C$. From the following example and the composition result in Section \ref{sec:clt}, we see that our bound can be easily evaluated for the Gaussian mechanism.

\paragraph{Example: DP bootstrap using the Gaussian mechanism}
% We analyze the DP bootstrap using the Gaussian mechanism to illustrate the results in Theorem \ref{thm:fdp_single_boot}. 
In Figure \ref{fig:mix_gdp1}a, we show that if $\cM$ satisfies 1-GDP, then the DP bootstrap mechanism $\cM\circ\mathtt{boot}$ satisfies $f_{\cM\circ\mathtt{boot}}$-DP where $f_{\cM\circ\mathtt{boot}}$, shown as the solid opaque curve, is our lower bound of all tradeoff functions $T_{\cM\circ \mathtt{boot}(D_1), \cM\circ \mathtt{boot}(D_2)}$ when $D_1\simeq D_2$. 
We also show the tradeoff functions from specific neighboring dataset pairs in Figure \ref{fig:mix_gdp1}a {(transparent curves)} to illustrate misuses in the existing literature: 1) bootstrap cannot be used for free with the same privacy guarantee, i.e., $\cM\circ\mathtt{boot}$ no longer satisfies $1$-GDP, as opposed to \citet{brawner2018bootstrap}; 2) DP bootstrap cannot be analyzed using the privacy loss distribution (PLD) method as in \citet{koskela2020computing}.
Details of this example are in Appendix \ref{sec:append_dpboostrap_gaussian}.

\begin{figure}[t]
    \begin{minipage}{\textwidth}
        \centering
        \includegraphics[width=0.75\textwidth]{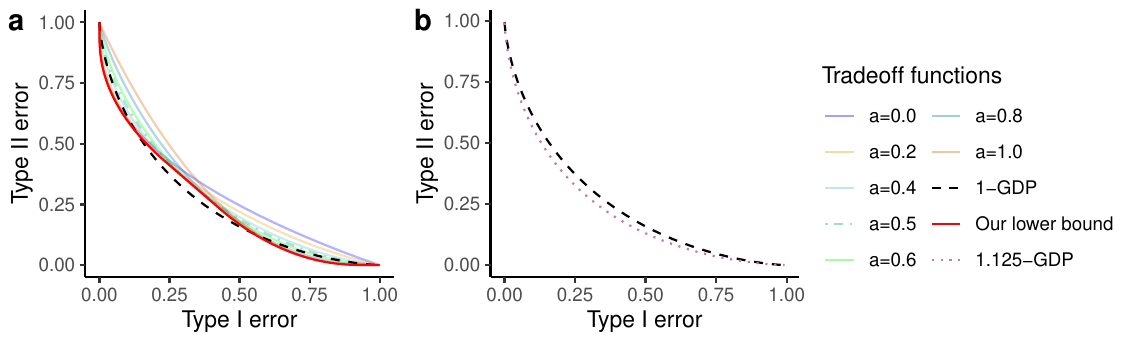}
        \caption{(a) An example showing the relationship between $f$, $f_{\cM\circ\mathtt{boot}}$, and tradeoff functions from specific neighboring dataset pairs. We consider the Gaussian mechanism $\cM$ satisfying 1-GDP shown as the dashed curve. The DP bootstrap mechanism $\cM\circ\mathtt{boot}$ satisfies $f_{\cM\circ\mathtt{boot}}$-DP by Theorem \ref{thm:fdp_single_boot} shown as the solid opaque curve. The transparent curves are tradeoff functions $T_{\cM\circ \mathtt{boot}(D_1), \cM\circ \mathtt{boot}(D_2)}$ where $\cM(D)=\frac{1}{n}\sum_{i=1}^n x_i + \xi,D=(x_1,x_2, \ldots, x_n),\xi\sim\cN(0,\frac{1}{n^2})$, and $D_1=(a,0,\ldots,0)$, $D_2=(a-1,0,\ldots,0)$. The solid curve is tight as a lower bound of the transparent curves. The dashed line (1-GDP) and the dotted dashed line (corresponding to $a=0.5$) are misused as lower bounds in \citet{brawner2018bootstrap} and \citet{koskela2020computing}, respectively (technically \citet{brawner2018bootstrap} worked in zCDP; see Appendix \ref{sec:append_dpboostrap_gaussian} for counter-examples to both).  \\
        (b) The asymptotic price of using the bootstrap. Running a $1$-GDP mechanism on $B$ different bootstrap samples has a similar privacy guarantee to running a $(\sqrt{2-2/\ee})$-GDP mechanism on the original dataset for $B$ times if $B$ is large enough 
        % In this figure, we let $\mu_0=1$. Note that  
        $(\sqrt{2-2/\ee}\approx1.125).$
        }\label{fig:mix_gdp1}
    \end{minipage}
\end{figure}

{\subsection{Finite and asymptotic composition of DP bootstrap}\label{sec:clt}
}

% \subsection{Asymptotic Composition of DP bootstrap with Gaussian mechanisms}\label{sec:clt}
{In this section, we derive composition results for the DP bootstrap because we need many bootstrap estimates to estimate the sampling distribution which is then used to conduct statistical inference.
We first present a numerical method in Proposition \ref{prop:numerical-composition} to calculate the exact privacy guarantee for the composition of several instances of the DP bootstrap. Then we prove an asymptotic GDP result for the composition of DP bootstrap when using the Gaussian mechanism.

By the definition of $f$-DP, there exist $P$ and $Q$ such that $f=T(P, Q)$. From the proof of \citep[Proposition 1]{dong2021gaussian}, a simple choice is to set the measurable space to be $[0, 1]$, $P$ to be the uniform distribution, and $Q$ to have density $-f'(1-x)$ on $[0, 1)$ and a point mass at $1$ with $Q[\{1\}] = 1-f(0)$. As a tradeoff function, $f$ is convex (see Proposition \ref{prop:convex-tradeoff}) which means $f$ is almost everywhere differentiable \citep[Theorem 25.5]{rockafellar1997convex}, and we set the density of $Q$ to be 0 at the nondifferentiable points of $f(1-x)$. With this $P$ and $Q$, we use Lemma 5.2 by \citet{zheng2020sharp} to calculate the privacy profile of the composition of $f_{\cM \circ \mathtt{boot}}$, which can be translated to $f$-DP by Proposition \ref{prop:primal-dual}. 
We summarize the result in Proposition \ref{prop:numerical-composition}, where we assume that $f_i'$ are strictly increasing to simplify the exposition.
When $f_i'$ are not strictly decreasing or $f_i$ are not differentiable at some points, a general but more complex result can be derived from Theorem \ref{thm:fdp_single_boot}. 

\begin{proposition}\label{prop:numerical-composition}
Let $f_i$ be symmetric tradeoff functions with strictly increasing derivatives $f_i'$ for $i=1,2,\ldots, n$. For a mechanism $\cM$ satisfying group privacy $f_{i}$-DP with group size $i$, $\cM\circ\mathtt{boot}$ satisfies $f_{\cM \circ \mathtt{boot}}$-DP where $f_{\cM \circ \mathtt{boot}} = T(P, Q)$, $P$ is $\mathrm{Uniform}(0,1)$ with density function $p(x)=1$, and $Q$ has the following density function $q(x)$, where $q$ and $x$ are parameterized by $C$, and $x^* = \sum_{i=1}^n  \frac{p_i}{1-p_0}  (f_i')^{-1}(-1)$,
$$
q(x) = 
\left\{
    \begin{aligned}
         & p_0 - (1-p_0)  C  \quad & 1 \geq x \geq 1 - x^*, \quad &  C < -1 \\
         & x \quad & 1-x^* > x \geq 1-p_0 - (1-2p_0) x^*, \quad &  C = -1 \\
         & 1/(p_0 - (1-p_0)/C) \quad & 1-p_0 - (1-2p_0) x^* > x \geq 0, \quad &  0 > C > -1, 
    \end{aligned}
\right.
$$
$$
x = 
\left\{
    \begin{aligned}
         & 1 - \sum_{i=1}^n  \frac{p_i}{1-p_0}  (f_i')^{-1}(C)  \quad & & C < -1 \\
         & 1-\sum_{i=1}^n  {p_i}  (f_i')^{-1}(C) - p_0\l(1-\sum_{i=1}^n  \frac{p_i}{1-p_0} (f_i')^{-1}(1/C)\r)  \quad & & 0 > C > -1.
    \end{aligned}
\right.
$$
The privacy profile of $f_{\cM \circ \mathtt{boot}}^{\otimes B}$-DP is $\delta_{\otimes, B}(\ep)$, which is recursively defined by \\ $\delta_{\otimes, B}(\ep) = \int_{\mathbb{R}} \delta_{\otimes, B-1}\l(\ep - \log(\frac{q(x)}{p(x)})\r) q(x)~\mathrm{d}x$ and $\delta_{\otimes, 1}(\ep) = \int_{\mathbb{R}} \max(0, q(x) - \mathrm{e}^\ep p(x))~\mathrm{d}x$.
\end{proposition}
\begin{remark}
    The $C$ in Proposition \ref{prop:numerical-composition} has the same meaning as the $C$ in Definition \ref{def:mix_tradeoff}: They are the matched slopes of the tradeoff functions $\{f_i\}_{i=1}^n$.
\end{remark}

}

Although Proposition \ref{prop:numerical-composition} can be used to evaluate the privacy guarantee of composition, the iterative calculation is often burdensome for a large number of compositions \citep{zheng2020sharp}. We prove an asymptotic result to understand the behavior as the number of compositions goes to infinity. For simplicity, we assume the initial mechanism satisfies GDP, but it may be possible to extend our result to mechanisms satisfying $f$-DP for other tradeoff functions $f$.  While Theorem \ref{thm:boot_comp} assumes that the base mechanism is GDP, it need not be the Gaussian mechanism. For example, \citet{gopi2022private} showed that the exponential mechanism satisfies GDP for many convex empirical risk minimization problems.

{\begin{restatable}{theorem}{thmbootcomp}\label{thm:boot_comp}
Let $\mu \in(0, \infty)$ be a given constant, and $\{\mu_B\}_{B=1}^{\infty}$ be a sequence such that $\mu_B\in(0,\infty)$ and $\lim_{B\rightarrow\infty}\mu_B\sqrt{(2-2/\ee)B}\rightarrow \mu$. For a mechanism $\cM_B$ that satisfies $\mu_B$ -GDP, let $f_{\cM_B,\mathtt{boot}}$ be the $f$-DP guarantee of $\cM_B\circ \mathtt{boot}$ from Theorem \ref{thm:fdp_single_boot}. Then 
$$
\lim\limits_{B\rightarrow\infty} f_{\cM_B,\mathtt{boot}}^{\otimes B} = G_{\mu\sqrt{(2-{1}/{n})(1-(1-{1}/{n})^n)/(2-{2}/{e})}} \geq G_{\mu}.
$$ 
% Let $\{f_{Bi}:1\leq i\leq B\}_{B=1}^\infty$ be a triangular array of symmetric tradeoff functions where each $f_{Bi}$ corresponds to $\mu_B$-GDP, i.e., $f_{Bi}=G_{\mu_B}$. 
% Call $f_{Bi,\mathtt{boot}}$ as the lower bound for $f_{Bi}$ from Theorem \ref{thm:fdp_single_boot}. Then 
% $\lim\limits_{B\rightarrow\infty} f_{B1,\mathtt{boot}}\otimes f_{B2,\mathtt{boot}}\otimes \cdots \otimes f_{BB,\mathtt{boot}} = G_{\mu\sqrt{(2-{1}/{n})(1-(1-{1}/{n})^n)/(2-{2}/{e})}} \geq G_{\mu}.$ 
\end{restatable}
}

{Figure \ref{fig:numerical_clt} shows the comparison between the privacy profiles of the asymptotic result by Theorem \ref{thm:boot_comp} and the numerical result by Proposition \ref{prop:numerical-composition} where $\mu = \sqrt{2-2/\ee}$. When $B$ is small, the asymptotic result tends to overestimate the privacy offered by the DP bootstrap, as the asymptotic $\delta(\ep)$ is lower than the numerical $\delta(\ep)$. However, as $B$ increases, the asymptotic and numerical results converge. Furthermore, for larger values of $n$, a higher $B$ is required for the asymptotic result to closely match the numerical result. More results are available in the supplementary material.
\begin{figure}[t]
        \centering
    \begin{minipage}{0.85\textwidth}
        \centering
        \includegraphics[width=0.24\textwidth]{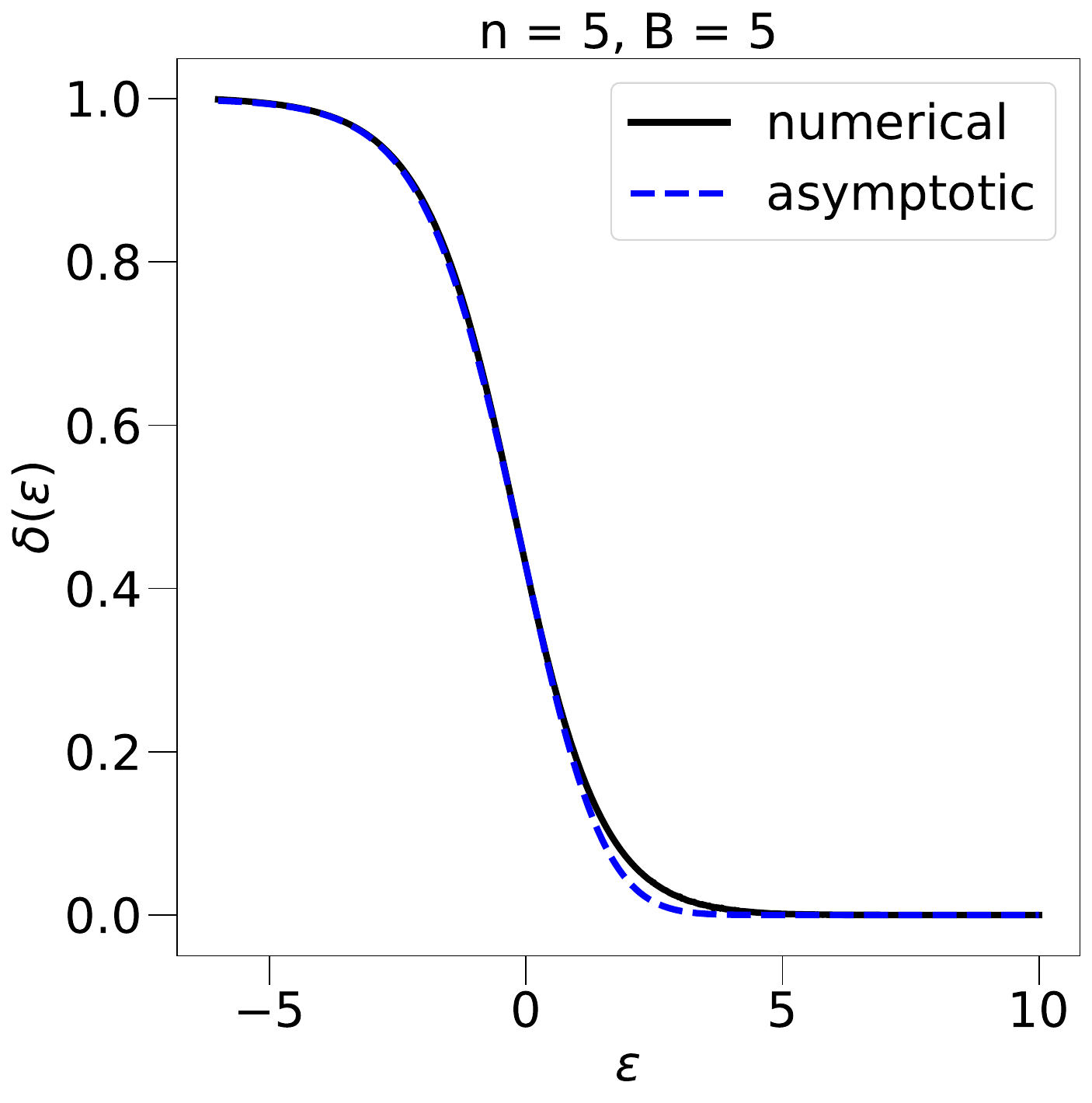}
        \includegraphics[width=0.24\textwidth]{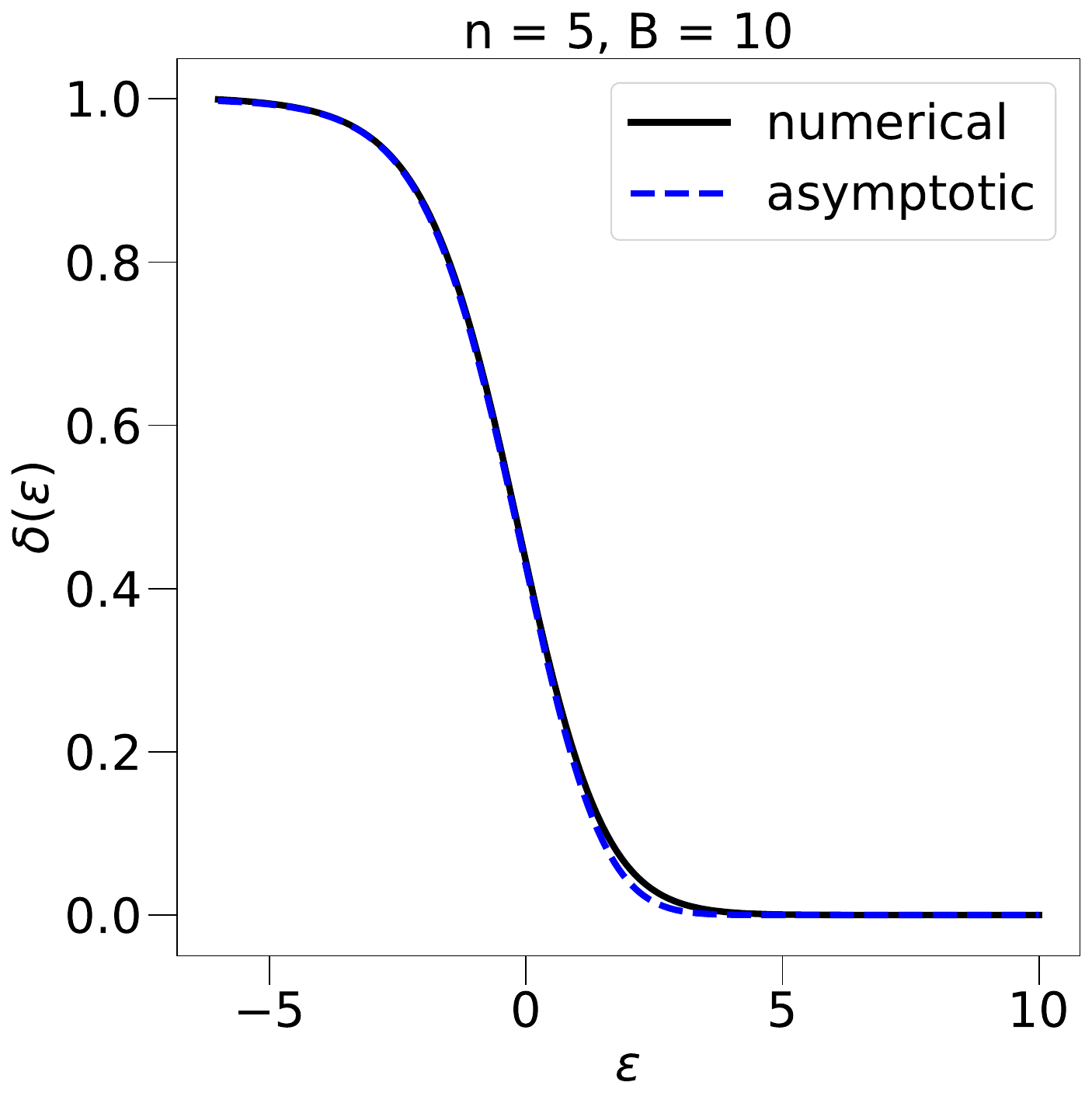}
        \includegraphics[width=0.24\textwidth]{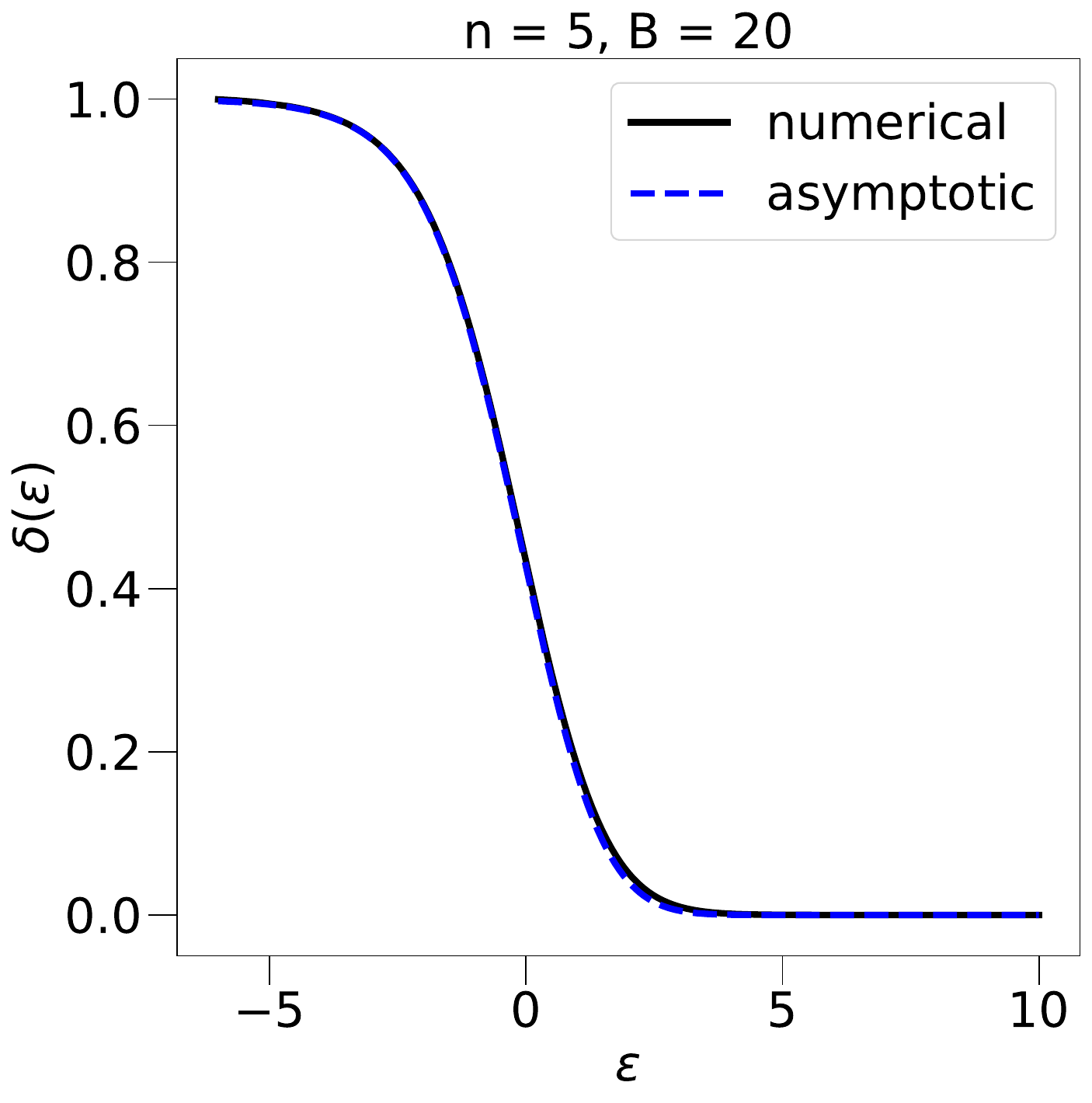}
        \includegraphics[width=0.24\textwidth]{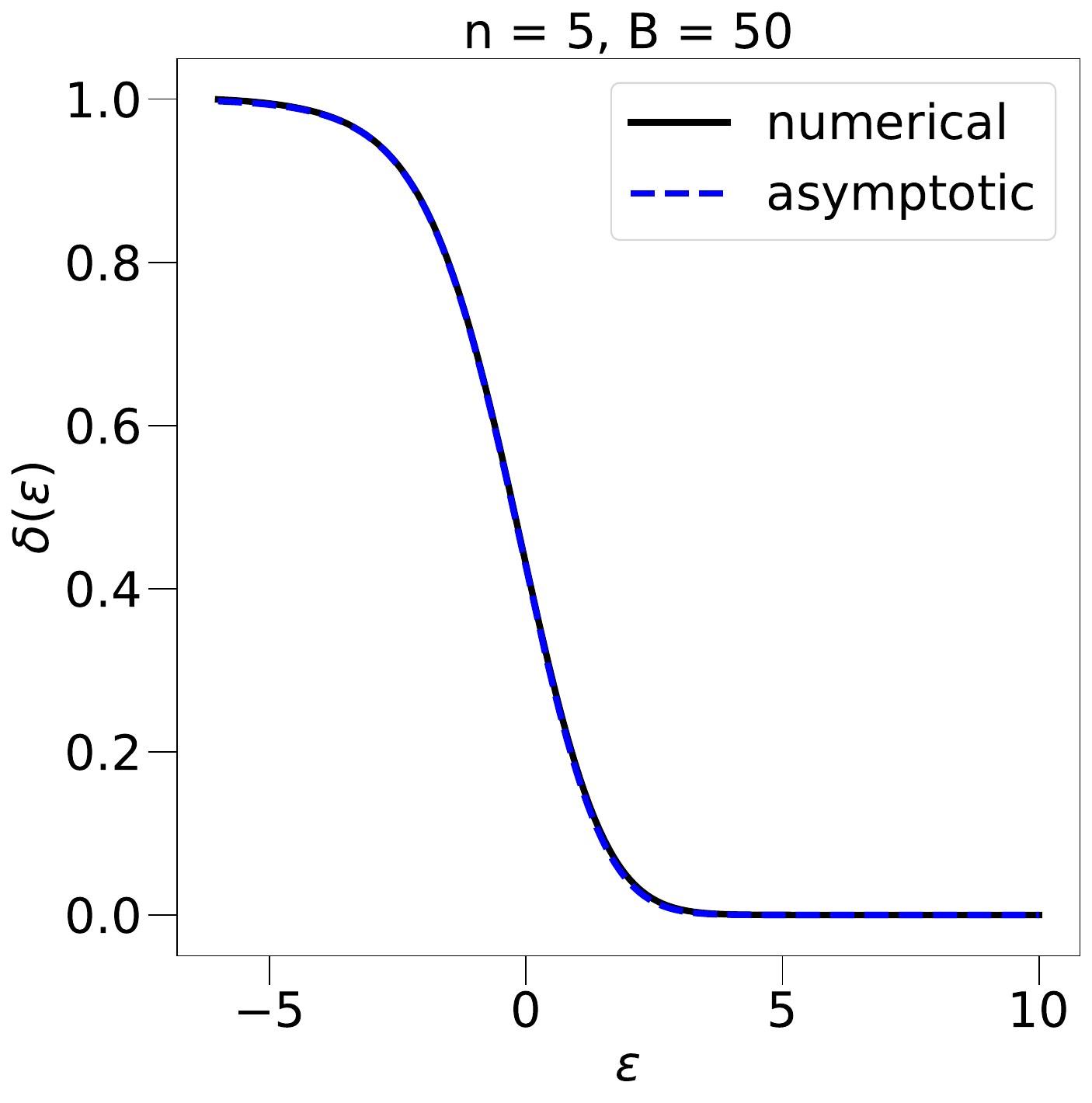}
        \includegraphics[width=0.24\textwidth]{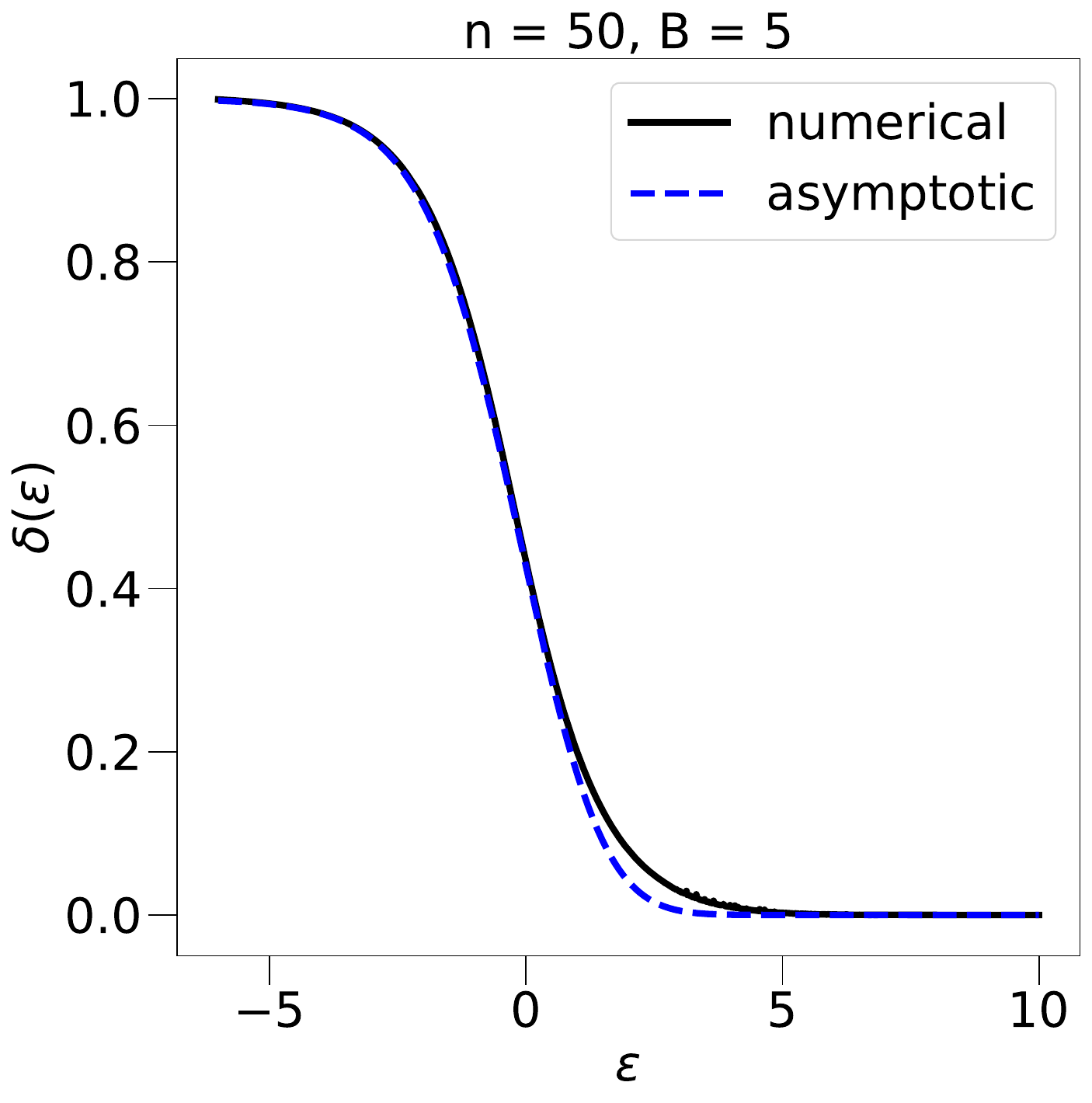}
        \includegraphics[width=0.24\textwidth]{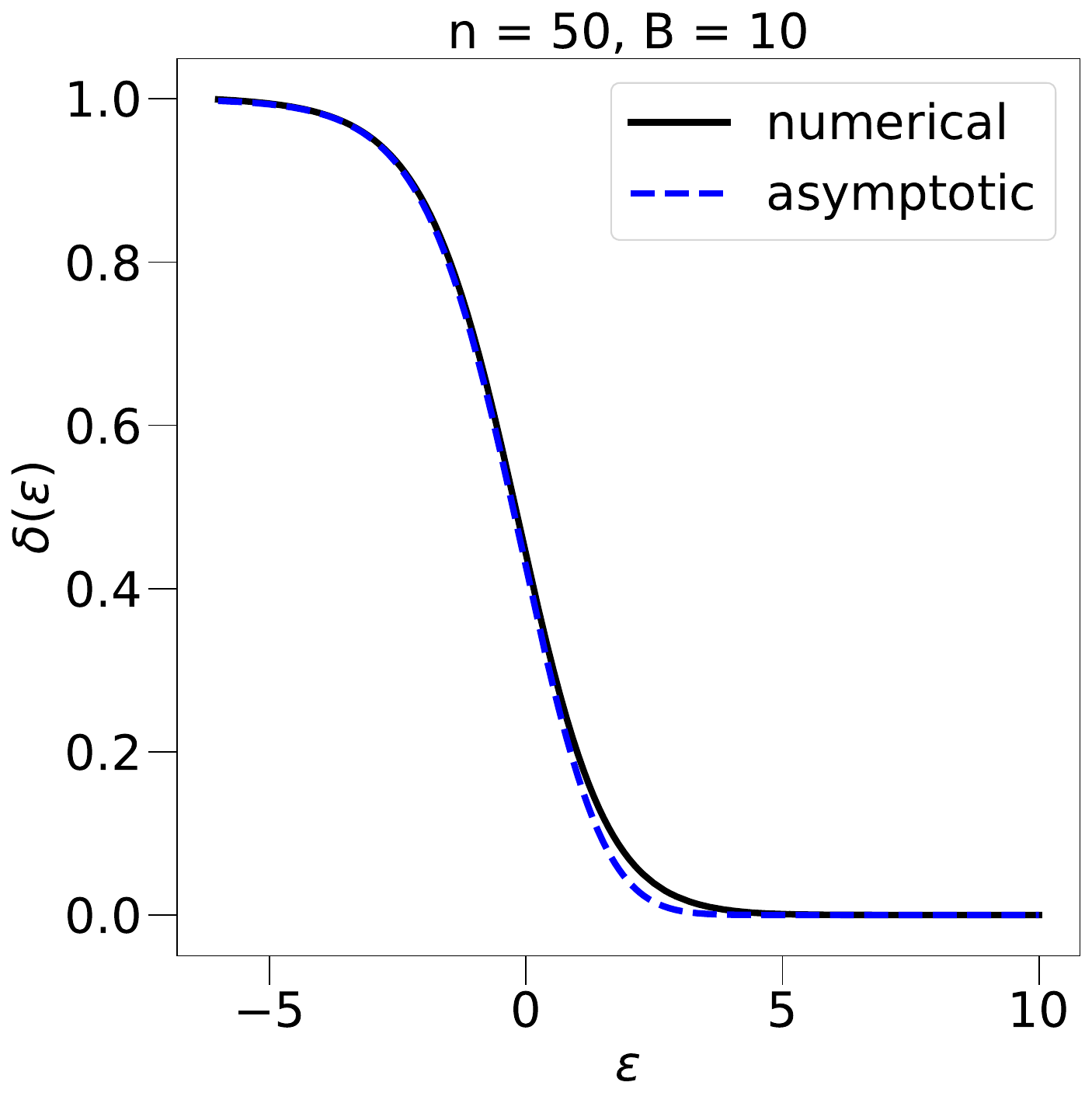}
        \includegraphics[width=0.24\textwidth]{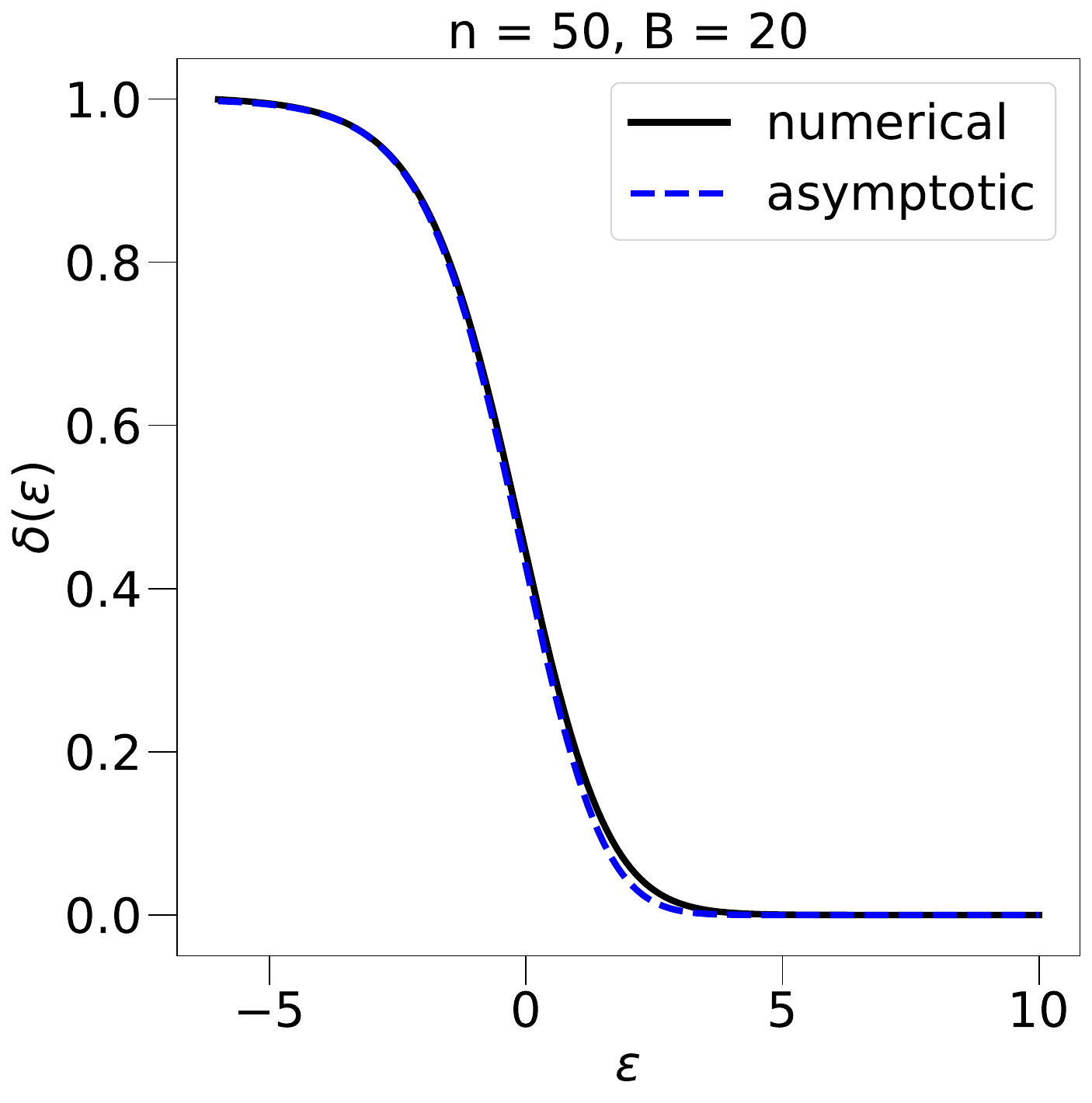}
        \includegraphics[width=0.24\textwidth]{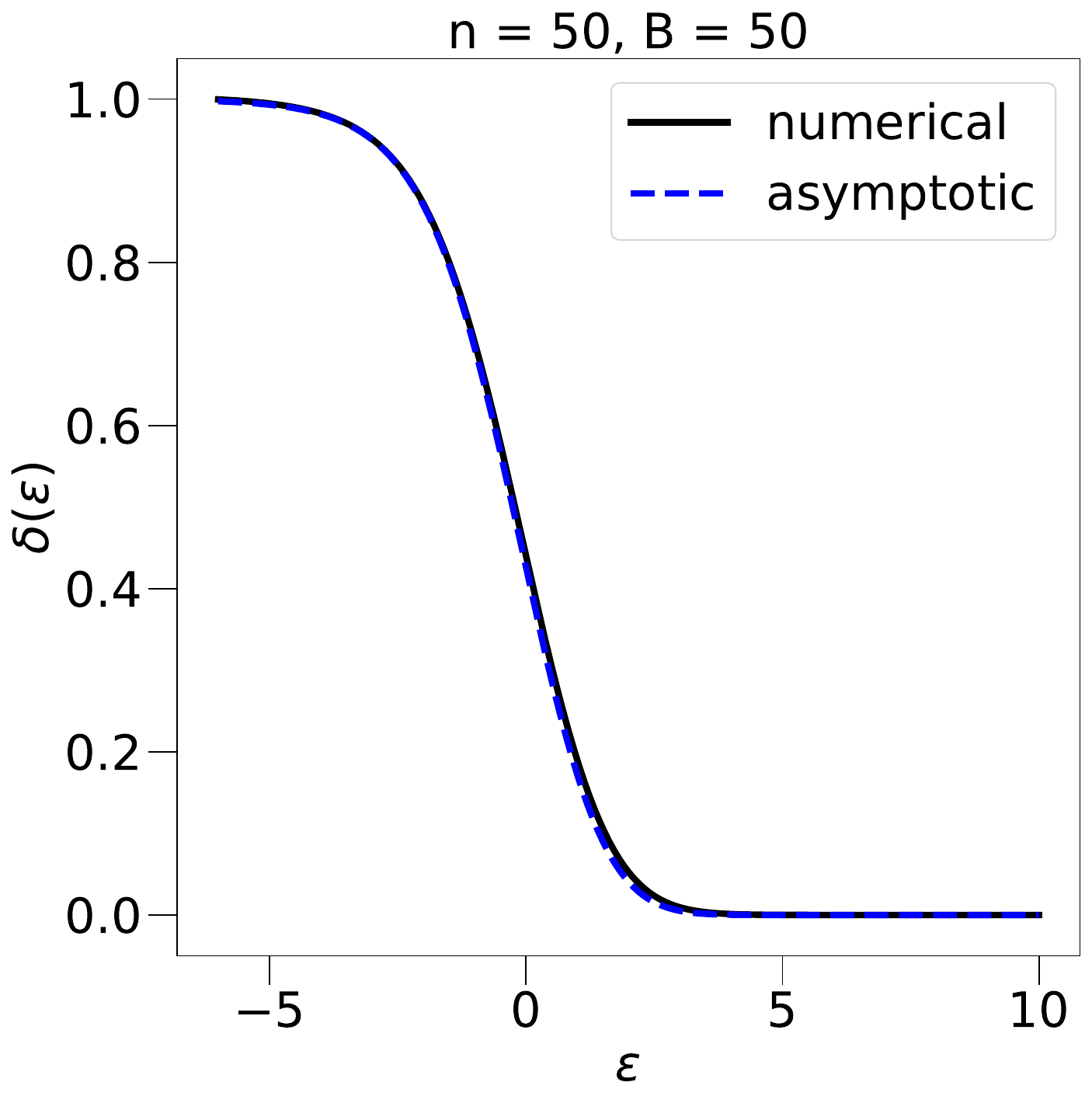}
    \end{minipage}
    \caption{Comparison between the privacy profile $\delta(\ep)$ of composition computed from asymptotics (Theorem \ref{thm:boot_comp}) and numerical evaluation (Proposition \ref{prop:numerical-composition}). The composition is for releasing $B$ DP bootstrap outputs where the original mechanism is $\sqrt{1/B}$-GDP.}\label{fig:numerical_clt}
\end{figure}
}

Note that the asymptotic privacy guarantee of the above composition result is the same as running a $(\sqrt{2-2/\ee})\mu_B$-GDP mechanism on the \textit{original} dataset (not on the bootstrap sample) for $B$ times \citep{dong2021gaussian}.
Therefore, the factor $(\sqrt{2-2/\ee})=1.12438\ldots<1.125$ is the price we pay for using the DP bootstrap (see Figure \ref{fig:mix_gdp1}b). Although there is a small increase in the privacy cost, the bootstrap samples now contain the randomness from sampling as well as from the privacy mechanism. In the next section, we will see how we can use DP bootstrap estimates to perform statistical inference.

\section{Statistical inference with DP bootstrap}\label{sec:inference}
In this section, {we use the insights from the classic methods for statistical inference with bootstrap estimates to develop private statistical inference using DP bootstrap estimates. We analyze the asymptotic distributions of the sample mean and sample variance of the DP bootstrap estimates, and we use them to build asymptotic CIs. 
We prove that the sample mean of DP bootstrap estimates is a consistent estimator with optimal convergence rate, and our asymptotic CIs has coverage guarantee and optimal average width. Although \citet{brawner2018bootstrap} have explored similar ideas, their analysis is less rigorous, without proof, and restricted to the population mean inference problem, while we provide solid proofs under more general settings.}
% Although the asymptotic rates demonstrate the theoretical potential of DP bootstrap, \citet{brawner2018bootstrap} showed that the empirical finite-sample coverage of the CI constructed from these point estimates using a normal approximation is unsatisfactory. 
{To further improve the finite-sample performance of using DP bootstrap for private inference,} we propose another approach using deconvolution to recover the sampling distribution from the DP bootstrap estimates. The deconvolved distribution satisfies the same $f$-DP guarantee because of the post-processing property \citep[Proposition 4]{dong2021gaussian}, and {in our simulation, it is very close to the non-private sampling distribution. In the end of the section, we provide the algorithms of the DP bootstrap estimates, the asymptotic CIs, and the deconvolution method.} 

 Throughout the rest of this paper, we focus on using the Gaussian mechanism in DP bootstrap and its corresponding asymptotic privacy guarantee, $\mu$-GDP. We include a remark when other types of mechanism can be used. Note that the asymptotic privacy guarantee can always be replaced by numerical composition results in Proposition \ref{prop:numerical-composition} if a strict privacy guarantee is needed. \citep{zheng2020sharp, gopi2021numerical, zhu2021optimal}.

\subsection{Statistical inference with Efron's bootstrap}

Before explaining how to conduct private statistical inference with the DP bootstrap, we briefly review two basic inference methods with the original Efron's bootstrap estimates \citep{efron1994introduction}, and we develop our private inference methods based on these classic methods.
We denote the original dataset by $D\in\cX^n$ and the estimator for a population parameter $\theta$ by $g(D)$. The sampling distribution of $g(D)$ can be estimated from bootstrap estimates, $\{g(D_j)\}_{j=1}^B$ where $D_j$ is the $j$th bootstrap sample of $D$. 
{\begin{enumerate}
    \item Standard interval: Estimate the sampling distribution using a normal distribution $\cN(\theta, \hat{s}^2_{g,B})$ where $\hat{s}^2_{g,B} = \frac{1}{B-1} \sum_{j=1}^B (g(D_j)-\hat{m}_{g,B})^2$ and $\hat{m}_{g,B} = \frac{1}{B} \sum_{j=1}^B g(D_j)$. The $(1-\alpha)$-CI of $\theta$ is $\l[g(D)+\Phi^{-1}(\frac{\alpha}{2})\hat{s}_{g,B},~ g(D)+\Phi^{-1}(1-\frac{\alpha}{2})\hat{s}_{g,B}\r]$ where $\Phi(x)$ is the CDF (cumulative distribution function) of a standard normal distribution.
    \item Percentile interval: Estimate the sampling distribution using the empirical CDF of $\{g(D_j)\}_{j=1}^B$ denoted by $\hat{F}_{g,B}$. The $(1-\alpha)$-CI of $\theta$ is $[\hat{F}_{g,B}^{-1}(\frac{\alpha}{2}),~ \hat{F}_{g,B}^{-1}(1-\frac{\alpha}{2})]$.
\end{enumerate}
}

% Using the normal approximation to the sampling distribution, we can obtain the \textit{standard interval} for $\theta$: we first estimate the standard error of $g(D)$ by $\hat{s}^2_{g,B} = \frac{1}{B-1} \sum_{j=1}^B (g(D_j)-\hat{m}_{g,B})^2$ where $\hat{m}_{g,B} = \frac{1}{B} \sum_{j=1}^B g(D_j)$, then approximate the sampling distribution of $g(D)$ by $\cN(\theta, \hat{s}^2_{g,B})$, and finally build CIs for $\theta$ using $g(D)$ and $\hat{s}^2_{g,B}$. Alternatively, if we approximate the sampling distribution directly using the empirical distribution of $\{g(D_j)\}_{j=1}^B$, we can build the \textit{percentile interval} for $\theta$ using the percentiles of $\{g(D_j)\}_{j=1}^B$. 

\subsection{Convergence rate of DP bootstrap point estimates and confidence intervals}\label{sec:population_mean}
{Similar to the standard interval based on the point estimates, $g(D)$ and $\hat{s}_{g,B}$, we denote $\tilde m_{g,B}$ and $\tilde s^2_{g,B}$ as the sample mean and sample variance of DP bootstrap estimates, respectively, and use their asymptotic distributions to build valid private CIs. 

By analyzing the uncertainty from sampling, bootstrap resampling, and our DP mechanism, we prove Lemma \ref{lem:point-estimates} about the mean squared error of $\tilde m_{g,B}$ for the population mean inference, and the asymptotic distributions of $\tilde m_{g,B}$ and $\tilde s^2_{g,B}$ assuming finite support of $D$.

\begin{lemma}[Consistency \& asymptotic distribution]\label{lem:point-estimates}
Let $D=\{x_1,\cdots,x_n\}\in\cX^n$ be a dataset where $x_i\iid F$ are random variables and $\theta$ is a parameter of interest determined by $F$, we let $g(D)$ be an estimator of $\theta$, $D_j$ be the $j$th bootstrap sample of $D$, and $\{\tilde{g}(D_j)\}_{j=1}^B$ be the DP bootstrap estimates using the Gaussian mechanism: $\tilde g(D_j) = g(D_j) + \xi_j$, $\xi_j\iid\cN(0,\sigma_e^2)$. Let $\tilde m_{g,B} = \frac{1}{B} \sum_{j=1}^B \tilde g(D_j)$ and $\tilde s^2_{g,B} = \frac{1}{B-1} \sum_{j=1}^B (\tilde g(D_j)-\tilde m_{g,B})^2$. 
\begin{enumerate}
    \item For the population mean inference where $g(D)=\frac{1}{n}\sum_{i=1}^n x_i$ and $\theta = \mathbb{E}[x_i]$, we have $\mathbb{E}[\tilde m_{g,B} - \theta]^2 = \frac{1}{n}[(1+\frac{1}{B} - \frac{1}{n B})\mathbb{E}(x_i-\theta)^2] + \frac{\sigma_e^2}{B}$. 
    \item Let $\Pi$ denote the set of all distributions with finite support on $\cX$. Assume that $F = F(\eta, s, d)=\sum_{k=1}^d \eta_k \delta(s_k) \in \Pi$ where $s=\{s_1,\cdots,s_d\} \subseteq \cX$ is the support of $F$, $\delta(s_j)$ denotes the point mass at $s_j$ with measure 1, the distribution parameter $\eta=(\eta_1,\cdots,\eta_{d-1}) \in H := \{\eta\in (0,1)^{d-1}\big|\sum_{k=1}^{d-1}\eta_k < 1\}$, and $\eta_d = 1 - \sum_{k=1}^{d-1}\eta_k$. The empirical distribution of $D$ is $\hat F_n=F(\hat\eta, s, d)$ where $\hat\eta_k=\frac{1}{n}\sum_{i=1}^n I(x_i=s_k)$. For a continuously differentiable function $T: H\to\mathbb{R}$ such that $\theta=T(\eta)$ and $g(D)=T(\hat\eta)$, we let $\sigma_g^2=\frac{1}{n}\l(\frac{\partial T}{\partial \eta}\r)^\intercal \Sigma \frac{\partial T}{\partial \eta}$ and $\Sigma = \diag(\eta) - \eta \eta^\intercal$, and we have 
    $$
    \frac{\tilde m_{g,B} - \theta}{\sqrt{\sigma_g^2+\frac{1}{B}(\sigma_g^2 + \sigma_e^2)}} \overset{d}{\to} \cN(0, 1) \quad\mathrm{and}\quad (B-1)\frac{\tilde s^2_{g,B}}{\sigma_g^2 + \sigma_e^2} \overset{d}{\to} \chi^2_{B-1}.
    $$
\end{enumerate}
\end{lemma}
\begin{remark}
    From Theorem \ref{thm:boot_comp}, using $\sigma_e^2=\frac{(\Delta(g))^2 (2-2/e) B}{\mu^2}$ where $\Delta(g)$ is the $\ell_2$-sensitivity of $g$ on $\cX^n$, we have that $(\tilde m_{g,B}, \tilde s^2_{g,B})$ satisfy $\mu$-GDP asymptotically as $B\rightarrow \infty$.
    For bounded $\mathcal{X}$, we have $\Delta(g) = O(\frac{1}{n})$, therefore, part 1 of Lemma \ref{lem:point-estimates} has the rate $\mathbb{E}[\tilde m_{g,B} - \theta]^2 = O\l(\frac{1}{n} + \frac{1}{n^2\mu^2}\r)$ as $\frac{\sigma_e^2}{B}=O(\frac{1}{n^2\mu^2})$ and $(1+\frac{1}{B} - \frac{1}{n B})=O(1)$ for $B\geq 1$. This rate means that $\tilde m_{g,B}$ is a consistent estimate of $\theta$, and it matches the minimax rate $\Omega\l(\frac{1}{n} + \frac{\log(1/\delta)}{n^2\ep^2}\r)$ for the population mean estimation under $(\ep,\delta)$-DP \citep[Theorem 3.1]{cai2021cost} since the Gaussian mechanism for $\mu$-GDP satisfies $(\ep,\delta)$-DP when $\mu^2 \propto \frac{\ep^2}{\log(1/\delta)}$ \citep{dwork2014algorithmic}. 
\end{remark}
    
\begin{remark}
    Part 2 of Lemma \ref{lem:point-estimates} follows the proof by \citet[Theorem 2.2]{beran1997diagnosing}, where they used the finite support assumption to prove that the non-parametric bootstrap distribution converges to the sampling distribution. Although our proof is limited to discrete distributions, this should not cause any problems for real-world datasets, since all measurements are taken with finite precision.
\end{remark}

% From part 2 of Lemma \ref{lem:point-estimates}, we can see that when $\sigma_e^2 \gg \sigma_g^2$, it is difficult to construct an estimate $\hat\sigma$ such that the distribution of $\frac{\tilde m_{g,B} - \theta}{\hat\sigma}$ is close to $\cN(0,1)$, which means that the na\"ive plugin standard CIs may have under-coverage issue.
In the remainder of this section, we build asymptotically valid private CIs for $\theta$ using the asymptotic distributions of $\tilde m_{g,B}$ and $\tilde s^2_{g,B}$ (part 2 of Lemma \ref{lem:point-estimates}) in Theorem \ref{thm:repro}.
The construction of this CI was inspired by the repro sample method \citep{xie2022repro}; for the reader's convenience, we provide a simple and self-contained proof of Theorem \ref{thm:repro} in the appendix, which does not require familiarity with the work of \citet{xie2022repro}.

\begin{theorem}[Asymptotically valid CI]\label{thm:repro}
Assume that $\frac{X_n-\theta}{\sqrt{\sigma_g^2+\frac{1}{B}(\sigma_g^2 + \sigma_e^2)}} \overset{d}{\to} \cN\l(0, 1\r)$ and $\frac{(B-1)Y_n}{\sigma_g^2 + \sigma_e^2} \overset{d}{\to} \chi^2_{B-1}$, respectively, as $n\rightarrow\infty$. Given $\alpha \in [0,1]$, for any $0 < \omega < \alpha$, we have 
$$
\lim_{n\rightarrow\infty}\mathbb{P}\l(\theta \in \l[X_n - \hat{r}_n,~ X_n + \hat{r}_n\r]\r) \geq 1-\alpha, 
$$ 
where $\hat{r}_n=\Phi^{-1}(1-\frac{\omega}{2}) \sqrt{\hat\sigma_{g}^2 + \frac{1}{B}(\hat\sigma_{g}^2 + \sigma_{e}^2)}$, $\hat\sigma_{g}^2 = \max\l(0, \frac{B-1}{c}Y_n  - \sigma_e^2\r)$, and $c$ is the $(\alpha - \omega)$ quantile of the $\chi^2_{B-1}$ distribution. That is, $\l[X_n - \hat{r}_n,~ X_n + \hat{r}_n\r]$ is an asymptotically valid CI for $\theta$ with level $(1-\alpha)$.
\end{theorem}
% \begin{remark}

    When used in finite settings in part 2 of Lemma \ref{lem:point-estimates} where $X_n:=\tilde m_{g,B}$ and $Y_n:=\tilde s^2_{g,B}$, the private CIs in Theorem \ref{thm:repro} are usually conservative. Note that the asymptotic coverage is for $n\rightarrow\infty$ with finite $B$. If we use a plugin estimator for $\sigma_g^2$ instead, i.e., $\acute\sigma_g^2 := Y_n - \sigma_e^2$ based on $\EE[Y_n] = \sigma_g^2 + \sigma_e^2$, the CI will not have enough coverage when $B$ is small since $\acute\sigma_g^2$ can be negative. Theorem \ref{thm:repro} uses an upper bound where $\hat\sigma_g^2$ ensures good coverage of CIs for any fixed $B$. 
    
    The CI width $\hat{r}_n$ is a random variable depending on $\alpha$, $\omega$, $B$, $\sigma_e^2$, and $\sigma_g^2$, where $\alpha$, $B$, and $\sigma_e^2$ are known, $\sigma_g^2$ is unknown, and $\omega$ is tunable while it must be chosen before observing $(X_n, Y_n)$ to ensure the validity of the coverage guarantee.
% \end{remark}

\begin{remark}
\citet{brawner2018bootstrap} built confidence intervals using a result similar to part 2 of Lemma \ref{lem:point-estimates} without providing a concrete proof. In their Equation (34), they estimated $\sigma_g^2$ by $\hat\sigma_g^2=\tilde s_{g,B}^2 - \frac{c_{\alpha'}}{B-1}\sigma_e^2$ where $c_{\alpha'}$ is a tunable parameter, and in Equation (35), they estimated the variance of $\tilde m_{g,B}$ by $\hat\sigma^2 = \hat\sigma_g^2+\frac{\sigma_e^2}{B}$. Then they showed that the coverage of the na\"ive private version of the standard interval $[\tilde m_{g,B} + \Phi^{-1}(\frac{\alpha}{2})\hat\sigma,~ \tilde m_{g,B} + \Phi^{-1}(1-\frac{\alpha}{2})\hat\sigma]$ is lower than the nominal confidence level when $c_{\alpha'}=B-1$, i.e., $\hat\sigma_g^2$ is an unbiased estimator of $\sigma_g^2$ \citep[Figure 7b]{brawner2018bootstrap}.
They noted that using an unbiased estimate could underestimate the uncertainty of $\tilde m_{g,B}$ (the worst case is that $\hat\sigma_g^2$ is negative), causing undercoverage. Therefore, they solved this problem using an ad hoc ``conservative'' (larger) estimate $\hat\sigma_g^2$ of $\sigma_g^2$ with a smaller $c_{\alpha'}$. However, they did not prove any guarantee of coverage for this approach. In contrast, we prove Theorem \ref{thm:repro} on the coverage of our approach. In Section \ref{subsec:priv_nonpriv}, we use simulation to compare our method with theirs.
\end{remark}

Proposition \ref{prop:repro-rate} shows that under certain choices of $B$ and $\omega$, the width of the CI in Theorem \ref{thm:repro} is asymptotically optimal compared to the width derived from a normal approximation with variance computed by the Cram\'er-Rao lower bound.
\begin{proposition}[Optimality]\label{prop:repro-rate}
% For the sample mean setting, $g(D)=\frac{1}{n}\sum_{i=1}^n x_i$, $x_i\iid P$ with support $\mathcal{X}=[0,1]$. We denote $\theta = \mathbb{E}[X]$ and $\sigma_g^2=\frac{1}{n}\Var(X)$ where $X\sim P$, and let . 
Under the assumption of part 2 of Lemma \ref{lem:point-estimates}, we use the CI in Theorem \ref{thm:repro} where $\sigma_e^2 = \frac{(2-2/e)B}{n^2\mu^2}$. When $n\rightarrow \infty$, we assume $B \rightarrow \infty$, $B = o(n^2)$, $\alpha-\omega=o(1)$, and $B^{-\frac{1}{2}} = o(\alpha-\omega)$, then $\frac{\hat{r}_n}{r_{\mathrm{opt}}} \overset{p}{\to} 1$ where $r_{\mathrm{opt}} = \Phi^{-1}(1-\frac{\alpha}{2})\sigma_g$.
\end{proposition}

}

\subsection{Deconvolution for estimating the sampling distribution}

{The method developed in Section \ref{sec:population_mean} is asymptotic and does not offer insight into the whole sampling distribution as the CI depends only on the sample mean and sample variance of the bootstrap estimates. 
% Although the repro sample method is a finite-sample statistical inference method, the assumptions of Theorem \ref{thm:repro} are from the asymptotic results of part 2 of Lemma \ref{lem:point-estimates} based on point estimates and a normal approximation. 
In this section, we use deconvolution to build a non-parametric and non-asymptotic estimate for the sampling distribution of $g(D)$, and we derive CIs based on the estimated sampling distribution, which is similar to percentile intervals.
}

% Instead of using the standard interval, which uses normal approximation and has potentially negative variance issues, {we propose to use deconvolution on DP bootstrap estimates to recover the distribution of the non-private bootstrap estimates and build corresponding percentile intervals as CIs.} Note that although we analyze this deconvolution method on DP bootstrap with Gaussian mechanism, it can be applied to DP bootstrap with any additive noise mechanisms, e.g., Laplace, Truncated-Uniform-Laplace (Tulap)  \citep{awan2020differentially}, canonical noise distributions \citep{awan2023canonical}.

Deconvolution is the main tool to solve the contaminated measurement problem, i.e., we want to estimate the distribution of $\{X_i\}_{i=1}^B$ while our measurement is $\{Y_i\}_{i=1}^B$ where $Y_i = X_i + e_i$ and $e_i$ is the measurement error. This is exactly the relationship between $g(D_j)$ and $\tilde g(D_j)$ since $\tilde g(D_j) = g(D_j) + \xi_j$ and $\xi_j\iid\cN(0,\frac{(\Delta(g))^2 (2-2/\ee)B}{\mu^2})$ in Lemma \ref{lem:point-estimates}.
As DP allows for all details of the privacy mechanism to be publicly revealed (except the private dataset), the distribution of the added noise can be incorporated into our post-processing without raising any privacy concerns.

The deconvolution method is more flexible as it does not require a normal approximation to build CIs compared to the point estimates in Section \ref{sec:population_mean}. It also circumvents the problem caused by the possibly negative estimate of the variance of the sampling distribution. However, it is difficult to analyze the convergence rate of deconvolution in DP bootstrap with respect to $B$ and $n$ because the distribution of added noise $e_i$ flattens when $B$ increases, and the distribution to be recovered is from bootstrap estimates {and varies for different $n$}. 
At the end of this section, we propose a signal-noise-ratio (SNR) measure as a rule of thumb for the choice of $B$ {given $n$ for use in deconvolution}.

We compared multiple numerical deconvolution methods through preliminary simulations with different settings on $n$ and $B$.
Among different deconvolution methods, we choose to use \texttt{deconvolveR} \citep{efron2016empirical} since it performs the best in our settings without tuning its hyper-parameters. We briefly summarize this method as follows. For the model $Y = X + e$, \texttt{deconvolveR} assumes that $Y$ and $X$ are distributed discretely with the sizes of their supports $|\cY|=k$ and $|\cX|=m$. Then it models the distribution of $X$ by $f(\alpha)=\ee^{Q\alpha}/c(\alpha)$ where $Q$ is an $m\times p$ structure matrix with values from the natural spline basis of order $p$, $ns(\cX,p)$, and $\alpha$ is the unknown $p$-dimensional parameter vector; $c(\alpha)$ is the divisor necessary to make $f$ sum to $1$. The estimation of the distribution of $X$ is obtained through the estimation of $\alpha$: We estimate $\alpha$ by maximizing a penalized log-likelihood $m(\alpha)=l(Y;\alpha)-s(\alpha)$ with respect to $\alpha$ {where $s(\alpha)=c_0\|\alpha\|_2$ is the penalty term with a tunable parameter $c_0$ (default 1)}, and $l(Y;\alpha)$ is the log-likelihood function of $Y$ derived from $f(\alpha)$ and the known distribution of $e$.
% (the distribution of $X$ is determined by , and the distribution of $e$ is known).

{The output of deconvolution is an estimate of the sampling distribution of non-private bootstrap estimates $g(D_j)$. Then we construct $(1-\alpha)$-CI by the $\frac{\alpha}{2}$ and $1-\frac{\alpha}{2}$ percentiles of the sampling distribution estimate.

\begin{example}
    Let $n=10000$, $B=1000$, $D=(x_1,\ldots,x_n)$, $x_i\in[0,1]=\cX$, $g(D)=\frac{1}{n}\sum_{i=1}^n x_i$, $\{D_j\}_{j=1}^B$ are $B$ bootstrap samples of $D$, $\tilde g(D_j) = g(D_j) + \xi_j$ where $\xi_j \sim \cN(0,\sigma_e^2=\frac{B}{n^2})$.
    We know that $\{\tilde g(D_j)\}_{j=1}^B$ asymptotically satisfies $(\sqrt{2-2/\ee})$-GDP. 
    In Figure \ref{fig:deconv_demo}, we generate $D$ by $x_i\iid \mathrm{Unif}(0,1)$ for three different replicates, and the `non-private bootstrap' and `private bootstrap' density functions are computed from $\{g(D_j)\}_{j=1}^B$ and $\{\tilde g(D_j)\}_{j=1}^B$ respectively. 
    Since the distribution of `private bootstrap' is much flatter than the `non-private bootstrap' due to the additional noises added for privacy, percentile intervals directly from the `private bootstrap' would be much more conservative (wider). In contrast, the `deconvolved private bootstrap' has a distribution similar to the `non-private bootstrap'. 
    % Therefore, in the rest of this paper, we always use DP bootstrap with deconvolution for building CIs.
\end{example}
}
\begin{figure}[t]
    \centering
    \includegraphics[width=\textwidth]{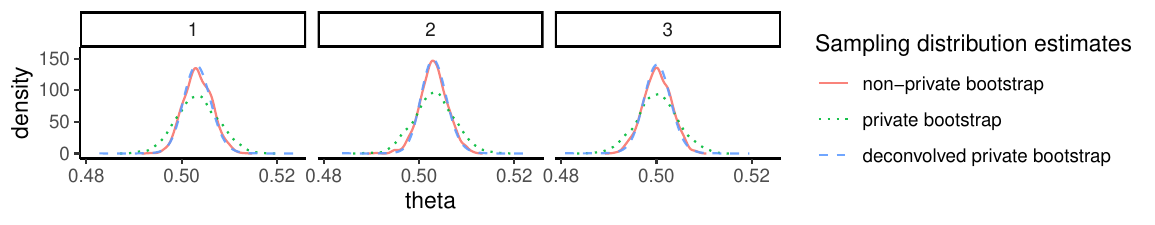}
    \caption{Comparison of the density functions from the non-private bootstrap result, the DP bootstrap result, and the deconvolved DP bootstrap result.}\label{fig:deconv_demo}
\end{figure}

\begin{remark}
Although the deconvolution method is specifically designed for additive noise mechanisms, we can potentially approximate more general mechanisms as additive ones asymptotically, e.g., exponential mechanism \citep{awan2019benefits, reimherr2019kng}. {Note that the covariance in the asymptotic distribution may depend on the confidential dataset and would have to be estimated.}
\end{remark}

\begin{remark}\label{rmk:snr}
{As it is difficult to analyze how to choose $B$ theoretically, we provide a general rule of thumb $B\in \Theta(\mu^2 n)$ and the reason is as follows. We define the oracle signal-noise ratio (SNR) as the ratio between the variance of the sampling distribution and the variance of the noise added for DP:} 1) If we have enough privacy budget, we choose the largest $B$ satisfying $\mathrm{SNR}\geq 1$; 2) If we only have a very limited privacy budget, e.g., $\mathrm{SNR}\leq 1$ for any $B\geq 10$, we choose the smallest $B$ such that $B\geq 2/\alpha$ for $1-\alpha$ CIs. Note that for the case of the population mean inference in Section \ref{sec:population_mean}, the expectation of the variance of non-private bootstrap estimates is $\frac{(n-1)\sigma_x^2}{n^2}$, and for the DP bootstrap, the added noise $\xi_j\sim \cN(0,\frac{(2-2/e)B}{\mu^2 n^2})$; therefore, $\mathrm{SNR}\geq 1$ suggests $B\in \Theta(\mu^2 n)$ which is consistent with our rate analysis $B = o(n^2)$ in Proposition \ref{prop:repro-rate} {for our asymptotic CIs.} 
More empirical results are available in the supplementary material.
\end{remark}

% \begin{remark}
% To the best of our knowledge, we are the first to use deconvolution to recover the non-private sampling distribution from DP estimates and conduct statistical inference.
% {\citet{farokhi2020deconvoluting} applied deconvolution to estimate the distribution of the sensitive data under local DP guarantees, which is different from the DP guarantee discussed in this paper.}
% \end{remark}

{\subsection{DP bootstrap algorithm for the Gaussian Mechanism}\label{sec:append_algorithm}
In this section, we summarize our DP bootstrap method for the Gaussian mechanism in Algorithm \ref{alg:dp_boot}, our method to build asymptotic CIs in Algorithm \ref{alg:dp_repro}, and the deconvolution method to obtain the DP sampling distribution estimate and non-parametric CIs in Algorithm \ref{alg:dp_dist_est}.

Note that our privacy analysis in this paper (for Algorithm \ref{alg:dp_boot}) is valid for $g: \cX^n \rightarrow \RR^d$ for any $d\in\{1,2,3,\ldots\}$, but the \texttt{deconvolveR} used in Algorithm \ref{alg:dp_dist_est} can only be applied when $d=1$. For $d\in\{2,3,\ldots\}$, one may try other deconvolution methods in Algorithm \ref{alg:dp_dist_est} or use our current procedure in each dimension separately. 

\begin{algorithm}[H]
	\caption{\texttt{DP\_bootstrap\_estimates} (with Gaussian mechanism)}\label{alg:dp_boot}
	\begin{algorithmic}[1]
		\STATE \textbf{Input} Dataset $D\in\cX^n$, statistic $g: \cX^n \rightarrow \RR$ with $\ell_2$ sensitivity $\Delta(g)$, privacy guarantee $\mu$-GDP, number of bootstrap samples $B$.\\
% 		\STATE Calculate the $\ell_2$ sensitivity of $g(D)$: $\Delta(g)=\max_{D_1,D_2\in\cX^n, D_1\simeq D_2}\|g(D_1)-g(D_2)\|_2$.
		% \STATE Let $\cM(D)=g(D) + \xi$, $\xi\sim\cN(0, \sigma_e^2), ~\sigma_e=(\sqrt{(2-2/\ee)B}) \Delta(g) / \mu$.
		\FOR {$j=1, \ldots, B$}
    		\STATE Obtain bootstrap sample (i.e., sample with replacement) $D_j\in\cX^n$ from $D$.
                \STATE DP bootstrap statistic: $\tilde{g}_j = g(D_j) + \xi_j$, $\xi_j\iid\cN(0, \sigma_e^2)$, $\sigma_e=(\sqrt{(2-2/\ee)B}) \Delta(g) / \mu$.
		\ENDFOR
		\STATE \textbf{Return} $(\tilde{g}_1,\tilde{g}_2,\ldots,\tilde{g}_B, \sigma_e^2)$ which approximately satisfies $\mu$-GDP.
	\end{algorithmic}
\end{algorithm}

{\begin{algorithm}[H]
	\caption{\texttt{DP\_bootstrap\_Asymptotic\_CI}}\label{alg:dp_repro}
	\begin{algorithmic}[1]
		\STATE \textbf{Input} Dataset $D\in\cX^n$, statistic $g: \cX^n \rightarrow \RR$, privacy guarantee $\mu$-GDP, number of bootstrap samples $B$, confidence level $1-\alpha$.\\
		\STATE Let $(\tilde{g}_1,\tilde{g}_2,\ldots,\tilde{g}_B, \sigma_e^2)=\mathtt{DP\_bootstrap\_estimates}(D, g, \mu, B)$. 
            \STATE Compute the statistics $s_1 = \frac{1}{B} \sum_{j=1}^B \tilde g_j$ and $s_2 = \frac{1}{B-1} \sum_{j=1}^B (\tilde g_j-s_1)^2$.
            \STATE Set $\omega \in (0,\alpha)$ and let $c$ be the $(\alpha - \omega)$ quantile of the $\chi^2_{B-1}$ distribution.
            \STATE Compute $\hat\sigma_{g}^2 = \max\l(0, \frac{B-1}{c}s_{2}  - \sigma_e^2\r)$ and $\hat\sigma_{\mathrm{upper}}^2 = \hat\sigma_{g}^2 + \frac{\hat\sigma_{g}^2 + \sigma_{e}^2}{B}$.
            \STATE Compute the confidence interval radius $r=\Phi^{-1}(1-\frac{\omega}{2}) \hat\sigma_{\mathrm{upper}}$.
		\STATE \textbf{Return} $(s_{1} - r,~ s_{1} + r)$ which satisfies approximately $\mu$-GDP.
	\end{algorithmic}
\end{algorithm}
}

\begin{algorithm}[H]
	\caption{\texttt{DP\_bootstrap\_deconvolution\_sampling\_distribution\_and\_CI}}\label{alg:dp_dist_est}
	\begin{algorithmic}[1]
		\STATE \textbf{Input} Dataset $D\in\cX^n$, statistic $g: \cX^n \rightarrow \RR$, privacy guarantee $\mu$-GDP, number of bootstrap samples $B$.\\
		\STATE Let $(\tilde{g}_1,\tilde{g}_2,\ldots,\tilde{g}_B, \sigma_e^2)=\mathtt{DP\_bootstrap\_estimates}(D, g, \mu, B)$. 
		\STATE DP estimate of sampling distribution of $g(D)$:  $\tilde{f}_g=\texttt{deconvolveR}((\tilde{g}_1,\tilde{g}_2,\ldots,\tilde{g}_B), \sigma_e^2)$.
            \STATE Let $\tilde{F}_g$ be the CDF corresponding to $\tilde{f}_g$.
		\STATE \textbf{Return} $(\tilde{f}_g,~ \tilde{F}_g^{-1}(\frac{\alpha}{2}),~ \tilde{F}_g^{-1}(1-\frac{\alpha}{2}))$ which satisfies approximately $\mu$-GDP.
	\end{algorithmic}
\end{algorithm}
}

\section{Simulations}\label{sec:simulation}
In this section, we use the confidence interval construction as a showcase for statistical inference with our DP bootstrap algorithm. First, we compare the non-private CI from the original bootstrap to the private CI from the recovered sampling distribution by using deconvolution on our DP bootstrap estimates. Then we compare our {deconvolution method with the asymptotic method, the method by \citet{brawner2018bootstrap}, and NoisyVar \citep{du2020differentially}, and we discuss the difference between these methods.} 
% {The algorithms used in this section are in Section \ref{sec:append_algorithm}.}

\subsection{Private CI compared to non-private CI for the population mean}\label{subsec:priv_nonpriv} 
{Consider $D=(x_1,x_2,\ldots,x_n)$ where $x_i\in[0,1]$ and $x_i\iid F_X$. We construct private CIs for $\mathbb{E}[x_i]$ with $D$. Let $g(D)=\frac{1}{n}\sum_{i=1}^n x_i$ be the non-private point estimate of $\mathbb{E}[x_i]$, and $\tilde{g}_B(D) = (\tilde{g}(D_1),\tilde{g}(D_2),\ldots,\tilde{g}(D_B))$ be DP bootstrap estimates where $\tilde{g}(D_j)={g}(D_j)+\xi_j$, $\xi_j\iid \cN(0,\frac{(2-2/\ee) B}{n^2 \mu^2})$ is Gaussian mechanism, and $D_j\iid \mathtt{boot}(D)$ is the $j$th bootstrap sample. From our privacy analysis, $\tilde{g}_B(D)$ asymptotically satisfies $\mu$-GDP. We construct private CIs for $\mathbb{E}[X]$ based on $\tilde{g}_B(D)$. 
% For different privacy guarantees $\mu=1, 0.5, 0.3, 0.1$, we choose $B=2000\mu^2$ correspondingly to have a constant SNR (defined in Remark \ref{rmk:snr}). 
The result is run with 2000 replicates where $x_i = \max(0, \min(1, z_i)), z_i \iid {N}(0.5,1)$ and $n=10000$. 

\begin{figure}[t]
    \centering
    \includegraphics[width=\linewidth]{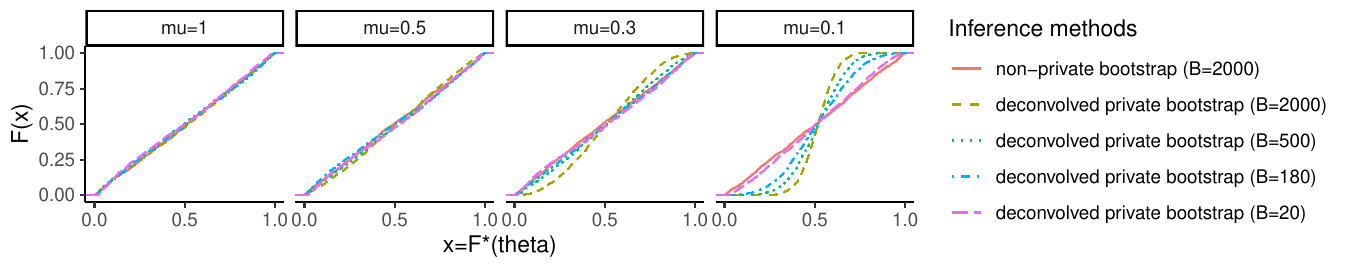}
    \caption{Coverage check of private CI with $\mu$-GDP where $\mu=1,0.5,0.3,0.1$. }\label{fig:exp_1gdp}
\end{figure}

We examine the coverage of CIs with all confidence levels in Figure \ref{fig:exp_1gdp}: Since the CI is built with the quantiles of the recovered distribution, and the coverage is determined by whether the true parameter value is in between the two quantiles, we evaluate the CDF of the recovered sampling distribution, $F^*$, at the true parameter value $\theta$, e.g., $0.05\leq F^*(\theta)\leq 0.95$ is equivalent to the 90\% CI covering $\theta$. Therefore, the coverage {at} different confidence {levels}  can be calculated by $\EE[\mathbbm{1}(p_{\mathrm{lower}}\leq F^*(\theta)\leq p_{\mathrm{upper}})]$, and we want it to be close to the nominal confidence level, $p_{\mathrm{upper}} - p_{\mathrm{lower}}$. This is achieved if the CDF of $u:=F^*(\theta)$ {is} close to the line $F(u)=u,~\forall u\in[0,1]$. In Figure \ref{fig:exp_1gdp}, we can see that our DP bootstrap result aligns with the $F(u)=u$ when $B=2000\mu^2$, similar to the non-private bootstrap, for $\mu=1, 0.5, 0.3, 0.1$, which corresponds to a constant SNR defined in Remark \ref{rmk:snr}.

In this simulation, our choice of $B$ varies from 20 to 2000. {We use the DP bootstrap with} a smaller $B$ under a stronger privacy guarantee because larger $B$ leads to smaller SNR, making deconvolution harder. If the coverage is satisfactory under many choices of $B$, e.g., $B=20,180,500,2000$ when $\mu=1$, the largest $B$ gives the shortest CI since the deconvolution accuracy is determined by $B$. 
Detailed comparisons of the width, coverage, and corresponding SNR for different choices of $B$ are provided in the supplementary material.
}

\subsection{Comparison between DP bootstrap with deconvolution and other methods}
{In this section, we compare our DP bootstrap with deconvolution to other methods, i.e., non-private bootstrap, DP bootstrap with asymptotic CI, the method by \citet{brawner2018bootstrap}, and NoisyVar \citep{du2020differentially}, under the settings we used in Section \ref{subsec:priv_nonpriv}.
}

In the non-private setting, the parametric CI for the population mean $\theta$ can be built with the $t$-statistic $t=\frac{\bar{X} - \theta}{\sqrt{s_X^2/n}}$ where $\bar{X}$ and $s_X^2$ are the sample mean and variance. Similarly, to construct a DP CI, one can obtain a DP $t$-statistic by replacing $\bar{X}$ and $s_X^2$ with their corresponding DP statistics, but this DP $t$-statistic may not follow the $t$-distribution {(or an approximate normal distribution)} because of the added noise for privacy. {To tackle this issue, our asymptotic method (Theorem \ref{thm:repro}) and the method by \citet{brawner2018bootstrap} essentially use conservative, over-estimates of $s_X^2$, which result in the over-coverage and larger width of the correpsonding CIs.} \citet{du2020differentially} adopted the idea of parametric bootstrap to construct the CI for the population mean: {They plugged in the DP sample mean and the DP sample variance to generate normally distributed samples and compute the corresponding DP sample means to estimate the sampling distribution of the DP sample mean.} We include NoisyVar in the appendix (Algorithm \ref{alg:noisyvar}) for easier reference.   

\begin{figure}[t]
        \centering
        \includegraphics[width=0.75\textwidth]{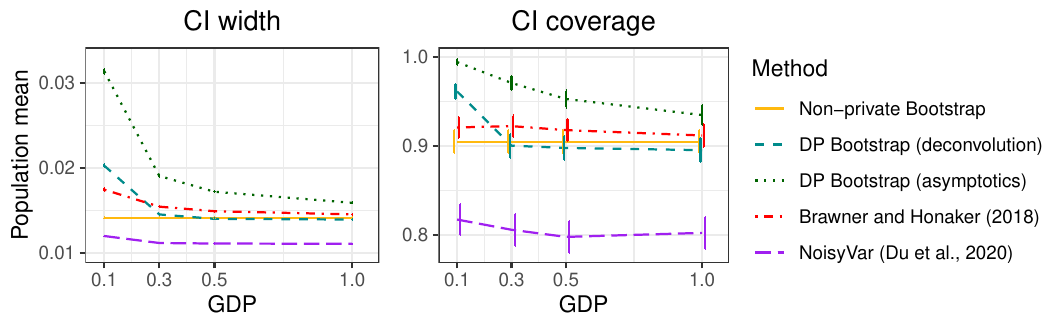}
        \caption{Coverage and width of CIs for the population mean with different privacy guarantees.  The confidence level is 90\%.}\label{tab:exp_1gdp}
\end{figure}

{In Figure \ref{tab:exp_1gdp}, the coverage and width of the 90\% CIs from the deconvolution method based on DP bootstrap are similar to the non-private bootstrap, except that under the strongest privacy guarantee $\mu=0.1$, DP bootstrap CI has over-coverage and larger width. In comparison, the asymptotic method and the method by \citet{brawner2018bootstrap} always have over-coverage and larger width even when $\mu$ is large, while NoisyVar has under-coverage and less width for all $\mu$.} 
The performance of NoisyVar is not as satisfactory as in \citep{du2020differentially} in terms of the coverage, because the distribution of our data, the clamped normal random variables, is not in the normal distribution family used by NoisyVar. Our comparison highlights the importance of non-parametric inference: The bootstrap CI does not assume the family of the sampling distribution or data distribution; therefore, our deconvolution results do not suffer from the coverage issue. 
Note that NoisyVar is also limited to building DP CIs for the population mean, while our DP bootstrap can be used on any DP statistic with additive noise mechanisms, {which we demonstrate in the real-world experiments below.}
% As a non-parametric method, the . We use deconvolution to correct the distribution of DP bootstrap estimates to the true sample distribution. 
% When the privacy guarantee is stronger, our results are more conservative, and the coverage is higher than the nominal coverage level in our simulation.

\section{Real-world experiments}\label{sec:realworld}
In this section, we conduct experiments with the 2016 Census Public Use Microdata Files (PUMF), which provide data on the characteristics of the Canadian population \citep{AB2-GDJRT8-2019}. We analyze the dependence between market income and shelter cost in Ontario by the inference of logistic regression and quantile regression under DP guarantees. We use DP bootstrap with output perturbation mechanism \citep{chaudhuri2011differentially} and compare our {asymptotic and deconvolution} method with \citet[Algorithm 5]{wang2019differentially}, Differentially Private Confidence Intervals for Empirical Risk Minimization, which we abbreviate as DP-CI-ERM\footnote{For the Line 5 in \citep[Algorithm 5]{wang2019differentially}, we replace $c$ with $0$ since in their analysis, the eigenvalues of the Hessian matrix is no less than $2c>0$ while the eigenvalues of the covariance matrix only need to be non-negative. This modification greatly improves the performance of DP-CI-ERM when $n$ is large, as overestimating the covariance matrix leads to over-coverage and wider CIs.}. Our main results are shown in Figure \ref{fig:realworld_experiment}, and more detailed comparisons are available in the supplementary material.

\subsection{Experiment settings}
The PUMF dataset contains 930,421 records of individuals, representing 2.7\% of the Canadian population. Among the 123 variables in this dataset, we choose three variables: the province or territory of current residence (named PR), the market income (named MRKINC), and the shelter cost (named SHELCO). We extract the records of MRKINC and SHELCO belonging to the people in Ontario (according to the values in PR). {After removing the records with unavailable values, the sample size is 217,360. We define the extracted data as the original dataset and analyze the relationship between MRKINC and SHELCO.

For the exploratory data analysis, we show the non-private empirical joint distribution between MRKINC and SHELCO in Figure \ref{fig:realworld_experiment}a, and we assume that their maximum values are prior information since the dataset is top-coded, as shown in Figure \ref{fig:realworld_experiment}a. We preprocess the data by scaling MRKINC and SHELCO so their ranges are $[0,1]$. 

To evaluate the performance of different statistical inference methods, we calculate the coverage and width of the CIs from 2000 simulations for each setting where the input datasets are sampled from the original dataset with replacement with size $n=1000$, 3000, 10000, 30000, 100000. We also calculate the probability that the CI of the slope parameter covers 0, which indicates that there is not enough evidence to reject the independence between MRKINC and SHELCO. The privacy guarantee is set to be $1$-GDP, the confidence level is 90\%, and we use $B=100$ for bootstrap and DP bootstrap.}

\begin{figure}[t]
        \centering
        \includegraphics[width=\textwidth]{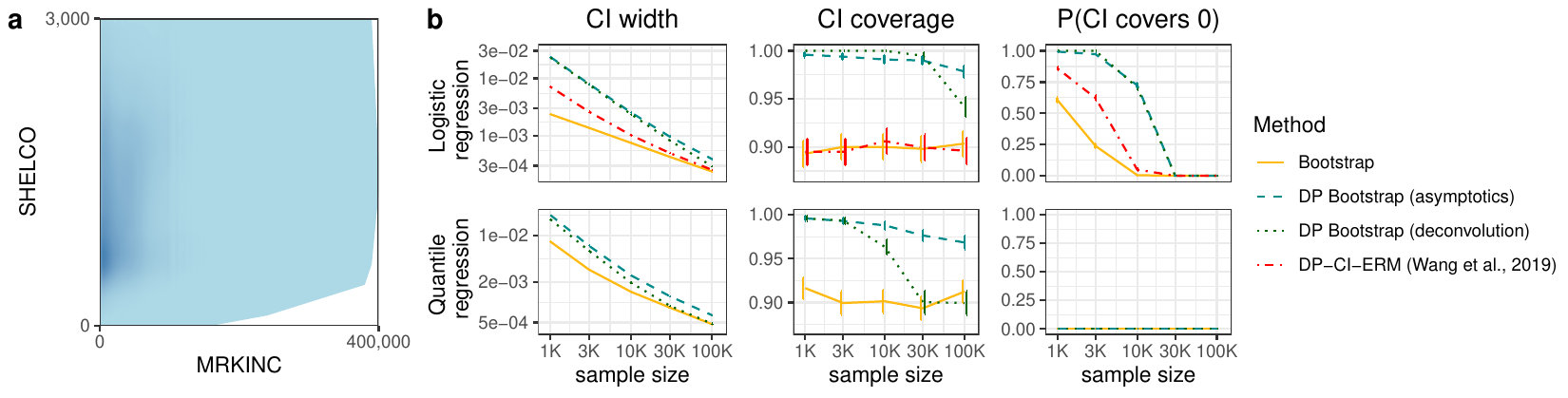}
        \caption{(a) Joint distribution between MRKINC and SHELCO in Ontario, Canada. The polygon region is the convex hull of all data points. (b) Results of 90\% CIs for the slope parameters in logistic regression and quantile regression between MRKINC and SHELCO. Note that DP-CI-ERM cannot be used in the inference of quantile regression.}\label{fig:realworld_experiment}
\end{figure}

{
\subsection{Logistic regression}\label{sec:logistic}
We set the response $y_i=1$ if $\mathrm{SHELCO}\geq 0.5$, otherwise $y_i=-1$. In logistic regression, the true model is $P(Y|X)=\frac{1}{1+\exp(-\theta^\intercal X\cdot Y)}$, and the empirical risk minimizer (ERM), also the maximum likelihood estimate of $\theta$, is $\hat\theta = \mathrm{argmin}_\theta R(\theta)$ where $R(\theta):=\frac{1}{n}\sum_{i=1}^n -\log(P(y_i|x_i))$. 

To obtain DP estimates, we implement the output perturbation mechanism following \citep[Algorithm 5]{wang2019differentially}, which replaces $R(\theta)$ by a regularized empirical risk $R(\theta) + c\|\theta\|_2^2$ and adds noise to the output: $\tilde\theta=\hat\theta + \xi$ where $\hat\theta=\mathrm{argmin}_\theta (R(\theta) + c\|\theta\|_2^2)$. As the sensitivity of the regularized ERM is $\Delta(\hat\theta) = \frac{1}{nc}$ \citep{wang2019differentially}, we use Gaussian mechanism, $\xi\sim\cN(0,\frac{1}{(\mu n c)^2 })$, then $\tilde\theta$ satisfies $\mu$-GDP; to satisfy the constraint, $\|x_i\|_2^2\leq 1$, in the sensitivity analysis, we let the covariate be $x_i=(1/\sqrt{2},\mathrm{MRKINC}/\sqrt{2})$. Following \citep{wang2019differentially}, we define the true parameter $\theta=(\theta_1,\theta_2)\in\mathbb{R}^2$ as the regularized ERM estimated with the original dataset under the same $c$. We build CI for the slope parameter $\theta_2$: If the $90\%$ CI does not cover 0, we are confident that MRKINC is related to SHELCO.

The results are shown in the upper figures of Figure \ref{fig:realworld_experiment}b where $c=1$. The CIs by DP bootstrap {from both deconvolution and asymptotic methods} are wider than the ones by DP-CI-ERM when the sample size $n$ is small: As a non-parametric method, DP bootstrap {is applicable to more general settings but} does not fully utilize the structure of the private ERM as opposed to DP-CI-ERM, so its CIs are often not as tight as the ones by DP-CI-ERM. Both types of private CIs are as wide as non-private bootstrap CI when $n=100000$, do not suffer from the under-coverage issue, and have $P(\mathrm{CI~covers~0})\approx 0$ when $n\geq 30000$.

\subsection{Quantile regression}
We use quantile regression as an example to demonstrate the advantage of DP bootstrap as a non-parametric method {being applicable to general settings where DP-CI-ERM is not}. We set the response $y_i=\mathrm{SHELCO}$ and the covariate $x_i=(1,\mathrm{MRKINC})$. Following \citep{reimherr2019kng}, we assume that the conditional quantile function of $Y$ given $X$ is $Q_{Y|X}(\tau)=X^\intercal \theta_\tau$, we estimate $\theta_\tau$ also by ERM with the objective function $R(\theta)=\frac{1}{n}\sum_{i=1}^n l(z_i)$ where $l(z_i) = (\tau-\mathbbm{1}(z_i\leq 0)) z_i$ and $z_i=y_i-x_i^\intercal \theta$. Similar to our experiment with logistic regression, we define the true parameter $\theta=(\theta_1,\theta_2)\in\mathbb{R}^2$ as the regularized ERM estimated with the original dataset under the same {regularization parameter} $c$. By \citep[Lemma 7]{chaudhuri2011differentially}, the sensitivity of the regularized ERM, $\hat\theta = \mathrm{argmin}_\theta \{R(\theta) + c\|\theta\|_2^2\}$, is bounded by $\Delta(\hat\theta)=\frac{\max\{2\tau, 2(1-\tau), \sqrt{2}\}}{2nc}$; the derivation is in Appendix \ref{app:quantile_sensitivity}. Then $\tilde\theta=\hat\theta + \xi$ satisfies $\mu$-GDP when $\xi\sim \mathcal{N}(0, \frac{\Delta(\hat\theta)^2}{\mu^2})$ (Gaussian mechanism.) 

To the best of our knowledge, this is the first result of building private CIs for the coefficients of quantile regression. As DP-CI-ERM uses Taylor expansion of the gradient of $R(\theta)$ to characterize the difference between the estimate $\hat\theta$ and the true parameter $\theta^*$, i.e., $\theta^* - \hat\theta \approx H[R(\hat\theta)]^{-1}(\nabla R(\theta^*)-\nabla R(\hat\theta))$, it is not usable in quantile regression as the Hessian of $R(\theta)$ is always $0$ when it exists. The results for DP bootstrap and non-private bootstrap are shown in the lower figures of Figure \ref{fig:realworld_experiment}b where we set $c=1$ and $\tau=0.5$. Similar to our experiment in logistic regression, DP bootstrap never suffers from the under-coverage issue. {The deconvolution CIs from DP bootstrap} perform similarly to non-private bootstrap CIs when $n\geq 30000$. We can see that $\mathbb{P}(\mathrm{CI~covers~0})\approx 0$ when $n\geq 1000$; therefore, we are confident that MRCINC and the median of corresponding SHELCO are not independent with 90\% confidence.

}
\section{Conclusion}\label{sec:conclusion}

Our analysis of the DP bootstrap provides a new perspective on resampling in DP by considering the output distribution as a mixture distribution which gives a tractable lower bound for the DP bootstrap in $f$-DP. Furthermore, our composition result for the DP bootstrap using the Gaussian mechanism gives a simple asymptotic privacy guarantee {confirmed by our numerical evaluation results of the exact cumulative privacy cost,} and it highlights the minimal cost of the DP bootstrap.
{For DP statistical inference, we show that the sample mean of the DP bootstrap estimates is a consistent point estimator, and we propose an asymptotic method and a deconvolution method to build CIs. We prove an asymptotic coverage guarantee of the asymptotic CIs, and show that their average width enjoys the optimal convergence rates. To improve the finite sample performance, we use deconvolution to construct CIs, and to the best of our knowledge, we are the first to use deconvolution\footnote{\citet{farokhi2020deconvoluting} applied deconvolution to estimate the distribution of the sensitive data under local DP guarantees, which is different from the DP guarantee discussed in this paper.} to recover the non-private sampling distribution from DP estimates and conduct statistical inference.}
% We are the first to use deconvolution to recover the non-private sampling distribution from DP bootstrap estimates for statistical inference. 
Our simulations and experiments show that the CIs generated by the deconvolved distribution achieve the nominal coverage, and our results are not only comparable to existing methods like NoisyVar and DP-CI-ERM but also applicable to the inference problems such as quantile regression where existing methods cannot be used.

One direction of future work is on improving the privacy analysis: The lower bound in Theorem \ref{thm:fdp_single_boot} could be tightened by considering all $\alpha_i$ jointly rather than individually. 

For statistical inference using the DP bootstrap, the choice of the number of bootstrap samples $B$ can potentially be further optimized for a given sample size $n$ and privacy parameter $\mu$. One may pursue finite sample utility results for the deconvolved distribution or the CIs obtained by the DP bootstrap deconvolution procedure.  Furthermore, our statistical results are limited to the Gaussian mechanism, which may not be an ideal mechanism, due to its additive nature and the need to calculate $\ell_2$-sensitivity. While our statistical approaches should be easily extended to allow for other additive noise mechanisms, non-additive noise mechanisms may require new inference techniques. We also notice in Figure \ref{tab:exp_1gdp} and Figure \ref{fig:realworld_experiment}b that with a strong privacy guarantee or small $n$, our CI may still be wider and have higher coverage than the non-private CI or the private CI from other methods, indicating that the deconvolution procedure can be further optimized to get a tighter estimate and improve the width of the CIs while maintaining the nominal coverage. Apart from the standard interval and the percentile interval, there are other inference methods based on the bootstrap estimates such as $\mathrm{BC}_\mathrm{a}$ (bias-corrected and accelerated) \citep{efron1987better} or ABC (approximate bootstrap confidence) \citep{diciccio1992more}, we leave it to future work to investigate the private versions of these methods. 

We can also use the DP bootstrap framework for high-dimensional inference. Our privacy analysis applies to any mechanism, including those multivariate mechanisms, e.g., the Gaussian mechanism on each dimension of the estimate. Future work in this direction is to apply existing methods of multivariate deconvolution \citep{youndje2008optimal, hazelton2009nonparametric, sarkar2018bayesian} and bootstrap \citep{hall1987bootstrap, van2005multivariatr} to our DP bootstrap framework.

\acks{This work was supported in part by the National Science Foundation [NSF grants no. SES-2150615, no. DMS-2134209, and no. CNS-2247795], the Office of Naval Research [ONR award no. N00014-22-1-2680], and Optum AI and CISCO research grants.  The authors would like to thank Dr. Chendi Wang for the discussion about the relationship between the advanced joint convexity property of $f$-DP and our Theorem \ref{thm:mixlow} and \ref{thm:fdp_single_boot}.}

\newpage
\appendix

\section{Proofs for Section \ref{sec:privacy_analysis}}\label{sec:append_proof}
In this section, we provide the proofs for the theorems and propositions in Section \ref{sec:privacy_analysis}.

\subsection{Proofs for Section \ref{sec:f_DP}}

We first restate the $(\ep,\delta)$-DP results in \citep{balle2018privacy} and provide some useful results for the proof of Proposition \ref{prop:epde_f_boot}.

% \balle*

\begin{restatable}[Theorem 10 in \citep{balle2018privacy}]{theorem}{balle}\label{thm:balle}
    Given $\ep\geq 0$, assume $\cM$ satisfies $(\ep,\delta_{\cM, i}(\ep))$-DP with group size $i$. Let $p_i={n\choose i}(\frac{1}{n})^i(1-\frac{1}{n})^{n-i}$ and $\ep'=\log(1+(1-p_0)(\ee^\ep-1))$, then $\cM \circ \mathtt{boot}$ satisfies $(\ep',\delta_{\cM \circ \mathtt{boot}}(\ep'))$-DP where $\delta_{\cM \circ \mathtt{boot}}(\ep')= \sum_{i=1}^n p_i\delta_{\cM,i}(\ep)$.
\end{restatable}

\begin{definition} 
	Let $f$ be a tradeoff function, $\bar{x} = \inf\{x\in[0,1]:-1\in\partial f(x)\}$. The symmetrization operator 
	{which maps a possibly asymmetric tradeoff function to a symmetric tradeoff function is} defined as
	\[\mathrm{Symm}(f):= \left\{
	\begin{array}{ll}
	\min\{f,f^{-1}\}^{**}, &~\text{if}~\,\, \bar{x}\leq f(\bar{x}),\\
	\max\{f,f^{-1}\}, &~\text{if}~\,\, \bar{x}>f(\bar{x}).
	\end{array}
	\right.\]
\end{definition}
\begin{proposition}[Proposition E.1 in \citealp{dong2021gaussian}]\label{prop:symm}
	Let $f$ be a tradeoff function. Suppose a mechanism is $(\ep,1+f^*(-e^{\ep}))$-DP for all $\ep\geq 0$, then it is $\mathrm{Symm}(f)$-DP.
\end{proposition}

% {Now we are ready to prove Proposition \ref{prop:epde_f_boot}.}
% \propepdefboot*

\begin{proof}[Proof of Proposition \ref{prop:epde_f_boot}]
% This proposition is a transformation from the $(\ep,\delta)$-DP result in \citep{balle2018privacy} to $f$-DP. We use the primal-dual mapping between the $(\ep,\delta)$-DP and $f$-DP \citep{dong2021gaussian}: $\delta(\ep)=1+f^*(-\ee^\ep)$.
We use $f$-DP to restate the $\delta_{\cM \circ \mathtt{boot}}(\ep')$ in Theorem \ref{thm:balle}:
Let $p=1-p_0$, $\delta_{\cM \circ \mathtt{boot}}(\ep') \leq \sum\limits_{i=1}^n p_i\delta_{\cM,i}(\ep)= \sum\limits_{i=1}^n p_i(1+f_{\cM,i}^*(-\ee^\ep)) 
    =p\l(1+\sum\limits_{i=1}^n \frac{p_i}{1-p_0} f_{\cM,i}^*(-\ee^\ep)\r)$ where
    $\cM$ satisfies $f_{\cM,i}$-DP with group size $i$ and $f_{\cM,i}$ is symmetric.

From the Supplement to \citep{dong2021gaussian} (Page 42), we know that $\ep'=\log(1-p + p\ee^\ep)$ and $\delta'=p\big(1+f^*(-\ee^{\ep})\big)$ can be re-parameterized into $\delta'=1+f_p^*(-\ee^{\ep'})$ where $f_p = pf+(1-p)\Id$. 
For symmetric $f$, we have $\mathrm{Symm}(f_p)=C_p(f)$ since $\bar{x}\leq f_p(\bar{x})$ where $\bar{x} = \inf\{x\in[0,1]:-1\in\partial f(x)\}$. Using Proposition \ref{prop:symm}, we have
$f_{\cM \circ \mathtt{boot}} = C_p\l(\l(\sum_{i=1}^n \frac{p_i}{1-p_0} f_{\cM,i}^*\r)^*\r).$
\end{proof}

\begin{restatable}{lemma}{lemmixwelldef}\label{lem:mix_welldef}
For $i=1,2,\ldots,k$, let $f_i$ be tradeoff functions and $p_i\in(0,1]$ satisfying $\sum_{i=1}^k p_i=1$. We write $\underline{f}=(f_1,\ldots,f_k)$ and $\underline{p}=(p_1,\ldots,p_k)$.
\begin{enumerate}
    \item $\mathrm{mix}(\underline{p}, \underline{f}): [0,1] \rightarrow [0,1]$ is a well-defined tradeoff function.
    \item If the tradeoff functions $f_i$ are all symmetric, then $\mathrm{mix}(\underline{p}, \underline{f})$ is symmetric.
\end{enumerate}
\end{restatable}

We first state some useful properties of tradeoff functions for proving Lemma \ref{lem:mix_welldef}.

% \defmixtradeoff*

\begin{proposition}[Proposition 1 in \citealp{dong2021gaussian}]\label{prop:convex-tradeoff}
    A function $f : [0, 1] \rightarrow [0, 1]$ is a tradeoff function if and only if $f$ is convex, continuous, non-increasing, and $f(x) \leq 1-x$.
\end{proposition}

\begin{proposition}\label{prop:tradeoffnewproperty}
For any tradeoff function $f$, 
% we denote the inverse of $f$ by $f^{-1}(\alpha):=\inf\{t\in[0,1]: f(t)\leq \alpha\}$.
\begin{enumerate}
    \item $f(\alpha)$ is strictly decreasing for $\alpha \in \{\alpha: f(\alpha) > 0\}$.
    \item $0\notin \partial f(\alpha)$ if and only if $\alpha \in \{\alpha: f(\alpha) > 0\}$.
    \item If $f^{-1}(f(y))\neq y$, then $f(y)=0$. 
    \item If $f=f^{-1}$ and there exists $C<0$ such that $C\in\partial f(\alpha)|_{\alpha=f(\alpha_0)}$, then we have $f(f(\alpha_0))=\alpha_0$ and $\frac{1}{C}\in \partial f(\alpha)|_{\alpha=f(\alpha_0)}$.
    \item There is exactly one $\bar\alpha$ such that $f(\bar\alpha)=\bar\alpha$.
    \item If $f=f^{-1}$ and $f(\bar\alpha)=\bar\alpha$, then $-1\in\partial f(\alpha)|_{\alpha=\bar\alpha}$.
    \item If $f=f^{-1}$ and $-1\in\partial f(\alpha)|_{\alpha=\alpha_0}$, then $-1\in\partial f(\alpha)|_{\alpha=f(\alpha_0)}$.
\end{enumerate}
\end{proposition}
\begin{proof}[Proof of Proposition \ref{prop:tradeoffnewproperty}]
We provide the proofs for each property of $f$ below.
\begin{enumerate}
    \item If there are $0\leq \alpha_1<\alpha_2\leq 1$ such that $f(\alpha_1)=f(\alpha_2)=\beta > 0$, since $g$ is convex, we will have $f(\alpha)\geq \beta ~\forall \alpha \geq \alpha_1$. Therefore, $f(1)\geq \beta > 0$ which contradicts with $f(1)\leq 1-1=0$. Since $f(\alpha)$ is not increasing, it is strictly decreasing for $\alpha \in \{\alpha: f(\alpha) > 0\}$.
    \item If $0\in \partial f(\alpha)$ and $f(\alpha)>0$, we have $f(y)\geq f(\alpha)>0 ~\forall y\in[0,1]$. This contradicts with the fact that $f(1)=0$. If $0\notin \partial f(\alpha)$ and $f(\alpha)=0$, there exists $y\in[0,1]$ such that $f(y)<0\cdot(y-\alpha)+f(\alpha)=0$ which contradicts with the fact that $f:[0,1]\rightarrow[0,1]$.
    \item If $f^{-1}(f(y))=\inf\{t\in[0,1]: f(t)\leq f(y)\}<y$, since $f$ is not increasing, $f(f^{-1}(f(y))) = f(y)$, which holds only when $f(y)=0$ (otherwise $f(y)$ is strictly decreasing).
    \item From the fact that $C$ is a sub-differential value of $f$ at $\alpha_0$, we have $f(z)\geq C(z-\alpha_0)+f(\alpha_0) ~\forall z\in[0,1]$. If there exists $z$ such that $z < \alpha_0$ and $f(z)\leq f(\alpha_0)$, it contradicts with $f(z)\geq C(z-\alpha_0)+f(\alpha_0)>f(\alpha_0)$. Therefore,
    $$f(f(\alpha_0))=f^{-1}(f(\alpha_0))=\inf\{t\in[0,1]: f(t)\leq f(\alpha_0)\}=\alpha_0.$$
    We prove $\frac{1}{C}\in\partial f(\alpha)|_{\alpha=f(\alpha_0)}$ by showing $f(y)\geq \frac{1}{C}(y-f(\alpha_0))+f(f(\alpha_0)) ~\forall y\in[0,1]$. Since $C\in(-\infty,0)$, if $f(y)\geq \frac{1}{C}(y-f(\alpha_0))+f(f(\alpha_0))$ holds when $y=f(0)$, then it also holds for $y> f(0)$ as $f(y)=0=f(f(0))\geq \frac{1}{C}(f(0)-f(\alpha_0))+\alpha_0\geq \frac{1}{C}(y-f(\alpha_0))+\alpha_0$. For $y\in[0,f(0)]$, we define $z:=f(y)$. Since $f=f^{-1}$ and $f$ is strictly decreasing when $f>0$, we know that $f(y)>0$ and $f(z)=f^{-1}(f(y))=y$ for $y\in[0,f(0))$. We also know that $f(z)=y$ when $y=f(0)$ since $z=f(y)=f(f(0))=0$. From the fact that $C$ is a sub-differential value of $f$ at $\alpha_i$, we have $f(z)\geq C(z-\alpha)+f(\alpha)$. Now we show $f(y)\geq \frac{1}{C}(y-f(\alpha_0))+f(f(\alpha_0)) ~\forall y\in[0,f(0)]$ through the analysis below.
    \begin{itemize}
        \item If $f(0)\geq y>f(\alpha_0)$, we have $z=f(y)<\alpha_0$, $f(z)=y$, and $C\geq \frac{f(z)-f(\alpha_0)}{z-\alpha_0}$. Therefore, $\frac{1}{C}(y-f(\alpha_0))+f(f(\alpha_0)) \leq \frac{z-\alpha_0}{f(z)-f(\alpha_0)}(y-f(\alpha_0))+\alpha_0=z=f(y)$.
        \item If $0\leq y<f(\alpha_0)$, we have $z=f(y)>\alpha_0$, $f(z)=y$, and $C\leq \frac{f(z)-f(\alpha_0)}{z-\alpha_0}$. Therefore, $\frac{1}{C}(y-f(\alpha_0))+f(f(\alpha_0)) \leq \frac{z-\alpha_0}{f(z)-f(\alpha_0)}(y-f(\alpha_0))+\alpha_0=z=f(y)$.
        \item If $y=f(\alpha_0)$, we have $\frac{1}{C}(y-f(\alpha_0))+f(f(\alpha_0)) = f(y)$.
    \end{itemize}
    \item If there are $\bar\alpha_1 < \bar\alpha_2$ that $\bar\alpha_1=f(\bar\alpha_1)$ and $\bar\alpha_2=f(\bar\alpha_2)$, then since $f$ is non-increasing, we have $f(\bar\alpha_1)\geq f(\bar\alpha_2)$ which contradicts with $\bar\alpha_1 < \bar\alpha_2$. Since $f(\alpha)$ is continuous and $f(0)-0\geq0-0= 0$, $f(1)-1=-1<0$, there exists $\bar\alpha\in[0,1]$ such that $f(\bar\alpha)-\bar\alpha=0$.
    \item 
    If $-1\notin\partial f(\bar\alpha)$, there exists $y\in[0,1]$ such that $f(y)<-(y-\bar\alpha)+f(\bar\alpha)$ and $f(f(y))=y$ because $f=f^{-1}$ (if $f(f(y))\neq y$, we have $f(y)=0$, then we can replace $y$ with $f(0)\leq y$ and we still have $f(y)<-(y-\bar\alpha)+f(\bar\alpha)$.) Therefore, $(y,f(y))$ and $(f(y),y)$ are both on the curve of $f$. Since $y\neq \bar\alpha$, we know $y\neq f(y)$. Without of the loss of generality, we assume that $y> f(y)$. Then we know that $y>\bar\alpha>f(y)$ since otherwise we will have contradictions: $\bar\alpha\geq y>f(y)\geq f(\bar\alpha)=\bar\alpha$ or $\bar\alpha\leq f(y)<y=f(f(y))\leq f(\bar\alpha)=\bar\alpha$. We denote $q=\frac{\bar\alpha-f(y)}{y-f(y)}>0$ and $1-q=\frac{y-\bar\alpha}{y-f(y)}>0$. Then $\bar\alpha=qy+(1-q)f(y)$, and by the convexity of $f$, we have $\bar\alpha = f(\bar\alpha)\leq qf(y)+(1-q)f(f(y))=qf(y)+(1-q)y$. Therefore, $f(\bar\alpha) + \bar\alpha \leq (qy+(1-q)f(y)) + (qf(y)+(1-q)y)=y+f(y)$. But from $f(y)<-(y-\bar\alpha)+f(\bar\alpha)$, we know $f(\bar\alpha) + \bar\alpha > y+f(y)$, which leads to a contradiction. Therefore, we have $-1\in\partial f(\bar\alpha)$.
    \item As $-1\in\partial f(\alpha)|_{\alpha=\alpha_0}$, we have $f(y)\geq -(y-\alpha_0)+f(\alpha_0) ~\forall y\in[0,1]$ and $f(f(\alpha_0))=\alpha_0$. Therefore, $f(y)\geq -(y-f(\alpha_0))+f(f(\alpha_0)) ~\forall y\in[0,1]$ and we have $-1\in\partial f(\alpha)|_{\alpha=f(\alpha_0)}$.
\end{enumerate}
\end{proof}

\begin{proof}[Proof of Lemma \ref{lem:mix_welldef}]
For part 1, first we show that for every $\alpha\in(0,1)$, there exists $C\in(-\infty,0]$ such that $\alpha\in A(C)$.
% and the corresponding $\alpha_i\in A_i(C)~\forall i\in\{1,2,\ldots,k\}$, $\sum_{i=1}^k p_i \alpha_i=\alpha$.
Since each $f_i$ is convex and non-increasing, its sub-differential $\partial f_i(\alpha_i)$ is in $(-\infty, 0]$ and non-decreasing with respect to $\alpha_i$. Therefore, for any $-\infty < C_1 < C_2 \leq 0$, we have $a_1 \leq a_2 ~\forall a_1\in A_i(C_1), a_2\in A_i(C_2)$, and for any $0 < a_1 < a_2 < 1$, we have $C_1 \leq C_2 ~\forall C_1\in \partial f_i(a_1), C_2\in \partial f_i(a_2)$. We name these two properties as the monotonicity of the sub-differential mapping.

Since each $f_i$ is convex and continuous, its sub-differential $\partial f_i(\alpha_i)$ is also continuous in the sense that for any $\ep > 0$, there exists $\delta > 0$ such that $\partial f_i(\alpha_i')\subset \partial f_i(\alpha_i) + (-\ep,\ep)$ whenever $\|\alpha_i'-\alpha_i\|< \delta$ (see Exercise 2.2.22(a) in \citep{borwein2010convex}). 

From the continuity and monotonicity of the sub-differential mapping, we know that for any $C\in(-\infty,0]$, $A_i(C)$ is a closed {interval}, e.g., $[a,b]$. Note that $A_i(C)$ is always nonempty: let $C_{i,\mathrm{range}} := \bigcup_{\alpha\in(0,1)} \partial f_i(\alpha)$; if $C < C_i ~\forall C_i\in C_{i,\mathrm{range}}$, we have $A_i(C)=\{0\}$; if $C > C_i ~\forall C_i\in C_{i,\mathrm{range}}$, we have $A_i(C)=\{1\}$; if $C \in C_{i,\mathrm{range}}$, there must be an $\alpha\in(0,1)$ such that $C \in \partial f_i(\alpha)$. By the same reasoning, we also have $(0,1)\subset \bigcup_{C\in(-\infty,0]} A_i(C)$ for any $i$. As $A(C)=\{\sum_{i=1}^k p_i \alpha_i | \alpha_i\in A_i (C)\}$, we have $(0,1)\subset \bigcup_{C\in(-\infty,0]} A(C)$, and we also have the monotonicity of $A(C)$ with respect to $C$.

Next, we show that $\mathrm{mix}(\underline{p}, \underline{f})$ is a {well-defined} function, i.e., for a given $\alpha$, although there could be multiple choices of $\{\alpha_i\}_{i=1}^k$ such that $\sum_{i=1}^k p_i \alpha_i=\alpha$, we will obtain the same value of $\sum_{i=1}^k p_i f_i(\alpha_i)$ for all choices. Consider two choices, $\{\alpha_i\}_{i=1}^k$ and $\{\alpha_i'\}_{i=1}^k$, that correspond to the same $\alpha$. Let $C$ and $C'$ correspond to $\{\alpha_i\}_{i=1}^k$ and $\{\alpha_i'\}_{i=1}^k$ respectively. As there exist $i\neq j$ such that $\alpha_i < \alpha_i'$ and $\alpha_j > \alpha_j'$ (since $\sum_{i=1}^k p_i \alpha_i=\sum_{i=1}^k p_i \alpha_i'$ and $p_i> 0$ for $i=1,2,\ldots,k$), from the monotonicity of the sub-differential mapping, we know that $C=C'$ since $C\leq C'$ and $C'\leq C$. {For $\{\alpha_i\}_{i=1}^k$ and $\{\alpha_i'\}_{i=1}^k$, since $\partial f_i(\alpha_i)=\partial f_i(\alpha_i') = C$, we have $f_i(\alpha_i)\geq C(\alpha_i-\alpha_i') +f_i(\alpha_i')\geq f_i(\alpha_i)$ for $i=1,2,\ldots,n$. Therefore, $f_i(\alpha_i) - f_i(\alpha_i') = C(\alpha_i-\alpha_i')$. As we know $\sum_{i=1}^k p_i \alpha_i=\sum_{i=1}^k p_i \alpha_i'$, we have $\sum_{i=1}^k p_i f_i(\alpha_i)-\sum_{i=1}^k p_i f_i(\alpha_i')=C\sum_{i=1}^k p_i(\alpha_i-\alpha_i')=0$, which means that $\mathrm{mix}(\underline{p}, \underline{f})$ is well-defined.}

Finally, we show that $\mathrm{mix}(\underline{p}, \underline{f})$ is a tradeoff function. 

Let $f=\mathrm{mix}(\underline{p}, \underline{f})$. We can see that $f(x)\in[0,1]~\forall x\in[0,1]$ and $f(x) \leq 1-x$. 
We also know $f$ is non-increasing because of the monotonicity of $A(C)$, the monotonicity of the sub-differential mapping, and $f_i$ being non-increasing. 

Now we prove that $f$ is continuous. For a fixed $\alpha$ and $\delta > 0$, we can find the $\{\alpha_i\}_{i=1}^k$ corresponding to this $\alpha$, and find $\ep_i$ such that $|f_i(\alpha_i')- f_i(\alpha_i)| < \delta$ whenever $|\alpha_i'-\alpha_i|<\ep_i$; then we let $\ep = \min_{i\in\{1,2,\ldots,k\}} \ep_i p_i$, and we have $|f(\alpha')-f(\alpha)|<\delta$ whenever $|\alpha'-\alpha|< \ep$. To prove this, without loss of generality, we assume $\alpha < \alpha'$, and we can find $\{\alpha_i\}_{i=1}^k$ and $\{\alpha_i'\}_{i=1}^k$ corresponding to $\alpha$ and $\alpha'$ respectively, where $\alpha_i \leq \alpha_i'$ for $i=1,2,\ldots,k$. Then if $\alpha'-\alpha<\ep$, we must have $\alpha_i'-\alpha_i<\ep_i$ for $i=1,2,\ldots,k$, therefore $|f(\alpha')-f(\alpha)|<\delta$.

{Now we prove that $f$ is convex. By the definition of convexity, we only need to show that for any $\alpha,\alpha',t\in[0,1]$, we have $t f(\alpha) + (1-t) f(\alpha') \geq f(t\alpha +(1-t)\alpha')$. From the construction of $\mathrm{mix}(\underline{p}, \underline{f})$, we can find $\{\alpha_i\}_{i=1}^k$, $\{\alpha_i'\}_{i=1}^k$, and $\{\tilde \alpha_i\}_{i=1}^k$ with their matched \mbox{sub-differential} being $C$, $C'$, and $\tilde C$ corresponding to $\alpha$, $\alpha'$, and $\tilde\alpha=t\alpha + (1-t)\alpha'$ respectively. Then  
$$
\begin{aligned}
    t f(\alpha) + (1-t) f(\alpha') &= \sum_{i=1}^k p_i (t f_i(\alpha_i) + (1-t) f_i(\alpha'_i)) 
    \geq \sum_{i=1}^k p_i f_i(t\alpha_i + (1-t)\alpha'_i) \quad(\textrm{convexity})\\
    &\geq \sum_{i=1}^k p_i (\tilde C(t\alpha_i + (1-t)\alpha'_i - \tilde \alpha_i) + f_i(\tilde \alpha_i)) \quad(\textrm{since~}\tilde C\in \partial f_i(\tilde \alpha_i)) \\
    &= \tilde C (t\alpha + (1-t)\alpha' - \tilde \alpha)  + f(t\alpha +(1-t)\alpha') = f(t\alpha +(1-t)\alpha'). \\
\end{aligned}
$$
Therefore, $f$ is convex, and we have proved that $f$ is a tradeoff function.}

For part 2, let $g = \mathrm{mix}(\underline{p}, \underline{f})$. 
We prove that $g$ is symmetric by showing $g^{-1} = \mathrm{mix}(\underline{p}, \underline{f})$.

By definition, $g^{-1}(\beta) = \inf\{\alpha\in[0,1]: g(\alpha)\leq \beta\}$. By the construction of $\mathrm{mix}(\underline{p}, \underline{f})$, for each $\alpha\in[0,1]$, there exist a constant $C$ and $\{\alpha_i\}_{i=1}^k$ such that $C\in\partial f_i(\alpha_i)$, $\alpha=\sum_{i=1}^k p_i \alpha_i$, and $g(\alpha)=\sum_{i=1}^k p_i f_i(\alpha_i)$.

If $\beta=0$, then $g(\alpha)=0$, and $f_i(\alpha_i)=0$. Therefore, $g^{-1}(0)=\inf\{\alpha\in[0,1]:\alpha=\sum_{i=1}^k p_i\alpha_i, f_i(\alpha_i)=0\}$. Since $f_i$ is symmetric, we have that $\inf\{\alpha_i\in[0,1]:f_i(\alpha_i)=0\}=f_i^{-1}(0)$. Therefore, $g^{-1}(0)=\sum_{i=1}^k p_i f_i^{-1}(0)=\sum_{i=1}^k p_i f_i(0)=g(0)$.

If $\beta \geq g(0)$, then from the definition of $g^{-1}$, we have $g^{-1}(\beta)=0$, {and we need to prove $g(\alpha)=0$ for $\alpha\geq g(0)$.  From the construction of $\mathrm{mix}(\underline{p}, \underline{f})$, we have $g(\alpha)=\sum_{i=1}^k p_if_i(\alpha_i)$ where $\alpha=\sum_{i=1}^k p_i \alpha_i$.  Let $\alpha \geq g(0)=\sum_{i=1}^k p_if_i(0)$. Then, if there exists $i$ that $\alpha_i> f_i(0)$ which means that $f_i(\alpha_i)=0$, then $\partial f_i(\alpha_i)=\{0\}$; therefore, $\partial f_j(\alpha_j)=\{0\}$ and $f_j(\alpha_j)=0$ for $j=1,2,\ldots,k$. We have $g(\alpha)=0$. If $\alpha_i \leq f_i(0)$ for $i=1,2,\ldots,k$, since $\sum_{i=1}^k p_i \alpha_i = \alpha \geq \sum_{i=1}^k p_if_i(0)$, we have $\alpha_i=f_i(0) ~\forall i$, which means $f_i(\alpha_i)=0 ~\forall i$; therefore, $g(\alpha)=0$.}

If $g(0)> \beta > 0$, since $g$ is a tradeoff function, there exists only one $\alpha\in[0,1]$ such that $g(\alpha)=\beta$. From the construction of $g$, i.e., $g(\alpha)=\sum_{i=1}^k p_i f_i(\alpha_i) =\beta > 0$, for the $\{\alpha_i\}_{i=1}^k$ corresponding to $\alpha$, there exists $i_0$ such that $f_{i_0}(\alpha_{i_0}) > 0$. Since $f_{i_0}(\alpha_{i_0}) > 0$, we have $0\notin\partial f_{i_0}(\alpha_{i_0})$. Therefore, the corresponding constant $C$ that $C\in\partial f_i(\alpha_i)$ for $i=1,2,\ldots,k$, and $C$ is not $0$. Therefore, $f_i(\alpha_i)\neq 0$ for $i=1,2,\ldots,k$. 
% Since $f_i$ are tradeoff functions, we have $f_i^{-1}(f_i(\alpha_i))=\alpha_i$ because $f_i$ are strictly decreasing when $f_i(\alpha_i)>0$. 

From Proposition \ref{prop:tradeoffnewproperty}, we know that $\frac{1}{C}\in\partial f_i(\alpha)|_{\alpha=f_i(\alpha_i)}$.
% , i.e., all $f_i$ have a common sub-differential value $\frac{1}{C}$ at $f_i(\alpha_i)$. 
Let $\alpha=\sum_{i=1}^k p_i \alpha_i$, and $g(\alpha)=\sum_{i=1}^k p_i f_i(\alpha_i)$. As $g$ is a tradeoff function, $g(\alpha)$ is strictly decreasing in $\{\alpha|g(\alpha) > 0\}$. Therefore, for any $g(\alpha)>0$, there is a one-to-one mapping between $\alpha$ and $g(\alpha)$, and $g^{-1}(g(\alpha))=\alpha$. 
Now we can view $g^{-1}$ as a mixture of $f_i$ at the values $f_i(\alpha_i)$: we let $\beta_i:=f_i(\alpha_i)$; since $f_i(f_i(\alpha_i))=\alpha_i$, we have $g^{-1}(\sum_{i=1}^k p_i \beta_i) = \sum_{i=1}^k p_i f_i(\beta_i)$, and there exists a constant $\frac{1}{C}$ such that $\frac{1}{C}\in\partial f_i(\alpha)|_{\alpha=\beta_i}$ for $i=1,2,\ldots,k$. Therefore, $g^{-1}=\mathrm{mix}(\underline{p}, \underline{f})$.
Since the mixture operation is well-defined, we know that $g=g^{-1}$.
\end{proof}

% \thmmixlow*

\begin{proof}[Proof of Theorem \ref{thm:mixlow}]
Consider the neighboring datasets $D_1, D_2$, and a rejection rule $\psi$ giving the type I error
$\alpha=\EE_{\cM(D_1)} \psi$, and type II error $\beta=\EE_{\cM(D_2)} (1-\psi).$
{We write} $\alpha_i=\EE_{\cM_i(D_1)} \psi,~ \beta_i=\EE_{\cM_i(D_2)} (1-\psi).$
Then $\alpha = \sum_{i=1}^k p_i\alpha_i, ~\beta = \sum_{i=1}^k p_i\beta_i.$ 

In order to obtain the lower bound of $\beta$ given $\alpha$, {which we denote as }%denoted by 
$f_{\min}(\alpha)$, we not only need the tradeoff between $\alpha_i$ and $\beta_i$, which is $f_i$, but also consider the tradeoff between $\alpha_i$ and $\alpha_j$ because of the constraint $\sum_{i=1}^k{p_i\alpha_i}=\alpha$. 
Therefore, we consider $f_{\min}(\alpha)=\min\{\sum_{i=1}^k{p_i\beta_i}~|~\beta_i\geq f_i(\alpha_i), \alpha=\sum_{i=1}^k{p_i\alpha_i}, \alpha_i\in[0,1]\}$,
which is a convex optimization problem since the objective function is linear, and the constraints $f_i$ are all convex (Proposition \ref{prop:convex-tradeoff}). Therefore, by Karush–Kuhn–Tucker theorem \citep{boyd2004convex}, let the Lagrangian function be
$L({\{\alpha_i,\beta_i,\mu_i,\nu_i,\kappa_i\}_{i=1}^k},\lambda) =\sum_{i=1}^k p_i \beta_i + \sum_{i=1}^k \big(\mu_i (f_i(\alpha_i)-\beta_i) + \nu_i(-\alpha_i) + \kappa_i(\alpha_i-1) \big) + \lambda\l(\alpha-\sum_{i=1}^k p_i \alpha_i \r)$,
and the minimum of $\sum_{i=1}^k{p_i\beta_i}$ is achieved at $\{\alpha_i,\beta_i\}_{i=1}^k$ {if and only if the following conditions are satisfied}
\begin{align*} 
    \text{Stationarity:}~ & p_i - \mu_i = 0,  ~ i=1,2,\ldots, k;\\
    & 0\in\mu_i \partial f_i(\alpha_i)-\nu_i+\kappa_i-\lambda p_i, ~ i=1,2,\cdots, k. \\
    \text{Primal feasibility:}~ & \sum_{i=1}^k \alpha_i=\alpha, ~\beta_i\geq f_i(\alpha_i), ~0\leq \alpha_i\leq 1, ~ i=1,2,\ldots, k;\\
    \text{Dual feasibility:}~ & \mu_i \geq 0, ~\nu_i \geq 0, ~\kappa_i \geq 0, ~ i=1,2,\ldots, k. \\
    \text{Complementary slackness:}~ & \sum_{i=1}^k \l(\mu_i (f_i(\alpha_i) - \beta_i) + \nu_i(-\alpha_i) + \kappa_i(\alpha_i-1) \r) = 0.
\end{align*}

Therefore, we have $\mu_i = p_i \neq 0$ from the stationarity condition, and $f_i(\alpha_i) = \beta_i$ from the complementary slackness condition. 
\begin{itemize}
    \item If $0<\alpha_i<1$, we have $\nu_i=\kappa_i=0$, therefore, $\lambda\in \partial f_i(\alpha_i)$; 
    \item If $\alpha_i=0$, we have $\kappa_i=0$ and $\partial f_i(\alpha_i)\ni\frac{\nu_i}{\mu_i}+\lambda\geq \lambda$, {therefore, for any $\alpha \in[0,1]$, we have $f_i(\alpha)\geq (\alpha-\alpha_i)(\frac{\nu_i}{\mu_i}+\lambda) + f_i(\alpha_i) \geq (\alpha-\alpha_i)\lambda + f_i(\alpha_i)$, i.e., $\lambda\in\partial f_i(\alpha_i)$};
    \item If $\alpha_i=1$, we have $\nu_i=0$ and $\partial f_i(\alpha_i)\ni-\frac{\kappa_i}{\mu_i}+\lambda\leq \lambda$; similarly, we have $\lambda\in\partial f_i(\alpha_i)$.  
\end{itemize}
As $f_i$ is a tradeoff functions, $\partial f_i(\alpha_i)$ is non-decreasing when $\alpha_i$ increases. We write $\partial f_i(\alpha_i) > \lambda$ if $a>\lambda ~\forall a\in\partial f_i(\alpha_i)$. For a given $\lambda$, 
\begin{itemize}
    \item {if there exist two constants $0<\alpha_{i,1}<\alpha_{i,2}<1$ such that $\partial f_i(\alpha_{i,1}) \leq \lambda \leq \partial f_i(\alpha_{i,2})$, we define $\alpha_i^{\lambda,\text{lower}}$ and $\alpha_i^{\lambda,\text{upper}}$ such that $\alpha_i\in[\alpha_i^{\lambda,\text{lower}}, \alpha_i^{\lambda,\text{upper}}]$ if and only if $\lambda \in \partial f_i(\alpha_i)$; 
    \item if $\partial f_i(\alpha_{i}) > \lambda$ for all $\alpha_{i}\in(0,1)$, we define $\alpha_i^{\lambda,\text{lower}}=\alpha_i^{\lambda,\text{upper}}=0$; intuitively, we do not want to have any of the type I error of $\alpha$ on the $f_i$ part since the corresponding type II error would be larger otherwise;
    \item if $\partial f_i(\alpha_{i}) < \lambda$ for all $\alpha_{i}\in(0,1)$, we define $\alpha_i^{\lambda,\text{lower}}=\alpha_i^{\lambda,\text{upper}}=1$; intuitively, we want to have as much of the type I error of $\alpha$ on the $f_i$ part as possible since the corresponding type II error would be larger otherwise.}
\end{itemize}
Define $\alpha^{\lambda,\text{lower}}=\sum_{i=1}^k p_i \alpha_i^{\lambda,\text{lower}}$, and $\alpha^{\lambda,\text{upper}}=\sum_{i=1}^k p_i \alpha_i^{\lambda,\text{upper}}$.
By definition, if $\lambda=-\infty$, we have $\alpha^{\lambda,\text{upper}} = \alpha^{\lambda,\text{lower}}=0$, and if $\lambda=+\infty$, we have $\alpha^{\lambda,\text{upper}} = \alpha^{\lambda,\text{lower}}=1$. For $\lambda\in(-\infty,+\infty)$, we have $[0,1]\subseteq\cup_{\lambda\in(-\infty,+\infty)}[\alpha^{\lambda,\text{lower}}, \alpha^{\lambda,\text{upper}}]$. Therefore, for any $\alpha\in[0,1]$, we can find $\lambda$ such that $\alpha\in[\alpha^{\lambda,\text{lower}},\alpha^{\lambda,\text{upper}}]$, and we can determine $\alpha_i$ by the $\lambda$.\footnote{If there are multiple choices of $\{\alpha_i\}_{i=1}^k$, all of them correspond to the same $\beta$.}

From the procedure above, we know $f_{\min}=\mathrm{mix}(\underline{p}, \underline{f})$, and $\cM$ satisfies $f_{\min}$-DP.
\end{proof}

We will use Lemma \ref{lem:cpformat} in the proof of Theorem \ref{thm:fdp_single_boot}.
\begin{lemma}[Equation (13) in \citep{dong2021gaussian}]\label{lem:cpformat}
For a symmetric tradeoff function $f$, define $f_p:= pf + (1-p)\Id$ for $0 \le p \le 1$, where $\Id(x) = 1-x$. Let $x^*$ be the unique fixed point of $f$, that is $f(x^*)= x^*$, we have
\begin{equation*}
		C_p(f)(x)=
			\left\{
			\begin{array}{ll}
			f_p(x),&x\in[0,x^*] \\
			x^*+f_p(x^*)-x, & x\in[x^*,f_p(x^*)]\\
			f_p^{-1}(x), &x\in[f_p(x^*),1].
			\end{array}
			\right.
	\end{equation*}
\end{lemma}
% \fdpsingleboot*
\begin{proof}[Proof of Theorem \ref{thm:fdp_single_boot}]
We are going to find a lower bound of $T_{\cM(\mathtt{boot}(D_1)), \cM(\mathtt{boot}(D_2))}$ uniformly for any neighboring datasets $D_1=(x_1,x_2,\cdots,x_n)$ and $D_2=(x_1',x_2,\cdots,x_n)$ (without loss of generality, we let $x_1$ be the different data point in $D_1$ and $D_2$). 
We use $\mathtt{boot}^i$ to denote the conditional bootstrap subsampling where $x_1$ or $x_1'$ are drawn for exactly $i$ times from $D_1$ or $D_2$. Note that both $\mathtt{boot}^i(D_1)$ and $\mathtt{boot}^i(D_2)$ are random variables for any $i=0,1,2,\cdots,n$. Furthermore, we define $\mathtt{boot}^>$ to denote the conditional bootstrap subsampling where $x_1$ or $x_1'$ is drawn for at least once from $D_1$ or $D_2$.

From Theorem \ref{thm:mixlow}, {we know $\mathrm{mix}(\{(q_i,g)\}_{i\in I}) = g$ for any $g$, $I$ and $q_i$. Therefore, we have} $T_{\cM(\mathtt{boot}^0(D_1)), \cM(\mathtt{boot}^0(D_2))}(\alpha)= f_0$ where $f_0(\alpha)=1-\alpha$ since $\mathtt{boot}^0(D_1)=\mathtt{boot}^0(D_2)$, and we also have $T_{\cM(\mathtt{boot}^k(D_1)), \cM(\mathtt{boot}^k(D_2))} \geq f_k$ since $\mathtt{boot}^k(D_1)$ and $\mathtt{boot}^k(D_2)$ are neighboring datasets with respect to group size $k$. {Now we consider $\cM\circ\mathtt{boot}^>$ as a mixture of $\cM\circ\mathtt{boot}^i$.} Using Theorem \ref{thm:mixlow} again, we have $T_{\cM(\mathtt{boot}^>(D_1)), \cM(\mathtt{boot}^>(D_2))} \geq f_>$ where $f_>:=\mathrm{mix}(\{(\frac{p_i}{1-p_0},f_i)\}_{i=1}^n)$ (here we use $\frac{p_i}{1-p_0}$ instead of $p_i$ because $\sum_{i=1}^n \frac{p_i}{1-p_0} = 1$).

{In order to obtain a better lower bound of $T_{\cM(\mathtt{boot}(D_1)), \cM(\mathtt{boot}(D_2))}$, we find the mixture of $f_>$ and $f_0$ with considering} the bootstrap resampling context because $\cM\circ \mathtt{boot}$ is a mixture of $\cM \circ \mathtt{boot}^>$ and $\cM \circ \mathtt{boot}^0$. 

Consider a rejection rule $\psi$. Let
\[ \begin{array}{ll}
\alpha:=\EE_{\cM(\mathtt{boot}(D_1))} \psi, &\beta:=\EE_{\cM(\mathtt{boot}(D_2))} (1-\psi), \\
\alpha_0:=\EE_{\cM(\mathtt{boot}^0(D_1))} \psi=\EE_{\cM(\mathtt{boot}^{0}(D_{2}))} \psi, &\beta_0:=\EE_{\cM(\mathtt{boot}^0(D_2))} (1-\psi)=1-\alpha_0,\\
\alpha_{>}:=\EE_{\cM(\mathtt{boot}^{>}(D_1))} \psi, &\beta_{>}:=\EE_{\cM(\mathtt{boot}^{>}(D_2))} (1-\psi).
\end{array}
\]
We prove that when calculating the mixture of $f_>$ and $f_0$, there are two additional constraints, $\beta_0\geq f_>(\alpha_>)$ and $\beta_>\geq f_>(\alpha_0)$, as the hypothesis testing between $\cM(\mathtt{boot}^0(D_1))$ and $\cM(\mathtt{boot}^{>}(D_1))$ is similar to the one between $\cM(\mathtt{boot}^>(D_2))$ and $\cM(\mathtt{boot}^{>}(D_1))$.

We consider the hypothesis testing between $\cM(\mathtt{boot}^0(D_1))$ and $\cM(\mathtt{boot}^{>}(D_1))$. Similar to the idea in \citep{balle2018privacy}, we consider replacing each $x_1$ in $\cM(\mathtt{boot}^{>}(D_1))$ with a data point independently and uniformly drawn from $(x_2,x_3,\ldots,x_n)$. We denote the distribution of $\mathtt{boot}^{>}(D_1)$ as $\omega_1$, the distribution of $\mathtt{boot}^0(D_1)$ as $\omega_0$, {the distribution of $\mathtt{boot}(D_1)$ as $\omega_{0\&1}$, and the replacement procedure as $\mathtt{replace}$}. Since \citet{balle2018privacy} did not provide proof for why this replacement procedure can transform $\omega_1$ to $\omega_0$, i.e., $\mathtt{replace}(\omega_1)=\omega_0$, we prove it below\footnote{The replacement procedure is defined deliberately: if it is defined to be replacing all $x_1$ with the same data point uniformly randomly drawn from $(x_2,x_3,\ldots,x_n)$, e.g., $x_1$ are all replaced by $x_2$, one can verify that this procedure will not transform $\omega_1$ to $\omega_0$ even for $n=3$.}. % and then build our proof based on it.

For one element in $\omega_1$, let {its} histogram be $h^> = (h_1,h_2,\cdots,h_n)$ where {$h_i$ is the number of occurrences of $x_i$ in this element,} $h_1\geq 1$, $h_i\geq 0$ for $i=2,3,\ldots,n$, $\sum_{i=1}^n h_i=n$. For one element in $\omega_0$, {we let its histogram be} $h^0 = (h_1=0,h_2,\cdots,h_n)$. Then 
\begin{align*}
P_{\omega_{0\& 1}}(H=h)&=\frac{1}{n^n}{n \choose h_1 ~ h_2 ~ \cdots ~ h_n}, \\ 
P_{\omega_{0}}(H=h^0)&=\frac{1}{(n-1)^n}{n \choose h_2 ~ h_3 ~ \cdots ~ h_n}, \\ 
P_{\omega_{1}}(H=h^>)&=\frac{1}{1-\l(1-\frac{1}{n}\r)^n} \frac{1}{n^n}{n \choose h_1 ~ h_2 ~ \cdots ~ h_n}=\frac{1}{n^n-(n-1)^n}{n \choose h_1 ~ h_2 ~ \cdots ~ h_n}.
\end{align*}
For the replacement procedure,
we replace the $h_1$ replicates of $x_1$ in $h^>$ with elements independently and uniformly drawing from $(x_2,x_3,\cdots,x_n)$, where the histogram of the replacement is $h'=(h_2',h_3',\cdots,h_n')$.
% , then 
% \begin{align*}
% P_{\mathtt{replace}(h_1)}(H'=h')\cdot P_{\omega_{1}}(H=h^>) &= \frac{{h_1 \choose h_2' ~ h_3' ~ \cdots ~ h_n'}}{(n-1)^{h_1}}  \frac{{n \choose h_1 ~ h_2 ~ \cdots ~ h_n}}{n^n-(n-1)^n} =  \frac{\frac{n!}{(h_2'! h_2!)(h_3'! h_3!)\cdots(h_n'! h_n!)}}{(n-1)^{h_1} (n^n-(n-1)^n)}.
% \end{align*}

% Then the final distribution of such replacement is 
% \begin{align*}
% &P_{\mathtt{replace}(\omega_{1})}(H=h^0) = \sum_{h',h_1}P_{\mathtt{replace}(h_1)}(H'=h')\cdot P_{\omega_{1}}(H^>=(h_1, h^0-h')) \\
% &= \sum_{h_1=1}^n \frac{1}{(n-1)^{h_1}(n^n-(n-1)^n)} \sum_{\{h' | \sum_{i=2}^n h_i' = h_1\}}  \frac{n!}{(h_2'! (h_2-h_2')!)\cdots(h_n'! (h_n-h_n')!)}
% \end{align*}

% It is not easy to show that $P_{\mathtt{replace}(\omega_1)}(H=h^0)=P_{\omega_{0}}(H=h^0)$ because of the requirement $\sum_{i=2}^n h_i' = h_1$ on $h'$. So, we use another way: 
Since for any element in $\omega_0$, the replacement does not change it, i.e., $\mathtt{replace}(\omega_0)=\omega_0$, we can perform the replacement on $\omega_{0\&1}$ and show $\mathtt{replace}(\omega_{0\&1}) = \omega_0$, then we also have 
$$\mathtt{replace}(\omega_{1})=\mathtt{replace}\l(\frac{\omega_{0\&1}-p_0\omega_0}{1-p_0}\r)=\frac{\l(\mathtt{replace}(\omega_{0\&1})-p_0\mathtt{replace}(\omega_0)\r)}{1-p_0}=\omega_0.$$
We prove $\mathtt{replace}(\omega_{0\&1}) = \omega_0$ below.

\begin{align*}
&P_{\mathtt{replace}(\omega_{0\&1})}(H^0=h^0) = \sum_{h'}P_{\mathtt{replace}(h_1)}(H'=h')\cdot P_{\omega_{0\&1}}(H=(h_1, h^0-h')) \\
&= \sum_{h_1=0}^n \frac{1}{(n-1)^{h_1}} \frac{1}{n^n}\sum_{\{h' | \sum_{i=2}^n h_i' = h_1, h_i'\leq h_i \forall i\}}  \frac{n!}{(h_2'! (h_2-h_2')!)\cdots(h_n'! (h_n-h_n')!)} \\
% &= \sum_{h_1=0}^n \frac{1}{(n-1)^{h_1}} \frac{1}{n^n}\sum_{\{h' | \sum_{i=2}^n h_i' = h_1, h_i'\leq h_i \forall i\}}  \frac{n!}{h_2!h_3!\cdots h_n!}\frac{h_2!}{h_2'! (h_2-h_2')!}\cdots\frac{h_n!}{h_n'! (h_n-h_n')!} \\
% &=  \frac{n!}{h_2!h_3!\cdots h_n!} \frac{1}{n^n} \sum_{h_1=0}^n \sum_{\{h' | \sum_{i=2}^n h_i' = h_1, h_i'\leq h_i \forall i\}}  \frac{h_2!}{(n-1)^{h_2'}h_2'! (h_2-h_2')!} \cdots\frac{h_n!}{(n-1)^{h_n'}h_n'! (h_n-h_n')!} \\
&=  \frac{n!}{h_2!h_3!\cdots h_n!} \frac{1}{n^n} \sum_{h_2'=0}^{h_2} \cdots \sum_{h_n'=0}^{h_n} \frac{h_2!}{(n-1)^{h_2'}h_2'! (h_2-h_2')!} \cdots\frac{h_n!}{(n-1)^{h_n'}h_n'! (h_n-h_n')!} \\
&=  \frac{n!}{h_2!h_3!\cdots h_n!} \frac{1}{n^n}  \l(\sum_{h_2'=0}^{h_2}  \frac{h_2!}{(n-1)^{h_2'}h_2'! (h_2-h_2')!}\r)
% \l(\sum_{h_3'}\frac{h_3!}{(n-1)^{h_3'}h_3'! (h_3-h_3')!}\r) 
\cdots\l(\sum_{h_n'0}^{h_n}\frac{h_n!}{(n-1)^{h_n'}h_n'! (h_n-h_n')!}\r) \\
&=  \frac{n!}{h_2!h_3!\cdots h_n!} \frac{1}{n^n}  \l(\frac{n}{n-1}\r)^{h_2}\l(\frac{n}{n-1}\r)^{h_3}\cdots \l(\frac{n}{n-1}\r)^{h_n} =  \frac{n!}{h_2!h_3!\cdots h_n!} \frac{1}{(n-1)^n}  \\
&\Rightarrow \mathtt{replace}(\omega_{0\&1}) = \omega_0.
\end{align*}
% Therefore, the replacement procedure indeed transforms $\omega_1$ into $\omega_0$.

Now we are ready to prove $\beta_0\geq f_>(\alpha_>)$ and $\beta_>\geq f_>(\alpha_0).$
From the result $\mathtt{replace}(\omega_{1}) = \omega_0$, if we consider the tradeoff function between the two distributions $\cM(\mathtt{boot}^{0}(D_1))$ and $\cM(\mathtt{boot}^{>}(D_{1}))$, we can break each of the two mixture distributions into parts following the replacement procedure so that there is a one-to-one mapping between the parts. When the number of occurrence of $x_1$ is $h_1$ {in the outcome of $\mathtt{boot}^>(D_1)$}, two parts in such a mapping pair have distance $h_1$ (between $\mathtt{boot}^{h_1}(D_{1})$ and $\mathtt{replace}(\mathtt{boot}^{h_1}(D_{1}))$), so the tradeoff function between $\cM(\mathtt{boot}^{h_1}(D_{1}))$ and $\cM(\mathtt{replace}(\mathtt{boot}^{h_1}(D_{1})))$ is $f_{h_1}$. We use Theorem \ref{thm:mixlow} to obtain the mixture of those $f_{h_1}$. Since $f_{h_1}$ only depends on $h_1$, its corresponding probability in $\mathtt{boot}^{>}(D_{1})$ is $\frac{p_{h_1}}{1-p_0}$. Therefore, the tradeoff function between $\cM(\mathtt{boot}^{>}(D_{1}))$ and $\cM(\mathtt{boot}^{0}(D_{1}))$ is $\mathrm{mix}(\{(\frac{p_i}{1-p_0},f_i)\}_{i=1}^n) = f_>$.
Recall that
\[ \begin{array}{ll}
\alpha_0:=\EE_{\cM(\mathtt{boot}^0(D_1))} \psi=\EE_{\cM(\mathtt{boot}^{0}(D_{2}))} \psi, &\beta_0:=\EE_{\cM(\mathtt{boot}^0(D_1))} (1-\psi)=1-\alpha_0,\\
\alpha_{>}:=\EE_{\cM(\mathtt{boot}^{>}(D_1))} \psi, &\beta_{>}:=\EE_{\cM(\mathtt{boot}^{>}(D_2))} (1-\psi),
\end{array} \]
we have $\beta_0\geq f_>(\alpha_>)$. Similarly, $\beta_>\geq f_>(\alpha_0).$

Now we have established the additional constraints for the mixture of $f_0$ and $f_>$. We are ready to derive the final mixture tradeoff function.
Notice that $\alpha = p_0 \alpha_0 + (1-p_0)\alpha_>$ and $\alpha_0,\alpha_>\in[0,1]$: For $\alpha=0$ and $\alpha=1$, we have $\alpha_0=\alpha_>=0$ and $\alpha_0=\alpha_>=1$ respectively. 

Now we consider the constrained optimization problem for $\alpha \in (0,1)$ where we replace $\alpha_>$ and $\beta_0$ with $\frac{\alpha-p_0\alpha_0}{1-p_0}$ and $1-\alpha_0$ respectively:
$f_{\min}(\alpha) = \min\{p_0(1-\alpha_0) + (1-p_0)\beta_> ~|~ \beta_>\leq 1, \alpha_0\geq 0,\beta_>\geq f_>(\alpha_0), \beta_>\geq f_>(\frac{\alpha-p_0\alpha_0}{1-p_0}), 1-\alpha_0 \geq f_>(\frac{\alpha-p_0\alpha_0}{1-p_0}) \}.$
% \begin{equation*}
% f_{\min}(\alpha) = \min_{\subalign{&\alpha_0,\beta_>, ~\mathrm{s.t.}~\\\beta_>&\leq 1, ~\alpha_0\geq 0,\\\beta_>&\geq f_>(\alpha_0),\\ \beta_>&\geq f_>(\frac{\alpha-p_0\alpha_0}{1-p_0}),\\ 1-\alpha_0 &\geq f_>(\frac{\alpha-p_0\alpha_0}{1-p_0}) }} p_0(1-\alpha_0) + (1-p_0)\beta_>,
% \end{equation*}
We ignore the constraint $\alpha_0\leq 1$ and $\beta_> \geq 0$ because they can be derived from $1-\alpha_0\geq f_>(\frac{\alpha-p_0\alpha_0}{1-p_0})$ and $\beta_>\geq f_>(\alpha_0)$ respectively. Since $f_>:\RR\mapsto\RR$ is a convex function, we use the Karush–Kuhn–Tucker theorem to solve to the convex optimization problem: let the Lagrangian function be
$L(\alpha_0,\beta_>,\mu_1,\mu_2,\mu_3, \mu_4,\mu_5) = p_0(1-\alpha_0)+(1-p_0)\beta_> + \mu_1(\beta_>-1)+\mu_2\l(-\alpha_0\r)+\mu_3(f_>(\alpha_0)-\beta_>)+\mu_4\l(f_>\l(\frac{\alpha-p_0\alpha_0}{1-p_0} \r)-\beta_>\r) + \mu_5\l(f_>\l(\frac{\alpha-p_0\alpha_0}{1-p_0}\r)-1+\alpha_0 \r)$,
and the minimum of $p_0(1-\alpha_0)+(1-p_0)\beta_>$ is achieved at $(\alpha_0,\beta_>)$ if and only if the following conditions are satisfied:\\
    \textbf{Stationarity:} $0\in -p_0- \frac{\mu_2p_0}{1-p_0}+\mu_3 \partial f_>(\alpha_0) - \frac{\mu_4p_0}{1-p_0}\partial f_>(\frac{\alpha-p_0\alpha_0}{1-p_0})+
    \mu_5 (-\frac{p_0}{1-p_0}\partial f_>(\frac{\alpha-p_0\alpha_0}{1-p_0})+1)$ and
    $0= (1-p_0)+\mu_1-\mu_3-\mu_4$.\\
    \textbf{Primal feasibility:} $1\geq\beta_>\geq f_>(\alpha_0),~\beta_>\geq f_>(\frac{\alpha-p_0\alpha_0}{1-p_0}),~1- f_>(\frac{\alpha-p_0\alpha_0}{1-p_0})\geq \alpha_0 \geq 0$. \\
    \textbf{Dual feasibility:} $\mu_i \geq 0, ~ i=1,2,3,4,5$. \\
    \textbf{Complementary slackness:} $\mu_1(\beta_>-1)+\mu_2(-\alpha_0)+\mu_3(f_>(\alpha_0)-\beta_>)+ \mu_4(f_>(\frac{\alpha-p_0\alpha_0}{1-p_0} )-\beta_>) + \mu_5 (f_>(\frac{\alpha-p_0\alpha_0}{1-p_0})-1+\alpha_0 )=0$.\\
% \begin{align*} 
%     &\text{Stationarity:}~  0\in -p_0- \frac{\mu_2p_0}{1-p_0}+\mu_3 \partial f_>(\alpha_0) - \frac{\mu_4p_0}{1-p_0}\partial f_>\l(\frac{\alpha-p_0\alpha_0}{1-p_0}\r)+\\
%     &\quad\quad\mu_5 \l(-\frac{p_0}{1-p_0}\partial f_>\l(\frac{\alpha-p_0\alpha_0}{1-p_0}\r)+1\r), \\
%     & 0= (1-p_0)+\mu_1-\mu_3-\mu_4.\\
%     &\text{Primal feasibility:}~ 1\geq\beta_>\geq f_>(\alpha_0),~\beta_>\geq f_>\l(\frac{\alpha-p_0\alpha_0}{1-p_0}\r),~1- f_>\l(\frac{\alpha-p_0\alpha_0}{1-p_0}\r)\geq \alpha_0 \geq 0. \\
%     &\text{Dual feasibility:}~  \mu_i \geq 0, ~ i=1,2,3,4,5. \\
%     &\text{Complementary slackness:}~  0=\mu_1(\beta_>-1)+\mu_2\l(-\alpha_0\r)+\mu_3(f_>(\alpha_0)-\beta_>)~+\\ 
%     &\quad\quad\mu_4\l(f_>\l(\frac{\alpha-p_0\alpha_0}{1-p_0} \r)-\beta_>\r) + \mu_5 \l(f_>\l(\frac{\alpha-p_0\alpha_0}{1-p_0}\r)-1+\alpha_0 \r).
% \end{align*} 
As our constrained optimization problem has both the convex objective function and convex constraints, we only need to find a solution of $\alpha_0,\beta_>$ satisfying the KKT conditions, and this solution will be a minimizer of the problem given fixed $\alpha$.

To simplify our analysis, first, we show that the solution satisfies $\alpha_0 \geq \alpha$. This is because when $\alpha_0 < \alpha$, there is always another choice, $(\alpha_0'=\alpha,\beta_>'=f_>(\alpha))$, achieving a lower value of the objective function: Notice that $\alpha = p_0 \alpha_0 + (1-p_0)\alpha_>$ where $p_0 \in (0,1)$. Therefore, if $\alpha_0 < \alpha$, we have $\alpha < \alpha_>$, and $\beta_>\geq f_>(\alpha_0)\geq f_>(\alpha)$ where the first inequality holds from previous analysis, and the second inequality holds due to $f_>$ being decreasing. Therefore,
\begin{align*}
    p_0 (1-\alpha_0) + (1-p_0) \beta_> &= p_0 (1-\alpha) + (1-p_0) f_>(\alpha) + p_0(\alpha-\alpha_0) + (1-p_0)(\beta_>-f_>(\alpha)) \\
    & > p_0 (1-\alpha) + (1-p_0) f_>(\alpha),
\end{align*}
which means that the choice $(\alpha_0'=\alpha,\beta_>'=f_>(\alpha))$ satisfies all the constraints and also achieves a value of the objective function lower than the choice $(\alpha_0,\beta_>)$. {This contradicts the optimality of $(\alpha_0, \beta_>)$}. Therefore, $\alpha_0 \geq \alpha > 0$.

Similarly, we show that we only need to consider $\beta_> < 1$. If $\beta_>=1$, since $\alpha_0 = \frac{\alpha - (1-p_0)\alpha_>}{p_0}\leq \frac{\alpha}{p_0}$, we have the objective function value $p_0(1-\alpha_0) + (1-p_0)\beta_>\geq 1-\alpha$. Now we consider another choice $(\alpha_0'=\alpha, \beta_>'=f_>(\alpha))$. We know that $\beta_>'=f_>(\alpha)\leq1-\alpha<1$, and the objective function value $p_0(1-\alpha) + (1-p_0)\beta_>'\leq p_0(1-\alpha) + (1-p_0)(1-\alpha)=1-\alpha$ which means that the new choice $(\alpha_0'=\alpha, \beta_>'=f_>(\alpha))$ is never worse than the choice of $(\alpha_0,\beta_>=1)$. Therefore, for $f_{\min}(\alpha)$ where $\alpha\in(0,1)$, we only need to consider $\beta_> < 1$.

{With $\alpha_0\geq \alpha>0,\beta_><1$,} from the complementary slackness, we know $\mu_1=\mu_2=0$. 

By Proposition \ref{prop:tradeoffnewproperty}, there is exactly one $\bar\alpha$ satisfying $\bar\alpha = f_>(\bar\alpha)$ and $-1\in \partial f_>(\bar\alpha)$.
We denote $C\in\partial f_>(\alpha)$, and $C'\in\partial f_>(\alpha)|_{\alpha=\frac{\alpha-p_0\alpha_0}{1-p_0}}$.

\begin{itemize}
    \item If $0<\alpha\leq\bar\alpha$, we verify that $(\alpha_0=\alpha, \beta_>=f_>(\alpha), \mu_3=(1-p_0)(p_0+p_0/C), \mu_4=(1-p_0)(1-p_0-p_0/C), \mu_5=0)$ satisfies the KKT conditions which means $(\alpha_0=\alpha, \beta_>=f_>(\alpha))$ is a minimizer. First, we see that $\frac{\alpha-p_0\alpha_0}{1-p_0}=\alpha$. 
    \begin{itemize}
        \item The primal feasibility is satisfied because $f_>(\alpha)\leq 1-\alpha$. 
        \item The complementary slackness is satisfied as $\mu_1=\mu_2=\mu_5=0$ and $f_>(\alpha)=\beta_>$. 
        \item Since $0<\alpha\leq\bar\alpha$, we know that $C\leq -1$. The dual feasibility conditions, $\mu_3\geq 0$ and $\mu_4>0$, hold because $p_0\in(0,1)$. 
        \item The stationarity conditions also hold when we plug in the value $C$ for $\partial f_>(\alpha)$.
    \end{itemize}
    \item If $\bar\alpha<\alpha\leq (1-p_0)\bar\alpha+p_0(1-\bar\alpha)$, we verify that $(\alpha_0=\frac{\alpha-(1-p_0)\bar\alpha}{p_0}, \beta_>=\bar\alpha, \mu_3=0, \mu_4=1-p_0, \mu_5=0)$ satisfies the KKT conditions. First, we see that $\frac{\alpha-p_0\alpha_0}{1-p_0}=\bar\alpha$.
    \begin{itemize}
        \item The primal feasibility is satisfied: since $\alpha>\bar\alpha$, we have {$\alpha_0>\bar\alpha$ and} $f_>(\alpha_0)\leq f_>(\bar\alpha)= \beta_>$; since $\alpha\leq(1-p_0)\bar\alpha+p_0(1-\bar\alpha)$, we have $\alpha_0\leq 1-\bar\alpha$, and $1-f_>(\bar\alpha)=1-\bar\alpha\geq \alpha_0$.
        \item The complementary slackness is satisfied as $\mu_1=\mu_2=\mu_3=\mu_5=0,f_>(\bar\alpha)=\beta_>$. 
        \item The dual feasibility condition, $\mu_4>0$, holds because $p_0\in(0,1)$. 
        \item The stationarity conditions also hold: We plug in the value $-1$ for $\partial f_>(\bar\alpha)$, then the right-hand {sides of both conditions are 0.}
    \end{itemize}
    \item If $(1-p_0)\bar\alpha+p_0(1-\bar\alpha)<\alpha<p_0+(1-p_0)f_>(0)$, we let $\alpha_0^*$ be the solution of $p_0\alpha_0 + (1-p_0) f_>(1-\alpha_0)=\alpha$. Since $f_>$ is {continuous and} not increasing, we know $g(\alpha_0)=p_0\alpha_0 + (1-p_0) f_>(1-\alpha_0)$ is {continuous and} strictly increasing with respect to $\alpha_0$. {We know $g(1)=p_0+(1-p_0)f_>(0)$ and $g(1-\bar\alpha)=(1-p_0)\bar\alpha+p_0(1-\bar\alpha)$. Therefore, $p_0\alpha_0 + (1-p_0) f_>(1-\alpha_0)=\alpha$ has only one solution which is $\alpha_0^*$ with $1>\alpha_0^* > 1-\bar\alpha$.} Now we verify that $(\alpha_0=\alpha_0^*, \beta_>=1-\alpha_0^*, \mu_3=0, \mu_4=1-p_0, \mu_5=\frac{p_0(1+C')}{1-\frac{p_0C'}{1-p_0}})$ satisfies the KKT conditions. First, we see that $\frac{\alpha-p_0\alpha_0}{1-p_0}=f_>(1-\alpha_0)> f_>(\bar\alpha)=\bar\alpha$ and $\frac{\alpha-p_0\alpha_0}{1-p_0}=f_>(1-\alpha_0)<f_>(0)$. Therefore, $0>C'\geq-1$.
    \begin{itemize}
        \item The primal feasibility is satisfied: since $\alpha_0\geq\alpha\geq\bar\alpha$, we have $1-\alpha_0<\bar\alpha$. {Therefore, by Proposition \ref{prop:tradeoffnewproperty}, since $f_>$ is convex and symmetric, there exists $C\leq -1$ such that $C\in\partial f_>(1-\alpha_0)$ and} $\beta_>=1-\alpha_0=f_>(f_>(1-\alpha_0))=f_>(\frac{\alpha-p_0\alpha_0}{1-p_0})$. We also have $\beta_>=1-\alpha_0\geq f_>(\alpha_0)$.
        \item The complementary slackness is also satisfied because $\mu_1=\mu_2=\mu_3=0$ and $\beta_>=1-\alpha_0=f_>(\frac{\alpha-p_0\alpha_0}{1-p_0})$. 
        \item The dual feasibility conditions hold as $p_0\in(0,1), C'\in[-1,0)$. 
        \item The stationarity conditions hold: We plug in the value $C'$ for $\partial f_>(\alpha)|_{\alpha=\frac{\alpha-p_0\alpha_0}{1-p_0}}$, then the right hand side is $0$.
    \end{itemize} 
    \item If $\alpha\geq p_0+(1-p_0)f_>(0)$, we let $(\alpha_0=1,\beta_>=0, \mu_5=p_0, \mu_3=(1-p_0)/2,\mu_4=(1-p_0)/2)$. Notice that $\frac{\alpha-p_0\alpha_0}{1-p_0} \geq f_>(0)$. Therefore, $f_>(\frac{\alpha-p_0\alpha_0}{1-p_0})=0$.
    \begin{itemize}
        \item The primal feasibility is satisfied.
        \item The complementary slackness is satisfied because $f_>(\alpha_0)-\beta_>=0,~ f_>(\frac{\alpha-p_0\alpha_0}{1-p_0} )-\beta_>=0,~ f_>\l(\frac{\alpha-p_0\alpha_0}{1-p_0}\r)-1+\alpha_0 =0$. 
        \item The dual feasibility conditions hold because $p_0\in(0,1)$. 
        \item The stationarity conditions hold: We use $0\in \partial f_>(\alpha_0)$ and $0\in\partial f_>(\alpha)|_{\alpha=\frac{\alpha-p_0\alpha_0}{1-p_0}}$ to plug in the first condition, and it holds because $-p_0+\mu_5=0$. The second condition holds because $ (1-p_0)-\mu_3-\mu_4=0$.
    \end{itemize} 
\end{itemize} 
Therefore, we have the tradeoff function $f_{\min}$ from $f_>$ as follows
  \begin{align*}
    f_{\min}(\alpha)=
    &\begin{cases}
    %   p_0+(1-p_0)f_>(0), & \text{if}~ \alpha=0 \\
      p_0+(1-p_0)f_>(0), & \text{if}~  \alpha=0 \\
      p_0(1-\alpha)+(1-p_0)f_>(\alpha), & \text{if}~ 0 < \alpha \leq \bar\alpha \\
      p_0-\alpha+2(1-p_0)\bar\alpha, & \text{if}~ \bar\alpha < \alpha \leq (1-p_0)\bar\alpha+p_0(1-\bar\alpha) \\
      1-\alpha_0^*, & \text{if}~ (1-p_0)\bar\alpha+p_0(1-\bar\alpha) < \alpha <   p_0+(1-p_0)f_>(0) \\
      0, & \text{if}~ \alpha\geq p_0+(1-p_0)f_>(0) \\
    \end{cases}
  \end{align*}
where {$\bar\alpha$ satisfies $f_>(\bar\alpha)=\bar\alpha$} and $\alpha_0^*$ is the only solution of $p_0\alpha_0 + (1-p_0) f_>(1-\alpha_0)=\alpha$ with respect to $\alpha_0$.

Now we verify the $f_{\min}$ above is the same as $C_{1-p_0}(f_>)$. {In Lemma \ref{lem:cpformat}, we replace $p$ and $f$ with $1-p_0$ and $f_>$ respectively:}
\begin{itemize}
    \item For $\alpha\in[0,\bar\alpha]$, we have $f_{\min}(\alpha)=(f_>)_{1-p_0}(\alpha)=C_{1-p_0}(f_>)(\alpha)$. 
    \item For $\alpha\in[\bar\alpha, (f_>)_{1-p_0}(\bar\alpha)]$, we have $f_{\min}(\alpha)=p_0-\alpha+2(1-p_0)\bar\alpha=\bar\alpha + ((1-p_0)\bar\alpha + p_0(1-\bar\alpha)) - \alpha = \bar\alpha + (f_>)_{1-p_0}(\bar\alpha) - \alpha=C_{1-p_0}(f_>)(\alpha)$. 
    \item For $\alpha\in[(f_>)_{1-p_0}(\bar\alpha), p_0+(1-p_0)f_>(0)]$, since $\alpha_0^*$ satisfies $p_0\alpha_0 + (1-p_0) f_>(1-\alpha_0)=\alpha$, we let $t=1-\alpha_0^*$, and we have $(f_>)_{1-p_0}(t)=(1-p_0)f_>(t)+p_0(1-t)=\alpha$. Since $f_>$ is non-increasing, we have that $(f_>)_{1-p_0}(t)$ is strictly decreasing because $p_0\in(0,1)$. Therefore, $(f_>)_{1-p_0}^{-1}(\alpha)=1-\alpha_0^*=f_{\min}(\alpha)$. 
    \item For $\alpha\in[p_0+(1-p_0)f_>(0),1]$, we know that $\alpha\geq (f_>)_{1-p_0}(0)$. Since $(f_>)_{1-p_0}$ is strictly decreasing, we have $(f_>)_{1-p_0}(t)\leq \alpha \forall t\in[0,1]$. Therefore, $(f_>)_{1-p_0}^{-1}(\alpha)=\inf\{t\in[0,1]:(f_>)_{1-p_0}(t)\leq \alpha\}=\inf\{t\in[0,1]\}=0$. We have $(f_>)_{1-p_0}^{-1}(\alpha)=0=f_{\min}(\alpha)$.
\end{itemize}
\end{proof}

\subsection{Derive \texorpdfstring{$(\ep,\delta)$}{(epsilon,delta)}-DP from $f$-DP for DP bootstrap}

In this section, we transform our Theorem \ref{thm:fdp_single_boot} to \citep[Theorem 10]{balle2018privacy}. 
% We use $\delta_{\cM,i}(\ep)$ to denote the group privacy guarantee of {a mechanism} $\cM$ with group size $i$, and restate the existing results below.

% \balle*

% \fdpsingleboot*

\begin{restatable}{proposition}{ourtoepdelta}\label{prop:our_to_epdelta}
    The $(\ep,\delta)$-DP result in Theorem \ref{thm:balle} can be derived from Theorem \ref{thm:fdp_single_boot}.
\end{restatable}

\begin{proof}
{ 
We use the primal-dual transformation (see Proposition \ref{prop:primal-dual})
to obtain the $(\ep',\delta')$-DP result by our $f_{\cM\circ\mathtt{boot}}$-DP result:
$\delta'(\ep')=1+f_{\cM\circ\mathtt{boot}}^*(-\ee^{\ep'})$.

Since $f$ is convex function when $f$ is a tradeoff function, we have $\alpha y - f(\alpha) \leq \alpha_0 y - f(\alpha_0)$ for any $\alpha$ and $(\alpha_0,y)$ satisfying $y\in \partial f(\alpha_0)$. Therefore, we have $f^*(y) = \sup_\alpha \alpha y-f(\alpha) = \alpha_0 y - f(\alpha_0)$ where $y\in \partial f(\alpha_0)$, and $1-\delta(\ep) = -f^*(-\ee^\ep) = \ee^\ep \alpha_\ep + f(\alpha_\ep)$ where $-\ee^\ep \in \partial f(\alpha_\ep)$. 

We let $f_> = \mathrm{mix}(\{(\frac{p_i}{1-p_0}, f_{i})\}_{i=1}^n)$, and our bootstrap privacy guarantee is $f_{\cM\circ\mathtt{boot}}=C_{1-p_0}(f_>)$. We let $\bar\alpha$ be the solution of $\alpha = f_>(\alpha)$. By Proposition \ref{prop:tradeoffnewproperty} and Lemma \ref{lem:cpformat}, we have $-1\in\partial f_>(\bar\alpha)$ and $-1\in\partial f_{\cM\circ\mathtt{boot}}(\bar\alpha)$. Since $-\ee^{\ep'} \leq -1$ for any $\ep' \geq 0$, we can find $\alpha_{\ep'}\leq \bar\alpha$ that $-\ee^{\ep'}\in \partial f_{\cM\circ\mathtt{boot}}(\alpha_{\ep'})$. By Lemma \ref{lem:cpformat}, we also have $-\ee^{\ep'}\in \partial ((f_>)_{1-p_0})(\alpha_{\ep'})$ and $f_{\cM\circ\mathtt{boot}}(\alpha_{\ep'}) = C_{1-p_0}(f_>)(\alpha_{\ep'}) = (f_>)_{1-p_0}(\alpha_{\ep'}) = (1-p_0) f_>(\alpha_{\ep'}) + p_0(1-\alpha_{\ep'})$, and 
$$
\begin{aligned}
\delta'(\ep')&=1+f_{\cM\circ\mathtt{boot}}^*(-\ee^{\ep'}) =-\ee^{\ep'} \alpha_{\ep'} - f_{\cM\circ\mathtt{boot}}(\alpha_{\ep'}) \\
&=-\ee^{\ep'} \alpha_{\ep'} - ((f_>)_{1-p_0})(\alpha_{\ep'}) = 1+((f_>)_{1-p_0})^*(-\ee^{\ep'})\\
\end{aligned}
$$

From the Supplement to \citep{dong2021gaussian} (Page 42), we know the following two equations, $\ep'=\log(1-p + p\ee^\ep), \delta'=p\big(1+f^*(-\ee^{\ep})\big)$ can be re-parameterized into $\delta'=1+f_p^*(-\ee^{\ep'})$ where $f_p = pf+(1-p)\Id$. We can re-parameterize $\delta'(\ep') = 1+((f_>)_{1-p_0})^*(-\ee^{\ep'})$ and $\ep'=\log(p_0 + (1-p_0)\ee^\ep)$ into $\delta'(\ep')=(1-p_0)\big(1+f_>^*(-\ee^{\ep})\big)$. Since $\ep'=\log(p_0 + (1-p_0)\ee^\ep)$ is also the relationship between $\ep'$ and $\ep$ in Theorem \ref{thm:balle}, in order to prove Theorem \ref{thm:balle}, we only need to show that $(1-p_0)\big(1+f_>^*(-\ee^{\ep})\big) = \sum_{i=1}^n p_i \delta_{\cM,i}(\ep)$.

As we have $f_> = \mathrm{mix}(\{(\frac{p_i}{1-p_0}, f_{i})\}_{i=1}^n)$, from the construction of the mixture tradeoff function, we let $C = -\ee^\ep$ and find $\{\alpha_i\}_{i=1}^n$ such that $C\in\partial f_i(\alpha_i)$, then we know that for $\alpha = \sum_{i=1}^n \frac{p_i}{1-p_0}\alpha_i$, we also have $C \in \partial f_> (\alpha)$: This is because for any $\alpha' = \sum_{i=1}^n \frac{p_i}{1-p_0}\alpha_i'$ and $f_> (\alpha') = \sum_{i=1}^n \frac{p_i}{1-p_0}f_i(\alpha_i')$, we have $f_> (\alpha') \geq f_>(\alpha) + C(\alpha' -\alpha)$ by the fact that $C\in\partial f_i(\alpha_i)$.
Therefore, we can find $\alpha_{\ep} = \sum_{i=1}^n \frac{p_i}{1-p_0}\alpha_{\ep,i}$ such that $-\ee^\ep\in\partial f_>(\alpha_{\ep})$ and $-\ee^\ep\in\partial f_i(\alpha_{\ep,i})$ for $i=1,2,\ldots,n$. Then we can prove $(1-p_0)\big(1+f_>^*(-\ee^{\ep})\big) = \sum_{i=1}^n p_i \delta_{\cM,i}(\ep)$ using the primal-dual transformation of $\delta_{\cM,i}(\ep)$ and the fact that 
$(1-p_0)\big(-\ee^\ep \alpha_\ep -f_>(\alpha_\ep)\big) = \sum_{i=1}^n p_i \big(-\ee^\ep \alpha_{\ep,i} -f_i(\alpha_{\ep,i})\big)$
due to $\alpha_{\ep} = \sum_{i=1}^n \frac{p_i}{1-p_0}\alpha_{\ep,i}$ and $f_>(\alpha_{\ep}) = \sum_{i=1}^n \frac{p_i}{1-p_0}f_i(\alpha_{\ep,i})$.
}

\end{proof}

\subsection{Proofs for Section \ref{sec:clt}}\label{sec:append_approx_eval_f_comp}
{The composition computation follows the general result in \citep{zheng2020sharp}, and we only need to prove the $q(x)$ in our numerical composition result which is derived from Theorem \ref{thm:fdp_single_boot}, Remark \ref{rmk:mix_tradeoff} and Lemma \ref{lem:cpformat}.
\begin{proposition}[numerical computation of composition: \citealp{zheng2020sharp}]\label{prop:numerical-composition-cite}
     Let $f_1=T(P_1, Q_1)$ and $f_2=T(P_2, Q_2)$, and we use $\delta_1$, $\delta_2$, and $\delta_{\otimes}$ to denote the dual view for $f_1$, $f_2$, and $f_{\otimes}:=f_1\otimes f_2$ correspondingly. If $P_i$ and $Q_i$ are distributions on $x\in\mathbb{R}$ with densities $p_i(x)$ and $q_i(x)$ for $i=1,2$ with respect to Lebesgue measure, then $\delta_{\otimes}(\ep) = \int_{\mathbb{R}} \delta_1(\ep - L_2(x)) q_2(x)\;\mathrm{d}x$ where $L_2(x):= \log(\frac{q_2(x)}{p_2(x)})$ and $\delta_1(\ep) = \int_{\mathbb{R}} \max(0, q_1(x) - \mathrm{e}^\ep p_1(x))\;\mathrm{d}x$.
\end{proposition}
\begin{proof}[Proof of Proposition \ref{prop:numerical-composition}]
In this proof, we replace $C$ with $\lambda$ to avoid confusion with the subsampling function $C_p$ in Proposition \ref{prop:fdp_subsamling}.

We know that $f_{\cM \circ \mathtt{boot}} = C_{1-p_0}(\mathrm{mix}(\underline{p}, \underline{f}))$, $\mathrm{mix}(\underline{p}, \underline{f})=(\sum_{i=1}^k (\frac{p_i}{1-p_0} f_i \circ (f_i')^{-1})) \circ (\sum_{i=1}^k \frac{p_i}{1-p_0} (f_i')^{-1})^{-1}$, and
$$
    C_p(f)(x)=
        \left\{
        \begin{array}{ll}
        f_p(x),&x\in[0,x^*] \\
        x^*+f_p(x^*)-x, & x\in[x^*,f_p(x^*)]\\
        f_p^{-1}(x), &x\in[f_p(x^*),1]
        \end{array}
        \right.
$$
where $x^*$ is the unique fixed point of $f$, $f_p:= pf + (1-p)\Id$ for $0 \le p \le 1$, and $\Id(x) = 1-x$. 

As $f_i$ is symmetric, let $x_i = (f_i')^{-1}(-1)$, then we have $f_i(x_i^*)=x_i^*$. Therefore, the choice $x^* = \sum_{i=1}^n  \frac{p_i}{1-p_0}  (f_i')^{-1}(-1)$ satisfies $\mathrm{mix}(\underline{p}, \underline{f})(x^*) = x^*$. We have $\mathrm{mix}(\underline{p}, \underline{f})_{1-p_0}(x^*) = (1-p_0) x^* + p_0 (1-x^*)$.

Let $y=1-x$ and $g(y) = \mathrm{mix}(\underline{p}, \underline{f})(y)$.

When $x \geq 1-x^*$, $y\leq x^*$, $q(x)=-f_{\cM \circ \mathtt{boot}}'(1-x)=-f_{\cM \circ \mathtt{boot}}'(y) = -((1-p_0)g + p_0 \Id)'(y) = -(1-p_0)g'(y) + p_0$. We let $\lambda = g'(y)$ be the slope, then $y=\sum_{i=1}^k \frac{p_i}{1-p_0} (f_i')^{-1}(\lambda)$ is the corresponding type I error. 

When $1- p_0-(1-2p_0)x^*\leq x < 1-x^*$, $x^* < y \leq p_0+(1-2p_0)x^* = \mathrm{mix}(\underline{p}, \underline{f})_{1-p_0}(x^*)$. $q(x)=-f_{\cM \circ \mathtt{boot}}'(1-x) = -f_{\cM \circ \mathtt{boot}}'(y) = 1$.

When $1- p_0-(1-2p_0)x^* > x \geq 0$, $y > \mathrm{mix}(\underline{p}, \underline{f})_{1-p_0}(x^*)$. $q(x) = -f_{\cM \circ \mathtt{boot}}'(1-x) = -f_{\cM \circ \mathtt{boot}}'(y) = -(((1-p_0)g + p_0 \Id)^{-1})'(y) = -\frac{1}{(1-p_0)g'(((1-p_0)g + p_0 \Id)^{-1}(y)) - p_0}$. We let $\lambda = 1/g'(((1-p_0)g + p_0 \Id)^{-1}(y))$. Then $q(x) = 1 / (p_0 - (1-p_0)/\lambda)$. And by Proposition \ref{prop:tradeoffnewproperty}.4, we have $(g')^{-1}(1/\lambda) = g((g')^{-1}(\lambda))$ since $g$ is a symmetric tradeoff function. Therefore, 
\[
\begin{aligned}
y &= ((1-p_0)g + p_0 \Id)((g')^{-1}(1/\lambda)) = (1-p_0) (g')^{-1}(\lambda) + p_0(1-(g')^{-1}(1/\lambda))\\
&=(1-p_0)\sum_{i=1}^n  \frac{p_i}{1-p_0}  (f_i')^{-1}(\lambda) + p_0\l(1-\sum_{i=1}^n  \frac{p_i}{1-p_0} (f_i')^{-1}(1/\lambda)\r).
\end{aligned}
\]
The computation of the privacy profile of $f_1 \otimes \cdots \otimes f_B$ is based on Proposition \ref{prop:numerical-composition-cite}.
\end{proof}
}
To prove our asymptotic composition result, we first restate the results in \citep{dong2021gaussian} for comparing the results of $f_{>}$ and $C_{1-p_0}(f_{>})$ which are used in our proof.
 
\newcommand{\kl}{\mathrm{kl}}
\newcommand{\bkl}{\boldsymbol{\kl}}
\newcommand{\bk}{\boldsymbol{\kappa_2}}
\newcommand{\bkk}{\boldsymbol{\kappa_3}}
\newcommand{\bkbar}{\boldsymbol{\bar{\kappa}_3}}

\begin{theorem}[Berry-Esseen CLT for the composition of $f$-DP: \citealp{dong2021gaussian}]\label{thm:clt}
Define $\kl(f) := -\int_0^1\log |f'(x)|\diff x$, 
 $\kappa_2(f):=\int_0^1\log^2 |f'(x)| \diff x$, 
 $\kappa_3(f):=\int_0^1\big|\log |f'(x)|\big|^3\diff x$, 
 $\bar{\kappa}_3(f):=\int_0^1\big|\log |f'(x)|+\kl(f)\big|^3\diff x$. 
Let $f_1,\ldots, f_n$ be symmetric tradeoff functions such that $\kappa_3(f_i) < \infty$ for all $1 \le i \le n$. Denote $\bkl:=(\kl(f_1), \ldots, \kl(f_n))$, $\bk:=({\kappa}_2(f_1), \ldots, {\kappa}_2(f_n))$, $\bkbar:=(\bar{\kappa}_3(f_1), \ldots, \bar{\kappa}_3(f_n))$,
$\mu:= \frac{2\|\bkl\|_1}{\sqrt{\|\bk\|_1 - \|\bkl\|_2^2}}$, 
$\gamma:=\frac{0.56\|\bkbar\|_1}{\big(\|\bk\|_1 - \|\bkl\|_2^2\big)^{3/2}}$
and assume $\gamma < \frac12$. Then, for all $\alpha\in[\gamma,1-\gamma]$, we have
\begin{equation}\label{eq:lower_upper}
G_\mu(\alpha+\gamma)-\gamma\leqslant f_{1}\otimes f_{2} \otimes \cdots \otimes f_{n}(\alpha)\leqslant G_\mu(\alpha-\gamma)+\gamma.
\end{equation}
Let $\{f_{ni}: 1\leqslant i \leqslant n\}_{n=1}^{\infty}$ be a triangular array of symmetric tradeoff functions and assume for some constants $K \ge 0$ and $s > 0$ that
$\lim_{n\to \infty}\sum_{i=1}^n \kl(f_{ni})= K$,
$\lim_{n\to \infty}\max_{1\leqslant i\leqslant n} \kl(f_{ni}) = 0$,
$\lim_{n\to \infty}\sum_{i=1}^n \kappa_2(f_{ni})= s^2$,
$\lim_{n\to \infty}\sum_{i=1}^n \kappa_3(f_{ni})= 0$.
Then, uniformly for all $\alpha \in [0,1]$, we have
	$$\lim_{n\to \infty} f_{n1}\otimes f_{n2} \otimes \cdots \otimes f_{nn} (\alpha) = G_{2K/s}(\alpha).$$
\end{theorem}

\begin{lemma}[Lemma F.2 and F.3 in the Supplement to \citealp{dong2021gaussian}]\label{lem:functionals2}
	Suppose $f$ is a symmetric tradeoff function with $f(0)=1$ and $x^*$ is its unique fixed point. 
 Then
	\begin{align*}
		\kl(f) &= \int_0^{x^*}\big(|f'(x)|-1\big)\log |f'(x)|\diff x \\
		% \lk(f^{-1}) &= \kl(f)\\
		\kappa_2(f)&= \int_0^{x^*}\big(|f'(x)|+1\big)\big(\log |f'(x)|\big)^2\diff x \\
		\bar{\kappa}_3(f)&=\int_0^{x^*}\l(\big|\log |f'(x)|+\kl(f)\big|^3+|f'(x)|\cdot\big|\log |f'(x)|-\kl(f)\big|^3\r)\diff x\\
		\kappa_3(f)&=\int_0^{x^*}\big(|f'(x)|+1\big)\big(\log |f'(x)|\big)^3\diff x.
	\end{align*}
	Let $g(x) = -f'(x)-1 = |f'(x)|-1$. Then
	\begin{align*}
		\kl(C_p(f)) &= p\int_0^{x^*}g(x)\log \big(1+pg(x)\big)\diff x \\
		\kappa_2(C_p(f))&= \int_0^{x^*}\big(2+pg(x)\big)\big[\log \big(1+pg(x)\big)\big]^2\diff x \\
		\kappa_3(C_p(f))&=\int_0^{x^*}\big(2+pg(x)\big)\big[\log \big(1+pg(x)\big)\big]^3\diff x.
	\end{align*}
\end{lemma}

For DP bootstrap with Gaussian mechanism, we can apply our Theorem \ref{thm:fdp_single_boot} as follows.
\begin{restatable}{corollary}{ourbootgauss}\label{cor:our_boot_gauss}
Let $\cM:\cX^n\rightarrow \cY$ satisfy $\mu$-GDP.
% Let $f$ be a tradeoff function and $f=G_{\mu}$.
% which is from Gaussian mechanism satisfying $\mu$-GDP. 
Then $\cM\circ\mathtt{boot}$ satisfies $f_{\mathtt{boot}}$-DP where $f_{\mathtt{boot}}=C_{1-p_0}(f_>)$, $f_>=\mathrm{mix}(\{\frac{p_i}{1-p_0},f_i\}_{i=1}^n)$, $p_i={n\choose i}(1/n)^i(1-1/n)^{n-i}$, $f_i=G_{i\mu}$.
\end{restatable}
Now we provide the exact representation of $f_>$ in Corollary \ref{cor:our_boot_gauss}.
\begin{restatable}{lemma}{ourbootgaussmixpart}\label{lem:our_boot_gauss_mixpart}
Let $f_>=\mathrm{mix}(\{\frac{p_i}{1-p_0},f_i\}_{i=1}^n)$ where $p_i={n\choose i}(1/n)^i(1-1/n)^{n-i}$, $f_i=G_{i\mu}$. Then $f_>$ is the tradeoff function between $\sum_{i=1}^n \frac{p_i}{1-p_0}\cN\l(-\frac{i^2\mu^2}{2}, i^2\mu^2\r)$ and $\sum_{i=1}^n \frac{p_i}{1-p_0}\cN\l(\frac{i^2\mu^2}{2}, i^2\mu^2\r)$.
\end{restatable}
\begin{proof}[Proof of Lemma \ref{lem:our_boot_gauss_mixpart}]
First, we know that $f_i(\alpha)=G_{i\mu}(\alpha)=\Phi(\Phi^{-1}(1-\alpha)-i\mu)$ where $\Phi$ is the CDF of the standard normal distribution.
We first obtain the subdifferential of $f_i$: $\frac{\diff f_i(\alpha_i)}{\alpha_i}=-\mathrm{exp}({-{i^2\mu^2}/{2}+i\mu\Phi^{-1}(1-\alpha_i)})$. We let the type I error and type II error in $f_i$ be $\alpha_i$ and $\beta_i$, and we have $f_i(\alpha_i)=\beta_i$. We let $\frac{\diff f_i(\alpha_i)}{\alpha_i}=C$. Then we have $\alpha_i = 1-\Phi\l({\log(-C)}/{(i\mu)}+{i\mu}/{2}\r),~\beta_i = \Phi\l({\log(-C)}/{(i\mu)}-{i\mu}/{2}\r)$.
This setting of $(\alpha_i,\beta_i)$ can also be achieved by using the rejection rule $\phi(x)=I_{x\geq \log(-C)}$ for the hypothesis testing between $H_0: x\sim \cN\l(-\frac{i^2\mu^2}{2}, i^2\mu^2\r)$ and $H_1: x\sim \cN\l(\frac{i^2\mu^2}{2}, i^2\mu^2\r)$.

For $f_>=\mathrm{mix}(\{\frac{p_i}{1-p_0},f_i\}_{i=1}^n)$, we let $f_>(\alpha)=\beta$ where $\alpha=\sum_{i=1}^n \frac{p_i}{1-p_0} \alpha_i$, $\beta=\sum_{i=1}^n \frac{p_i}{1-p_0} \beta_i$, $\beta_i=f_i(\alpha_i)$, and $\frac{\diff f_i(\alpha_i)}{\alpha_i}=C$. Then $\alpha$ and $\beta$ as type I error and type II error can be achieved by using the rejection rule $\phi(x)=I_{\{x\geq \log(-C)\}}$ for the hypothesis testing between $H_0: x\sim \sum_{i=1}^n \frac{p_i}{1-p_0} \cN\l(-\frac{i^2\mu^2}{2}, i^2\mu^2\r)$ and $H_1: x\sim \sum_{i=1}^n \frac{p_i}{1-p_0}\cN\l(\frac{i^2\mu^2}{2}, i^2\mu^2\r)$. 

{We let $h_1(x)=\sum_{i=1}^n \frac{p_i}{1-p_0}\frac{1}{i\mu_B}\phi(\frac{x}{i\mu_B}+\frac{i\mu_B}{2})$ and  $h_2(x)=\sum_{i=1}^n\frac{p_i}{1-p_0}\frac{1}{i\mu_B}\phi(\frac{x}{i\mu_B}-\frac{i\mu_B}{2})$ be the density functions of $\sum_{i=1}^n \frac{p_i}{1-p_0} \cN\l(-\frac{i^2\mu^2}{2}, i^2\mu^2\r)$ and $\sum_{i=1}^n \frac{p_i}{1-p_0}\cN\l(\frac{i^2\mu^2}{2}, i^2\mu^2\r)$ respectively. We have $\log\l(\frac{h_1(x)}{h_2(x)}\r)=\log(\ee^x)=x$.} Therefore, by the Neyman-Pearson Lemma, the most powerful rejection rule for the test between $H_0$ and $H_1$ is $\phi^*(x)=I_{\{x\geq \lambda\}}$ where $\lambda$ is a constant. This aligns with our previous rejection rule $\phi(x)=I_{\{x\geq \log(-C)\}}$. 
Therefore, $f_>$ is the tradeoff function between $\sum_{i=1}^n \frac{p_i}{1-p_0}\cN\l(-\frac{i^2\mu^2}{2}, i^2\mu^2\r)$ and $\sum_{i=1}^n \frac{p_i}{1-p_0}\cN\l(\frac{i^2\mu^2}{2}, i^2\mu^2\r)$.
\end{proof}
\begin{proof}[Proof of Theorem \ref{thm:boot_comp}]
Let $n$ be the sample size for the bootstrap resampling.
From the result of Corollary \ref{cor:our_boot_gauss} and Lemma \ref{lem:our_boot_gauss_mixpart}, the tradeoff function is $f_{Bi,\mathtt{boot}}=C_{1-p_0}(f_{>})$ where $p_0=(1-1/n)^n$ and $f_{>}$ is the tradeoff function between two Gaussian mixtures, $\sum_{i=1}^n \frac{p_i}{1-p_0}\cN(\mu=-\frac{i^2\mu_B^2}{2}, i^2\mu_B^2)$ vs $\sum_{i=1}^n \frac{p_i}{1-p_0}\cN(\mu=\frac{i^2\mu_B^2}{2}, i^2\mu_B^2)$ where $p_i={n\choose i}(1/n)^i(1-1/n)^{n-i}$.
{We let $h_1(x)=\sum_{i=1}^n \frac{p_i}{1-p_0}\frac{1}{i\mu_B}\phi(\frac{x}{i\mu_B}+\frac{i\mu_B}{2})$ and  $h_2(x)=\sum_{i=1}^n\frac{p_i}{1-p_0}\frac{1}{i\mu_B}\phi(\frac{x}{i\mu_B}-\frac{i\mu_B}{2})$. From the proof of Lemma \ref{lem:our_boot_gauss_mixpart}, we can parameterize the tradeoff function $f_>(\alpha_C) = \beta_C$ using $C\in(-\infty, \infty)$ with $\alpha_C=\int_{C}^{\infty} h_1(x) \diff x$ and $\beta_C=\int_{-\infty}^{C} h_2(x) \diff x$. We have $\frac{\diff \alpha_C}{\diff C}=-h_1(C)$, $\frac{\diff \beta_C}{\diff C}=h_2(C)$, $\frac{h_1(C)}{h_2(C)} = \ee^{-C}$, $\frac{h_1(0)}{h_2(0)} = 1$, $\alpha_{-\infty}=\beta_{\infty} = 1$, $\alpha_{\infty}=\beta_{-\infty} = 0$, $\alpha_0=\beta_0$, and $f'(\alpha_C)=\frac{{\diff \beta_C}/{\diff C}}{{\diff \alpha_C}/{\diff C}} = -\ee^{C}$. we can transform the result in Lemma \ref{lem:functionals2} to
	\begin{align*}
		&\kl(f_>) = \int_{\alpha_\infty}^{\alpha_0}(\ee^C-1) C \diff \alpha_C = \int_{0}^{\infty}(\ee^C-1) C h_1(C)\diff C, \\
		&\kappa_2(f_>)= \int_{\alpha_\infty}^{\alpha_0}(\ee^C+1)C^2\diff \alpha_C = \int_{0}^{\infty}(\ee^C+1) C^2 h_1(C)\diff C,  \\
		&\bar{\kappa}_3(f_>)=\int_{\alpha_\infty}^{\alpha_0}(|C+\kl(f)|^3+\ee^C|C-\kl(f)|^3)\diff \alpha_C = \int_{0}^{\infty} (|C+\kl(f)|^3+\ee^C|C-\kl(f)|^3)\diff C, \\
		&\kappa_3(f_>)=\int_{\alpha_\infty}^{\alpha_0} (\ee^C+1) C^3\diff \alpha_C = \int_{0}^{\infty}(\ee^C+1) C^3 h_1(C)\diff C. \\
		&\kl(C_{1-p_0}(f_>)) = (1-p_0)\int_{0}^{\infty}(\ee^C-1)\log \big(1+(1-p_0)(\ee^C-1)\big) h_1(C)\diff C, \\
		&\kappa_2(C_{1-p_0}(f_>))= \int_{0}^{\infty}\big(2+(1-p_0)(\ee^C-1)\big)\big[\log \big(1+(1-p_0)(\ee^C-1)\big)\big]^2 h_1(C)\diff C, \\
		&\kappa_3(C_{1-p_0}(f_>))=\int_{0}^{\infty}\big(2+(1-p_0)(\ee^C-1)\big)\big[\log \big(1+(1-p_0)(\ee^C-1)\big)\big]^3 h_1(C)\diff C.
	\end{align*}

Now we consider $\mu_B\rightarrow 0$. Since $h_1(x)=\sum_{i=1}^n \frac{p_i}{1-p_0}\frac{1}{i\mu_B}\phi(\frac{x}{i\mu_B}+\frac{i\mu_B}{2})$,
$$
\begin{aligned}
\kl(f_>) &= \sum_{i=1}^n \frac{p_i}{1-p_0} \int_{0}^{\infty}(\ee^x-1) \frac{x}{i\mu_B} \phi\l(\frac{x}{i\mu_B}+\frac{i\mu_B}{2}\r)\diff x \\
&= \sum_{i=1}^n \frac{p_i}{1-p_0} \int_{0}^{\infty}(\ee^x-1) \frac{x}{i\mu_B}  \frac{1}{\sqrt{2\pi}}\ee^{-\frac{(\frac{x}{i\mu_B}+\frac{i\mu_B}{2})^2}{2}}\diff x\\
&= \sum_{i=1}^n \frac{p_i}{1-p_0} \int_{0}^{\infty} (\ee^{i\mu_B y}-1) (i\mu_B y) \frac{1}{\sqrt{2\pi}}\ee^{-\frac{(y+\frac{i\mu_B}{2})^2}{2}}\diff y \quad (\mathrm{let~} y=\frac{x}{i\mu_B}).
\end{aligned}
$$
Similarly, we have
$$
\begin{aligned}
\kl(C_{1-p_0}(f_>)) &= \sum_{i=1}^n p_i \int_{0}^{\infty} (\ee^{i\mu_B y}-1) \log\big(1+(1-p_0)(\ee^{i\mu_B y}-1)\big) \frac{1}{\sqrt{2\pi}}\ee^{-\frac{(\frac{y}{i}+\frac{i\mu_B}{2})^2}{2}}\diff y.
\end{aligned}
$$
Now we have $\lim_{\mu_B\rightarrow0}\frac{\kl(C_{1-p_0}(f_>))}{\kl(f_>)} = (1-p_0)^2$ from the two following facts
$$
\begin{aligned}
\lim_{\mu_B\rightarrow0}\frac{\int_{0}^{\infty} \ee^{i\mu_B y} \log\big(1+(1-p_0)(\ee^{i\mu_B y}-1)\big) \frac{1}{\sqrt{2\pi}}\ee^{-\frac{(y+\frac{i\mu_B}{2})^2}{2}}\diff y}{\int_{0}^{\infty} \ee^{i\mu_B y} (i\mu_B y) \frac{1}{\sqrt{2\pi}}\ee^{-\frac{(y+\frac{i\mu_B}{2})^2}{2}}\diff y} &= 1-p_0, \\
\lim_{\mu_B\rightarrow0}\frac{\int_{0}^{\infty}  \log\big(1+(1-p_0)(\ee^{i\mu_B y}-1)\big) \frac{1}{\sqrt{2\pi}}\ee^{-\frac{(y+\frac{i\mu_B}{2})^2}{2}}\diff y}{\int_{0}^{\infty} (i\mu_B y) \frac{1}{\sqrt{2\pi}}\ee^{-\frac{(y+\frac{i\mu_B}{2})^2}{2}}\diff y} &= 1-p_0
\end{aligned}
$$
which are obtained using L'Hospital's rule and the Leibniz integral rule (i.e., the interchange of the integral and partial differential operators).

Similarly, $\lim_{\mu_B\rightarrow0}\frac{\kappa_2(C_{1-p_0}(f_>))}{\kappa_2(f_>)} = (1-p_0)^2$ and $\lim_{\mu_B\rightarrow0}\frac{\kappa_3(C_{1-p_0}(f_>))}{\kappa_3(f_>)} = (1-p_0)^3$.

We re-parameterize the definition of $\kl(f_>)$ in Theorem \ref{thm:clt} with $C$ and calculate it using $h_1(x)=\sum_{i=1}^n \frac{p_i}{1-p_0}\frac{1}{i\mu_B} \phi(\frac{x}{i\mu_B}+\frac{i\mu_B}{2})$:
\begin{align*}
	\kl(f_>) &= -\int_{\alpha_\infty}^{\alpha_{-\infty}} C \diff \alpha_C = \int_{-\infty}^{\infty} C h_1(C)\diff C 
=\sum_{i=1}^n \frac{p_i}{1-p_0}\int_{-\infty}^{\infty} \frac{-x}{i\mu_B} \phi(\frac{x}{i\mu_B}+\frac{i\mu_B}{2})\diff x \\
&=\sum_{i=1}^n \frac{p_i}{1-p_0}\int_{-\infty}^{\infty} (-i\mu_B y) \phi(y+\frac{i\mu_B}{2})\diff y = \sum_{i=1}^n \frac{p_i}{1-p_0} \frac{i^2\mu_B^2}{2} = \frac{(2-1/n)\mu_B^2}{2(1-(1-1/n)^n)}.
\end{align*}
Similarly, for any $n=1,2,\ldots$, and $\mu_B\rightarrow 0$, we have 
\begin{align*}
\kappa_2(f_{>})&=\sum_{i=1}^n \frac{p_i}{1-p_0}\int_{-\infty}^{\infty} \frac{x^2}{i\mu_B} \phi(\frac{x}{i\mu_B}+\frac{i\mu_B}{2})\diff x = \sum_{i=1}^n \frac{p_i}{1-p_0} \l(\frac{i^4\mu_B^4}{4} +i^2\mu_B^2\r) \\
&=\frac{(2-1/n)\mu_B^2}{1-(1-1/n)^n}+\Theta(\mu_B^4), \\
\kappa_3(f_{>})&=\sum_{i=1}^n \frac{p_i}{1-p_0}\int_{-\infty}^{\infty} \frac{|x|^3}{i\mu_B} \phi(\frac{x}{i\mu_B}+\frac{i\mu_B}{2})\diff x = \sum_{i=1}^n \frac{p_i}{1-p_0}(i\mu_B)^3 \int_{-\infty}^\infty \l|z-\frac{i\mu_B}{2}\r|^3 \phi(z) \diff z \\
&\leq \sum_{i=1}^n \frac{p_i(i\mu_B)^3}{1-p_0} \int_{-\infty}^\infty \l(|z|^3+3\frac{i\mu_B}{2}|z|^2+3\l(\frac{i\mu_B}{2}\r)|z|+\l(\frac{i\mu_B}{2}\r)^3\r) \phi(z) \diff z \in \Theta(\mu_B^3).
\end{align*}

By Theorem \ref{thm:clt}, we have $\lim_{B\rightarrow\infty} f_{\mathtt{boot}}^{\otimes B}=G_{2K/s}$  if 
\begin{align*}
\lim_{B\rightarrow\infty} \sum_{i=1}^B\kl(f_{Bi,\mathtt{boot}})&=K, \quad 
\lim_{B\rightarrow\infty} \max_{1\leq i\leq B} \kl(f_{Bi,\mathtt{boot}})=0, \\
\lim_{B\rightarrow\infty} \sum_{i=1}^B\kappa_2(f_{Bi,\mathtt{boot}})&=s^2, \quad
\lim_{B\rightarrow\infty} \sum_{i=1}^B\kappa_3(f_{Bi,\mathtt{boot}})=0.
\end{align*}
Since we assume that $\sqrt{B}\mu_B \rightarrow \mu'$, we have $\mu_B\in\Theta(B^{-\frac{1}{2}})$ and
\begin{align*}
\lim_{B\rightarrow\infty} \sum_{i=1}^B\kl(f_{Bi,\mathtt{boot}})&=\lim_{B\rightarrow\infty} B(1-p_0)^2 \frac{(2-1/n)\mu_B^2}{2(1-(1-1/n)^n)} = (1-p_0)^2 \frac{(2-1/n)(\mu')^2}{2(1-(1-1/n)^n)}, \\
 \lim_{B\rightarrow\infty}\max_{1\leq i\leq B} \kl(f_{Bi,\mathtt{boot}})&=\lim_{B\rightarrow\infty} (1-p_0)^2 \frac{(2-1/n)\mu_B^2}{2(1-(1-1/n)^n)}=0, \\
\lim_{B\rightarrow\infty} \sum_{i=1}^B\kappa_2(f_{Bi,\mathtt{boot}})&=\lim_{B\rightarrow\infty} B(1-p_0)^2 \l(\frac{(2-1/n)\mu_B^2}{1-(1-1/n)^n}+\Theta(\mu_B^4)\r) \\
&= (1-p_0)^2 \frac{(2-1/n)(\mu')^2}{1-(1-1/n)^n}, \\ 
\lim_{B\rightarrow\infty} \sum_{i=1}^B\kappa_3(f_{Bi,\mathtt{boot}})&=\lim_{B\rightarrow\infty} B(1-p_0)^3\Theta(\mu_B^3)=0.
\end{align*}
Therefore,
% \begin{align*}
    $\frac{2K}{s} = \frac{2(1-p_0)^2 \frac{(2-1/n)(\mu')^2}{2(1-(1-1/n)^n)}}{\sqrt{(1-p_0)^2 \frac{(2-1/n)(\mu')^2}{1-(1-1/n)^n}}} = \sqrt{\l(2-\frac{1}{n}\r)\l(1-\l(1-\frac{1}{n}\r)^n\r)} \mu' < (\sqrt{2-2/\ee}) \mu'.$
% \end{align*}
}

This bound holds for $n\in\{1,2,3,\cdots\}$ since the function is monotonically increasing with respect to $n$ and $\lim_{n\rightarrow\infty}\sqrt{\l(2-\frac{1}{n}\r)\l(1-\l(1-\frac{1}{n}\r)^n\r)} =\sqrt{2-2/\ee}=1.12438\ldots$ 
\end{proof}
% Our comment of Theorem \ref{thm:boot_comp}, the privacy guarantee of the composition result above can be seen as running $(\sqrt{2-2/\ee})\mu_B$-GDP mechanism on the original dataset (not on the bootstrap sample) for $B$ times, is based on the following Theorem. 
% \begin{theorem}[Composition of $\mu$-GDP; Corollary 2 \citep{dong2021gaussian}]
% The $n$-fold composition of $\mu_i$-GDP mechanisms is $\sqrt{\mu_1^2+\ldots+\mu_n^2}$-GDP.
% \end{theorem}

{\section{Proofs for Section \ref{sec:inference}}
In this section, we provide the proofs for the lemmas and theorems in Section \ref{sec:inference}.

\subsection{Proofs for Section \ref{sec:population_mean}}
To prove Lemma \ref{lem:point-estimates}, we use Lemma \ref{lem:delta-method} by Guillaume F.\footnote{\url{https://math.stackexchange.com/questions/2793833}} and include their proof here.
\begin{lemma}[Generalized delta method]\label{lem:delta-method}
    Let $\{X_n\}$ and $\{Y_n\}$ be sequences of random vectors taking values in $\mathbb{R}^k$ and let $f: \mathbb{R}^k \to \mathbb{R}^s$ be a transformation. Assume that 
    \begin{enumerate}
        \item[(A1)] $\{a_n\}$ is a real-valued sequence that $a_n > 0$, $a_n\to \infty$ and $a_n(\|X_n - Y_n\|_2) = O_p(1)$, 
        \item[(A2)] For any $\ep > 0$, there exists a set $S$ where $\limsup\limits_{n\to\infty} P(Y_n \notin S) < \ep$ and $\nabla f(y)$ is uniformly continuous for $y\in S$.
    \end{enumerate}
    Then we have $\|a_n[f(X_n) - f(Y_n)] - a_n \nabla f(Y_n^*)^\intercal (X_n - Y_n)\|_2 \overset{p}{\to} 0$ where $Y_n^* = Y_n$ if $f$ is differentiable at $Y_n$ and $Y_n^*$ is an arbitrary value in $S$ otherwise.
\end{lemma}
\begin{proof}[Proof of Lemma \ref{lem:delta-method}]    
    Define $A_n = a_n \|f(X_n) - f(Y_n) - \nabla f(Y_n^*)^\intercal (X_n - Y_n)\|_2$. 
    Let $R(h; y) = \frac{\|f(y+h) - f(y) - \nabla f(y)^\intercal h\|_2}{\|h\|_2}$.
    By the mean value theorem, we have $R(h; y) = \frac{\|[\int_0^1 \nabla f(y+th)\mathrm{d}t - \nabla f(y)]^\intercal h\|_2}{\|h\|_2} \leq \|\int_0^1 \nabla f(y+th)\mathrm{d}t - \nabla f(y)\|_2 \leq \max_{t\in[0,1]}\|\nabla f(y+th) - \nabla f(y)\|_2$.
    From (A2) where $\nabla f$ is uniformly continuous, we have $\lim_{\|h\|_2\to 0} (\sup_{y\in S} R(h; y)) = 0$. 
    
    Given $\ep > 0$, we can find $S$ and $N_y$ such that $P(Y_n \notin S) < \ep/2$ for $n > N_y$ by (A2). For $Y_n \in S$, we can write $A_n = a_n\|X_n - Y_n\|_2 R(X_n - Y_n; Y_n) $. 
    From (A1), we know $\|X_n - Y_n\|_2 \overset{p}{\to} 0$.
    Therefore, $a_n\|X_n - Y_n\|_2 (\sup_{y\in S} R(X_n - Y_n; y)) \overset{p}{\to} 0$. Therefore, for any $\delta > 0$ and $\ep > 0$, we can find $N_b$ such that $P(a_n\|X_n - Y_n\|_2 (\sup_{y\in S} R(X_n - Y_n; y)) > \delta) \leq \ep/2$ when $n > N_b$. Therefore, for any $\ep$ and $\delta$, we can let $N=\max(N_y, N_b)$, and for $n > N$, we have 
    $P(A_n > \delta) \leq P(Y_n \notin S) + P(a_n\|X_n - Y_n\|_2 (\sup_{y\in S} R(X_n - Y_n; y)) > \delta) \leq \ep.$
\end{proof}

\begin{proof}[Proof of Lemma \ref{lem:point-estimates}]
For part 1, we have 
    \begin{align}
        \mathbb{E}[\tilde m_{g,B} - \theta]^2 &= \mathbb{E}\l[\mathbb{E}\l[\l(\frac{1}{B} \sum_{j=1}^B (g(D_j) + \xi_j) - g(D) + g(D) - \theta\r)^2 \Bigg| D\r]\r] \\
        &= \mathbb{E}\l[\mathbb{E}\l[\l(\frac{1}{B} \sum_{j=1}^B (g(D_j) - g(D)) + \frac{1}{B}\sum_{j=1}^B \xi_j\r)^2 \Bigg| D\r]\r]  + \mathbb{E}[\mathbb{E}[(g(D) - \theta)^2) | D]] \label{eq:rate1} \\
        &= \mathbb{E}\l[\mathbb{E}\l[\l(\frac{1}{B} \sum_{j=1}^B (g(D_j) - g(D)) \r)^2+ \l(\frac{1}{B}\sum_{j=1}^B \xi_j\r)^2 \Bigg| D\r] \r] + \frac{1}{n}\mathbb{E}(x_i-\theta)^2 \label{eq:rate2}\\
        &= \mathbb{E}\l[\mathbb{E}\l[\frac{1}{B} \sum_{j=1}^B\l( g(D_j) - g(D) \r)^2 \Bigg| D\r] \r] + \frac{\sigma_e^2}{B}  + \frac{1}{n}\mathbb{E}(x_i-\theta)^2 \\
        &= \frac{1}{B}\mathbb{E}\l[\mathbb{E}\l[\l( g(D_j) - g(D) \r)^2 \big| D\r]\r] + \frac{\sigma_e^2}{B}  + \frac{1}{n}\mathbb{E}(x_i-\theta)^2 \label{eq:rate3}\\
        &= \frac{1}{B}\mathbb{E}\l[\frac{1}{n} \l(\frac{1}{n}\sum_{i=1}^n\l(x_i-\frac{1}{n}\sum_{i=1}^n x_i\r)^2\r)\r] + \frac{\sigma_e^2}{B}  + \frac{1}{n}\mathbb{E}(x_i-\theta)^2 \label{eq:rate4}\\
        &= \frac{1}{B}\l(\frac{n-1}{n^2}\mathbb{E}(x_i-\theta)^2\r) + \frac{2-2/e}{\mu^2 n^2} + \frac{1}{n}\mathbb{E}(x_i-\theta)^2
    \end{align}
    where Equation (\ref{eq:rate1}) is due to $\mathbb{E}\l[\frac{1}{B} \sum_{j=1}^B (g(D_j) - g(D)) + \frac{1}{B}\sum_{j=1}^B \xi_j \big| D\r] = 0$ and $(g(D) - \theta | D) = g(D) - \theta$, Equation (\ref{eq:rate2}) is from $\mathbb{E}[\xi_j]=0$ and the independence between $\xi_j$ and $D_j$, Equation (\ref{eq:rate3}) is because $D_1,\ldots,D_B$ are independent given $D$, and Equation (\ref{eq:rate4}) is by viewing $D$ as the population and $D_j$ as a sample from $D$ to obtain $\mathbb{E}[(g(D_j) - g(D))^2 | D ] = \frac{1}{n}\Var(D)$.

    For part 2, we use the techniques in the proofs of \citet[Theorem 2.2]{beran1997diagnosing} and \citet[Theorem 2]{awan2025one}.

    In the Lindeberg-Feller central limit theorem \citep[Proposition 2.27]{van2000asymptotic}, we let $Y_{n,i}=\frac{1}{\sqrt{n}}\sum_{k=1}^{d-1} I(x_i=s_k) e_k$ where $e_k$ is the unit vector with $k$-th entry being 1, then $\sqrt{n}(\hat\eta-\eta)=\sum_{i=1}^n (Y_{n,i}-\EE Y_{n,i}) \overset{d}{\to} \cN(0, \Sigma)$ since $\sum_{i=1}^n \Cov(Y_{n,i}) = \Sigma$ where $\Sigma = \diag(\eta) - \eta \eta^\intercal$ \citep[Example 12.7]{severini2005elements} and $\sum_{i=1}^{n} \EE\|Y_{n,i}\|^2 I(\|Y_{n,i}\| \geq \epsilon) \rightarrow 0$ for every $\epsilon > 0$ since $\|Y_{n,i}\|\leq \frac{1}{\sqrt{n}}$.
    
    \citet{beran1997diagnosing} showed that a sufficient condition for the bootstrap distribution converging to the sampling distribution is the local asymptotic equivariance (LAE) of $\hat\eta$, i.e., for the distribution of $\sqrt{n}(\hat\eta-\eta)$ denoted by $H_n(\eta)$, there exists $H(\eta)$ such that $H_n(\eta + \frac{h_n}{\sqrt{n}}) \rightarrow H(\eta)$ for every converging sequence $h_n\rightarrow h$ where $h_n, h \in \mathbb{R}^{d-1}$. For our setup of $\hat\eta$, we have $\|Y_{n,i}\|\leq \frac{1}{\sqrt{n}}$ and $\sum_{i=1}^n \Cov(Y_{n,i}) =  \diag(\eta+ \frac{h_n}{\sqrt{n}}) - (\eta+ \frac{h_n}{\sqrt{n}}) (\eta+ \frac{h_n}{\sqrt{n}})^\intercal \rightarrow \Sigma$. Therefore, the LAE condition holds with $H(\eta)=\cN(0, \Sigma)$ by the Lindeberg-Feller central limit theorem.
    
    Using the Skorohod Representation Theorem \citep[Theorem 1.6.3]{serfling2009approximation}, there exist $U_0\sim \mathrm{Uniform}([0,1])$ and measurable functions $A, A_n: [0,1] \rightarrow \mathbb{R}^d$ for $n\in\mathbb{N}$ such that $A_n(U_0) \sim H_n(\eta)$ for all $n\in\mathbb{N}$, $A(U_0) \sim \cN(0, \Sigma)$, and $A_n(U_0) \overset{a.s.}{\to} A(U_0)$. In other words, $A_n(U_0) \overset{d}{=} \sqrt{n}(\hat\eta - \eta)$ and $\hat\eta \overset{d}{=} \eta + \frac{1}{\sqrt{n}}A_n(U_0)$. Let $S=\{u_0\in[0,1] | A_n(u_0) \rightarrow A(u_0)\}$. We know $\mathbb{P}(U_0 \in S)=1$ and we can use $A_n(u_0)$ and $A(u_0)$ to replace $h_n$ and $h$ for $u_0 \in S$, respectively, in the LAE condition: For $j=1,\ldots, B$, let $\hat\eta^*_j$ be the empirical distribution parameter of the $j$-th bootstrap dataset $D_j$, and we have $\sqrt{n}(\hat\eta^*_j - \hat\eta) \big| (\hat\eta = \eta + A_n(u_0)/\sqrt{n}) \sim H_n(\eta + \frac{1}{\sqrt{n}}A_n(u_0)) \rightarrow \cN(0, \Sigma)$. Again, by the Skorohod Representation Theorem, we define $U_j \iid \mathrm{Uniform}([0,1])$ independent of $U_0$, and for all $u_0\in S$, there exist measurable functions $C, C_{n,A_n(u_0)}: [0,1] \rightarrow \mathbb{R}^d$ for all $n \in \mathbb{N}$ such that $C_{n,A_n(u_0)}(U_j) \sim H_n(\eta + \frac{1}{\sqrt{n}}A_n(u_0))$ for all $n \in \mathbb{N}$, $C(U_j) \sim \cN(0, \Sigma)$, and $C_{n,A_n(u_0)}(U_j) \overset{a.s.}{\to} C(U_j)$. 

    From the above construction of $A_n(U_0)$ and $C_{n,A_n(u_0)}(U_j)$, $j=1,\ldots, B$, we have the joint distribution
    \begin{equation}\label{eq:initial_normal}
    \begin{pmatrix}
\sqrt{n}(\hat\eta - \eta) \\
\sqrt{n}(\hat\eta^*_1 - \hat\eta) \\
\cdots \\
\sqrt{n}(\hat\eta^*_B - \hat\eta) 
\end{pmatrix} \overset{d}{=} 
    \begin{pmatrix}
A_n(U_0) \\
C_{n,A_n(U_0)}(U_1) \\
\cdots \\
C_{n,A_n(U_0)}(U_B)
\end{pmatrix} \overset{a.s.}{\to} 
    \begin{pmatrix}
A(U_0) \\
C(U_1) \\
\cdots \\
C(U_B)
\end{pmatrix} \sim
\cN\l(0,
    \begin{pmatrix}
\Sigma & & & \\
& \Sigma &  & \\
&  &  \ddots & \\
& & & \Sigma  
\end{pmatrix}\r)
    \end{equation}
    where we have $(B+1)$ sub-matrices $\Sigma \in \mathbb{R}^{(d-1)\times(d-1)}$ on the block diagonal of the covariance matrix for the multivariate normal distribution and zeros elsewhere.

    We define a function $\tilde T: \mathbb{R}^{(d-1)(B+1)} \to \mathbb{R}^{(B+1)}$ that \begin{equation*}
    \tilde \eta = \begin{pmatrix}
\hat\eta  \\
\hat\eta^*_1  \\
\cdots \\
\hat\eta^*_B  
\end{pmatrix}, \quad
    \tilde T(\tilde \eta) {=} 
    \begin{pmatrix}
T(\hat\eta)  \\
T(\hat\eta^*_1)  \\
\cdots \\
T(\hat\eta^*_B)  
\end{pmatrix}, \quad\mathrm{and}\quad \frac{\partial \tilde T}{\partial \tilde\eta}=
    \begin{pmatrix}
\frac{\partial T}{\partial \eta}\big|_{\eta=\hat\eta} & & & \\
& \frac{\partial T}{\partial \eta}\big|_{\eta=\hat\eta_1^*} &  & \\
&  &  \ddots & \\
& & & \frac{\partial T}{\partial \eta}\big|_{\eta=\hat\eta_B^*}
\end{pmatrix}.
    \end{equation*}
    where we have $(B+1)$ sub-matrices of $\frac{\partial T}{\partial \eta}\in \mathbb{R}^{(d-1)\times 1}$ on the block diagonal of $\frac{\partial \tilde T}{\partial \tilde\eta}$ and zeros elsewhere.
    
    In Lemma \ref{lem:delta-method}, we let $f=\tilde T$, $X_n=\tilde \eta$, $Y_n=(\eta, \hat\eta, \ldots, \hat\eta)^\intercal \in \mathbb{R}^{(B+1)}$, $a_n = \sqrt{n}$, and $Y=(\eta, \eta. \ldots, \eta)^\intercal \in \mathbb{R}^{(B+1)}$. We know that (A1) is satisfied by Equation (\ref{eq:initial_normal}). As $\|Y_n - Y\|_2 \overset{p}{\to} 0$ from Equation (\ref{eq:initial_normal}), for any $\ep > 0$, we can find $N_y > 0$ and $\delta > 0$ such that $P(Y_n \notin B) < \ep$ for $n > N_y$ where $B=\{Y_n | \|Y_n - Y\|_2 \leq \delta\}$. We also know that $\nabla f$ is uniformly continuous on $B$ because $\nabla f$ is continuous and $B$ is compact. Therefore, (A2) is also satisfied. Using Lemma \ref{lem:delta-method}, we have
    $\|a_n[f(X_n) - f(Y_n)] - a_n \nabla f(Y_n)^\intercal (X_n - Y_n)\|_2 \overset{p}{\to} 0$.
    Furthermore, since $\nabla f$ is continuous and $\|Y_n - Y\|_2 \overset{p}{\to} 0$, we have $\|\nabla f(Y_n) - \nabla f(Y)\|_2 \overset{p}{\to} 0$, and $\|a_n[f(X_n) - f(Y_n)] - a_n \nabla f(Y)^\intercal (X_n - Y_n)\|_2 \overset{p}{\to} 0$. In other words, we have
    $$\begin{pmatrix}
\sqrt{n}(T(\hat\eta) - T(\eta)) \\
\sqrt{n}(T(\hat\eta^*_1) - T(\hat\eta)) \\
\cdots \\
\sqrt{n}(T(\hat\eta^*_B) - T(\hat\eta)) 
\end{pmatrix}  \overset{d}{\to} 
\cN\l(0, ~(n\sigma_g^2) I_{(B+1)}\r).
    $$ 
    Since the $\xi_j$ are independent of $T(\hat\eta^*_j)$ and $T(\hat\eta)$, we have 
    $$\begin{pmatrix}
\sqrt{n}(T(\hat\eta) - T(\eta)) \\
\sqrt{n}(T(\hat\eta^*_1) - T(\hat\eta) + \xi_1) \\
\cdots \\
\sqrt{n}(T(\hat\eta^*_B) - T(\hat\eta) + \xi_B) 
\end{pmatrix}  \overset{d}{\to} 
\cN\l(0,
    \begin{pmatrix}
n\sigma_g^2 & & & \\
& (n\sigma_g^2 + n\sigma_e^2) &  & \\
&  &  \ddots & \\
& & & (n\sigma_g^2 + n\sigma_e^2)  
\end{pmatrix}\r)
    $$ 
    where the off-diagonal values in the covariance matrix are zeros.
    
    Using $\tilde m_{g,B} = \frac{1}{B} \sum_{j=1}^B \tilde g(D_j) = \frac{1}{B} \sum_{j=1}^B (T(\hat\eta^*_j) + \xi_j)$ and $\theta = T(\eta)$, we know
    $$
    \sqrt{n}(\tilde m_{g,B} - \theta) = \sqrt{n}(T(\hat\eta) - T(\eta)) + \frac{1}{B}\sum_{j=1}^B \sqrt{n}(T(\hat\eta^*_j) - T(\hat\eta) + \xi_j).
    $$
    Then, using the continuous mapping theorem and the joint distribution above, we know that $\sqrt{n}(\tilde m_{g,B} - \theta)$ converges to a normal random variable with variance $(n\sigma_g^2+\frac{n\sigma_g^2 + n\sigma_e^2}{B})$. Therefore, we have $\frac{\sqrt{n}(\tilde m_{g,B} - \theta)}{\sqrt{n\sigma_g^2+\frac{n}{B}(\sigma_g^2 + \sigma_e^2)}} \overset{d}{\to} \cN(0, 1)$.
    
    Similarly, from $\tilde s^2_{g,B} = \frac{1}{B-1} \sum_{j=1}^B (\tilde g(D_j)-\tilde m_{g,B})^2$, we  know
    $$
    n\tilde s^2_{g,B} = \frac{1}{B-1}\sum_{j=1}^B \l[\sqrt{n}(T(\hat\eta^*_j) - T(\hat\eta) + \xi_j) - \frac{1}{B}\sum_{j=1}^B \sqrt{n}(T(\hat\eta^*_j) - T(\hat\eta) + \xi_j)\r]^2.
    $$
    Using the continuous mapping theorem and the joint distribution above, $n\tilde s^2_{g,B}$ is the sample variance of $B$ random variables converging to $B$ i.i.d. normal random variables with mean 0 and variance $(n\sigma_g^2 + n\sigma_e^2)$. Therefore,
    we have $(B-1)\frac{n\tilde s^2_{g,B}}{n\sigma_g^2 + n\sigma_e^2} \overset{d}{\to} \chi^2_{B-1}.
    $
\end{proof}

\begin{proof}[Proof of Theorem \ref{thm:repro}]
    Let $r = \Phi^{-1}(1-\frac{\omega}{2}) \sqrt{\sigma_{g}^2 + \frac{\sigma_{g}^2 + \sigma_{e}^2}{B}}$. 
    
    We know 
    $\mathbb{P}\l(\hat{r}_n \geq r\r) = \mathbb{P}\l(\hat\sigma_{g}^2 \geq \sigma_g^2 \r) = \mathbb{P}\l(Y_n \geq \frac{c}{B-1}(\sigma_g^2 + \sigma_e^2) \r) \rightarrow (1-\alpha+\omega)$ and $\mathbb{P}\l(\theta \in \l[X_n - r,~ X_n + r\r]\r) \rightarrow (1 - \omega)$. Therefore,
    $$
    \begin{aligned}
    \lim_{n\rightarrow\infty}\mathbb{P}\l(\theta \in \l[X_n - \hat{r}_n,~ X_n + \hat{r}_n\r]\r) &\geq \l(\lim_{n\rightarrow\infty}\mathbb{P}\l(\theta \in \l[X_n - r,~ X_n + r\r]\r) \r) \l(\lim_{n\rightarrow\infty}\mathbb{P}\l(\hat{r}_n \geq r\r) \r)  \\
    & = (1-\omega)(1-\alpha+\omega) = 1-\alpha+\omega(\alpha-\omega) > 1-\alpha.    
    \end{aligned} 
    $$
\end{proof}

\begin{proof}[Proof of Proposition \ref{prop:repro-rate}]
    First, we note that the Cram\'er-Rao lower bound \citep[Theorem 3.3]{shao2003mathematical} indicates that $\Var(T(\hat\eta))\geq\sigma_g^2=\frac{1}{n}\l(\frac{\partial T}{\partial \eta}\r)^\intercal \Sigma \frac{\partial T}{\partial \eta}$ where $\Sigma = \diag(\eta) - \eta \eta^\intercal$ is equal to the inverse of the Fisher information matrix for the multinomial distribution.

    We will prove $\hat r_n \overset{a.s.}{\to} \Phi^{-1}(1-\frac{\alpha}{2})\sigma_g$.
    As $B = o(n^2)$, we have $\sigma_e^2 \rightarrow 0$ and $\sigma_g^2 \rightarrow 0$ when $n\rightarrow \infty$.
    By the Berry-Esseen theorem and the fact that the Chi-square distribution is the distribution of a sum of the squares of independent standard normal random variables, for $X\sim \chi^2_{B-1}$, we have $\l|\mathbb{P}\l(\frac{X - (B-1)}{\sqrt{2(B-1)}} \leq x\r) - \Phi(x)\r| = O(\frac{1}{\sqrt{B}})$ as $B\rightarrow \infty$. 
    
    Then we prove that $\frac{(B-1) - c_B}{B}$ converges to 0 where $c_B$ is the $(\alpha - \omega)$ quantile of the $\chi^2_{B-1}$ distribution. As $(\alpha-\omega) \rightarrow 0$, if there is $\Delta>0$ such that $\Delta <\frac{(B-1) - c_B}{B}$ for a subsequence of $\{c_B\}_{B=1}^{\infty}$, then $\alpha - \omega=\mathbb{P}(X \leq c_B) \leq \mathbb{P}(X < (B-1)-B\Delta ) = \mathbb{P}(\frac{X - (B-1)}{\sqrt{2(B-1)}} < -\frac{B\Delta}{\sqrt{2(B-1)}} ) \leq \Phi(-\frac{B\Delta}{\sqrt{2(B-1)}}) + O(\frac{1}{\sqrt{B}})$. From the Chernoff's bound, we know $\mathbb{P}(Z>c) \leq \ee^{-c^2/2}$ for $c>0$ and $Z \sim N(0,1)$, and we have $\Phi(-\frac{B\Delta}{\sqrt{2(B-1)}}) = o(\frac{1}{\sqrt{B}})$.
    % We follow Chernoff's bounding method to prove $\Phi(-\frac{B\Delta}{\sqrt{2(B-1)}}) = o(\frac{1}{\sqrt{B}})$. For $x>0$ and $Z \sim N(0,1)$, we know $\Phi(-x) = \mathbb{P}(Z > x) = \inf\limits_{t>0} \mathbb{P}(\ee^{tZ} > \ee^{tx}) \leq \inf\limits_{t>0} \frac{\EE(\ee^{tZ}) }{\ee^{tx}}$ by Markov's inequality, and $\inf\limits_{t>0} \frac{\EE(\ee^{tZ}) }{\ee^{tx}} = \inf\limits_{t>0} \ee^{\frac{1}{2}t^2 - tx} = \ee^{-\frac{1}{2}x^2}$. Therefore, we have $\Phi(-x) \leq \ee^{-\frac{1}{2}x^2}$ and $\Phi(-\frac{B\Delta}{\sqrt{2(B-1)}}) \leq \ee^{-\frac{B^2\Delta^2}{{4(B-1)}}} = o(\frac{1}{\sqrt{B}})$.
    This means $\alpha - \omega = O(B^{-\frac{1}{2}})$ which contradicts our assumption $B^{-\frac{1}{2}}=o(\alpha-\omega)$. Therefore, $\frac{(B-1) - c_B}{B}$ must converge to 0 which means $\frac{c_B}{B} \rightarrow 1$. 
    
    Let $\nu \sim \chi^2_{B-1}$. As $\chi^2_{B-1}$ is the sum of $(B-1)$ i.i.d. $\chi^2_1$ random variables, by the strong law of large numbers, we have $\frac{1}{B-1}\nu \overset{a.s.}{\to} 1$. As we have $(B-1)\frac{Y_n}{\sigma_g^2 + \sigma_e^2} \overset{d}{\to} \chi^2_{B-1}$, we also have $\l(\frac{Y_n - \sigma_e^2}{\sigma_g^2} - 1\r) \frac{\sigma_g^2}{\sigma_g^2 + \sigma_e^2} =\frac{Y_n }{\sigma_g^2 + \sigma_e^2} - 1 \overset{d}{\to} \l(\frac{1}{B-1}\nu -1\r) \overset{a.s.}{\to} 0$ as $B\rightarrow \infty$ and $n\rightarrow \infty$, which means $\frac{Y_n - \sigma_e^2}{\sigma_g^2} \overset{p}{\to} 1$.
    Using $\frac{c_B}{B} \rightarrow 1$, we have $\frac{\hat\sigma_g^2}{\sigma_g^2} \overset{p}{\to} 1$. 
    
    Then, using the definition of $\sigma_g^2$ and $\sigma_e^2$, we have $\frac{\sigma_{e}^2}{B\sigma_{g}^2}=\frac{1}{n}\frac{(2-2/e)}{\mu^2 \l(\frac{\partial T}{\partial \eta}\r)^\intercal \Sigma \frac{\partial T}{\partial \eta}} \to 0$. Therefore, $\frac{\sqrt{\hat\sigma_{g}^2 + \frac{1}{B}(\hat\sigma_{g}^2 + \sigma_{e}^2)}}{\sigma_g} = \sqrt{(1+ \frac{1}{B})\frac{\hat\sigma_{g}^2}{\sigma_{g}^2}  + \frac{\sigma_{e}^2}{B\sigma_{g}^2}}\overset{p}{\to} 1$. As we know $\omega \rightarrow \alpha$, we have $\hat r_n \overset{a.s.}{\to} \Phi^{-1}(1-\frac{\alpha}{2})\sigma_g$ when $n\rightarrow \infty$ and $B\rightarrow \infty$.
\end{proof}

}

\section{Privacy analysis of DP bootstrap with Gaussian mechanism}\label{sec:append_dpboostrap_gaussian}
In this section,
we first derive the curve for our lower bound on DP bootstrap with Gaussian mechanism in Figure \ref{fig:mix_gdp1}a; then
we show why the privacy analysis by \citet{brawner2018bootstrap} is incorrect;
Finally, we show why the PLDs by \citet{koskela2020computing} is incorrect.
\subsection{Our lower bound}
In this section, we explain how the curves in Figure \ref{fig:mix_gdp1}a were derived.

We evaluate our lower bound $C_{1-p_0}(f_>)$ based on Corollary \ref{cor:our_boot_gauss} and Lemma \ref{lem:our_boot_gauss_mixpart} and visualize it in Figure \ref{fig:mix_gdp1}a. From the proof of Lemma \ref{lem:our_boot_gauss_mixpart}, the tradeoff function $\beta=f_>(\alpha)$ is parameterized by $C\in(-\infty,0)$ where  
$\alpha=\sum_{i=1}^n \frac{p_i}{1-p_0} \alpha_i$, $\alpha_i = 1-\Phi\l({\log(-C)}/{(i\mu)}+{i\mu}/{2}\r)$, $\beta=\sum_{i=1}^n \frac{p_i}{1-p_0} \beta_i$, $\beta_i = \Phi\l({\log(-C)}/{(i\mu)}-{i\mu}/{2}\r)$;
Then we use Lemma \ref{lem:cpformat} to obtain $C_{1-p_0}(f_>)$.
% More details can be found in the proof of Lemma \ref{lem:our_boot_gauss_mixpart}.

In Figure \ref{fig:mix_gdp1}a, we also visualize the tradeoff functions for specific neighboring datasets.
{We let $D=(x_1,x_2, \ldots, x_n)$, and $\cM_G(D)=\frac{1}{n}\sum_{i=1}^n x_i + \xi$ where $\xi\sim\cN(0,1/(n\mu)^2)$.}
We study the tradeoff function between $\cM_G\circ\mathtt{boot}(D_1)$ and $\cM_G\circ\mathtt{boot}(D_2)$ where $D_1=(a,0,0,\ldots,0),\quad D_2=(a-1,0,0,\ldots,0), \quad |D_1|=|D_2|=n$.
The tradeoff function between $\cM_G(D_1)$ and $\cM_G(D_2)$ is $G_\mu$.
By the number of occurrences of $a$ and $1-a$, we have 
\begin{equation*}
\cM_G\circ\mathtt{boot}(D_1)\sim \sum_{i=0}^n p_i \cN\l(\frac{ia}{n}, \frac{1}{(n\mu)^2}\r),\quad \cM_G\circ\mathtt{boot}(D_2)\sim \sum_{i=0}^n p_i \cN\l(\frac{i(a-1)}{n}, \frac{1}{(n\mu)^2}\r),
\end{equation*}
{{where we are referring to the} distribution of the output of $\cM_G$ applied to one bootstrap sample which includes the randomness of $\cM_G$ as well as the randomness of $\mathtt{boot}$.}
Since the tradeoff function is {a lower bound for the curve between type I error and type II error from any} rejection rule, we consider the hypothesis tests $H_0:D=D_1,~H_1:D=D_2$ and the rejection rule $\{\cM_G\circ\mathtt{boot}(D) < C\}$ {(which need not be the optimal rejection rule)}. The type I and type II errors are $\alpha = \sum_{i=0}^n \Phi(Cn\mu-ia\mu)$ and $\beta = \sum_{i=0}^n \Phi(i(a-1)\mu-Cn\mu)$. In Figure \ref{fig:mix_gdp1}a, we show the curves of $(\alpha,\beta)$ for $\mu=1,~n=1000,~a\in{0,0.2,0.4,0.6,0.8,1}$. We can see that the curve for $a=0$ is not lower bounded by 1-GDP which shows that the bootstrap cannot be used for free with the same $f$-DP guarantee.

\subsection{Counterexample of the privacy analysis by \texorpdfstring{\citet{brawner2018bootstrap}}{Brawner and Honaker (2018)}}

In this section, we show an example disproving the statement `bootstrap for free' by \citet{brawner2018bootstrap}. Their result is under zCDP, a variant of DP.
% , based on R\'enyi divergences: %DP definitions:
\begin{definition}[Zero-Concentrated DP (zCDP): \citealp{bun2016concentrated}]
$\cM$ is $\rho$-zCDP if for all neighboring datasets $D_1$ and $D_2$ and all $\alpha\in(1,\infty)$, $D_\alpha(\cM(D_1) || \cM(D_2))\leq \rho\alpha$ where $D_\alpha(P||Q)$ is the $\alpha$-R\'enyi divergence, $D_\alpha(P||Q)=\frac{1}{\alpha-1}\log\l[\EE_{x\sim Q}\l(\frac{\diff P}{\diff Q}(x)\r)^\alpha\r]$, and $\frac{\diff P}{\diff Q}$ is the Radon-Nikodym derivative of $P$ with respect to $Q$. 
\end{definition} 

We consider $\cX = [0,1]$. For any dataset $D$ containing two individuals from $\cX$, i.e., $D\in \cX^2$, $D=(x_1,x_2)$, we define the Gaussian mechanism on the sample sum: $\cM(D):= x_1+x_2+\xi$ where $\xi\sim\cN(0,1)$.
{From the results in \citep{bun2016concentrated}}, we know that $\cM$ satisfies $\frac{1}{2}$-zCDP, i.e., $\rho=\frac{1}{2}$. 
Now we show that $\cM\circ \mathtt{boot}$ does not satisfy $\frac{1}{2}$-zCDP. 

To find a counterexample, we consider the neighboring datasets $D_1 = (1,0)$, $D_2=(0,0)$. Under this case, we have $\cM\circ\mathtt{boot}(D_1)\sim\frac{1}{4} \cN(0,1) + \frac{1}{2} \cN(1,1) + \frac{1}{4} \cN(2,1)$ and $\cM\circ\mathtt{boot}(D_2)\sim\cN(0,1)$. We check the $\alpha$-R\'enyi divergence when $\alpha=2$: 
\begin{align*}
D_{2}\l(\frac{1}{4} \cN(0,1) + \frac{1}{2} \cN(1,1) + \frac{1}{4} \cN(2,1)\Bigg\|\cN(0,1)\r) &= \log\l(\frac{1+4\ee+\ee^4+4+2+4\ee^2}{16}\r) \\
&\approx 1.85265\ldots > \rho \alpha=1.
\end{align*}

Therefore, there exists a mechanism $\cM$ satisfying $\frac{1}{2}$-zCDP such that $\cM \circ \mathtt{boot}$ does not satisfy $\frac{1}{2}$-zCDP which disproves the result in \citep{brawner2018bootstrap}.

{
\subsection{Counterexample of the privacy analysis by \texorpdfstring{\citet{koskela2020computing}}{Koskela et al. (2020)}}
In this section, we show that the privacy loss distribution (PLD) of the DP bootstrap with Gaussian mechanism cannot be the one shown in \citet{koskela2020computing}.

The privacy loss function and privacy loss distribution are defined as below.
 
\begin{definition}[Definition 3 and 4 \citep{koskela2020computing}] \label{def:plf}
Let $\mathcal{M} \, : \, \mathcal{X}^N \rightarrow \mathbb{R}$ be a randomised mechanism and
let $X \simeq Y$. Let $f_X(t)$ denote the density function of
$\mathcal{M}(X)$ and $f_Y(t)$ the density function of $\mathcal{M}(Y)$. Assume
$f_X(t)>0$ and $f_Y(t) > 0$ for all $t \in \mathbb{R}$.
 We define the privacy loss function
of $f_X$ over $f_Y$ as
$
\mathcal{L}_{X/Y}(t) = \log \frac{f_X(t)}{f_Y(t)}.
$
% \end{definition}

% The privacy loss distribution is defined as below.
% \begin{definition}[Definition 4 in \citep{koskela2020computing}] \label{def:pld}
{Using the notation of 
	%Let the assumptions of 
	Def.~\ref{def:plf},} % hold and
	 suppose that $\mathcal{L}_{X/Y} \, : \,\mathbb{R} \rightarrow D$, $D \subset \mathbb{R}$ is a continuously differentiable bijective function.
 The privacy loss distribution (PLD) of $\mathcal{M}(X)$ over $\mathcal{M}(Y)$ is defined to be a random variable
which has the density function 
\begin{equation*}
	\omega_{X/Y}(s)= \begin{cases}
		 f_X\big( \mathcal{L}_{X/Y}^{-1}(s)  \big)  \,  \frac{\diff \mathcal{L}_{X/Y}^{-1}(s)}{\diff s}, & s \in \mathcal{L}_{X/Y}(\mathbb{R}),\\
		 0, &\mathrm{else.}
	\end{cases}
\end{equation*}
\end{definition}
{Although \citet{koskela2020computing} did not mention in this definition of the privacy loss distribution, their density function requires $\mathcal{L}_{X/Y}(s)$ to be a monotonically increasing function with respect to $s$. When $\mathcal{L}_{X/Y}(s)$ monotonically decreases, we need to replace $\frac{\diff \mathcal{L}_{X/Y}^{-1}(s)}{\diff s}$ by $\l|\frac{\diff \mathcal{L}_{X/Y}^{-1}(s)}{\diff s}\r|$.}
Then they state the following result on privacy profile.
\begin{lemma}  [Lemma 5 \citep{koskela2020computing}]\label{lem:maxrepr}
Assume $(\ep,\infty) \subset \mathcal{L}_{X/Y}(\mathbb{R})$.
$\mathcal{M}$  is tightly $(\varepsilon,\delta)$-DP for 
$
\delta(\ep) = \max_{X\simeq Y} \{ \delta_{X/Y}(\ep)\}, \text{ where }
		\delta_{X/Y}(\ep) = \int_\ep^\infty  (1-\ee^{\ep - s}) \, \omega_{X/Y}(s)   \, \diff s. 
  $
\end{lemma}
Their result on subsampling with replacement is as below.
\begin{proposition}[Section 6.3 in \citep{koskela2020computing}]\label{prop:koskela}
Consider the sampling with replacement and the $\simeq$-neighbouring relation. 
The number of contributions of each member of the dataset is not limited.
{The size of the new sample is $m$.}
Then $\ell$, the number of times the differing sample $x'$ is in the batch, is binomially distributed,
i.e., $\ell \sim \mathrm{Binomial}(1/n,m)$, {where $n$ is the size of the original sample. Then the subsampled Gaussian mechanism, $\cM(D)=\sum_{x\in B} f(x) + \cN(0,\sigma^2 I_d)$ where $B$ is a uniformly randomly drawn subset (with replacement) of $D=\{x_1,\ldots,x_n\}$ and $\|f(x)\|_2 \leq 1$ for all $x\in X$, satisfies $(\ep,\delta(\ep))$-DP where $\delta(\ep) = \delta_{Y/X}(\ep) = \delta_{X/Y}(\ep)$ which is derived from}\\
$$
\begin{aligned}
		f_X(t) &=   \frac{1}{\sqrt{2 \pi \sigma^2}} \sum\limits_{\ell=0}^m \left( \frac{1}{n} \right)^\ell \left( 1- \frac{1}{n} \right)^{m-\ell} {m\choose\ell} \ee^{ \frac{-(t-\ell)^2}{2 \sigma^2}}, \\
 	f_Y(t) &=   \frac{1}{\sqrt{2 \pi \sigma^2}}  \sum\limits_{\ell=0}^m \left( \frac{1}{n} \right)^\ell \left( 1- \frac{1}{n} \right)^{m-\ell} {m\choose\ell} \ee^{ \frac{-(t+\ell)^2}{2 \sigma^2}}.
\end{aligned}
  $$
\end{proposition}

After restating the results by \citet{koskela2020computing}, we prove that the privacy profile $\delta(\ep)$ from the above theorem is not a valid privacy guarantee for some neighboring datasets. 

Consider {the}  Gaussian mechanism $\cM$ on $f(x)=x$ as $\cM(D)=\sum_{x\in D}f(x)+\xi$ where $\xi\sim\cN(0,1)$, $D=(x_1,x_2)$, $x_1,x_2\in\cX:=[-1,1]$. We consider two neighboring datasets,  $D_1=(1,1)$ and $D_2=(-1,1)$, and bootstrap, i.e., sampling with replacement when sample size remaining the same. We have the two output distributions as $\cM(\mathtt{boot}(D_1))\sim \cN(2,1)$ and $\cM(\mathtt{boot}(D_2))\sim \frac{1}{4}\cN(-2,1) + \frac{1}{2} \cN(0,1) +  \frac{1}{4}\cN(2,1)$ {with their density functions being}\\
% \awan{what are $X$ and $Y$ in this example? Is is $X=D_1$ and $Y=D_2$?}
$
		f_{D_1,1}(t) =  \tfrac{1}{\sqrt{2 \pi}} \l(\ee^{ \frac{-(t-2)^2}{2} }\r), 
		f_{D_2,1}(t) =  \tfrac{1}{\sqrt{2 \pi}} \l(\frac{1}{4} \ee^{ \frac{-(t-2)^2}{2} } +\frac{1}{2} \ee^{ \frac{-t^2}{2} } + \frac{1}{4} \ee^{ \frac{-(t+2)^2}{2} }\r).$

% Now we can compare the $(\ep,\delta)$ result from the distributions and the guarantee from Eq (\ref{eq:ko}) (we rewrite it below).\awan{Restate: However, Proposition C.5 gives the following privacy loss distributions:}
However, Proposition \ref{prop:koskela} gives the privacy loss distribution from the following {(which can also be derived from $D_1'=(1,0)$ and $D_2'=(-1,0)$)}:\\
$$
\begin{aligned}
		f_{D_1,2}(t) &=  \tfrac{1}{\sqrt{2 \pi}} \l(\frac{1}{4} \ee^{ \frac{-t^2}{2} } +\frac{1}{2} \ee^{ \frac{-(t-1)^2}{2} } + \frac{1}{4} \ee^{ \frac{-(t-2)^2}{2} }\r), \\
		f_{D_2,2}(t) &=  \tfrac{1}{\sqrt{2 \pi}} \l(\frac{1}{4} \ee^{ \frac{-t^2}{2} } +\frac{1}{2} \ee^{ \frac{-(t+1)^2}{2} } + \frac{1}{4} \ee^{ \frac{-(t+2)^2}{2} }\r). 
    \end{aligned}
$$

Now, we can calculate the $\delta$ when $\ep=1$. 
{
\begin{equation*}% \label{eq:deltarepr}
	\begin{aligned}
		\delta_{X/Y}(\ep) &= \int_\ep^\infty  (1-\ee^{\ep - s}) \, \omega_{X/Y}(s)   \, \diff s 
		= \int_\ep^\infty  (1-\ee^{\ep - s}) \, f_X\big( \mathcal{L}_{X/Y}^{-1}(s)  \big)  \,  \frac{\diff \mathcal{L}_{X/Y}^{-1}(s)}{\diff s}   \, \diff s \\
		&= \int_\ep^\infty  (1-\ee^{\ep - s}) \, f_X\big( \mathcal{L}_{X/Y}^{-1}(s)  \big)  \,  \diff \mathcal{L}_{X/Y}^{-1}(s) 
		  = \int_{\mathcal{L}_{X/Y}^{-1}(\ep)}^{\mathcal{L}_{X/Y}^{-1}(\infty)}  (f_X(z)-\ee^{\ep}f_Y(z)) \,  \diff z.
% 		&= \int_{\mathcal{L}_{X/Y}^{-1}(\ep)}^\infty  (F_X(z)-\ee^{\ep}F_Y(z)) \,  \diff z \\
	\end{aligned}
\end{equation*}
}

{For $f_{D_1,1}$ vs $f_{D_2,1}$, we know that $\frac{f_{D_2,1}}{f_{D_1,1}}$ is monotonically decreasing since the ratio between each of the three parts in $f_{D_2,1}$ and $f_{D_1,1}$ is decreasing. Therefore, $\log(\frac{f_{D_2,1}}{f_{D_1,1}})$ is decreasing, and we let it be $\mathcal{L}_{X/Y}$ with $f_X = f_{D_2,1}$ and $f_Y = f_{D_1,1}$. To calculate the $\delta$, we find $t_1$ such that $t_1=\mathcal{L}_{X/Y}^{-1}(1)$ and $t_2=\mathcal{L}_{X/Y}^{-1}(\infty)$, i.e., $\frac{f_{D_2,1}(t_1)}{f_{D_1,1}(t_1)}=\ee$ and $\frac{f_{D_2,1}(t_2)}{f_{D_1,1}(t_2)}=\infty$. Using \texttt{R}, we have $t_1=0.2228743\ldots$, $t_2=-\infty$. Then $\delta_1 = F_{D_2,1}(t_1) - \ee * F_{D_1,1}(t_1)$ where $F$ is the CDF corresponding to $f$. Using \texttt{R}, we have $\delta_1=0.4475773$.}

Similarly, we know $\frac{f_{D_2,2}}{f_{D_1,2}}$ is monotonically decreasing from the explanation in the Section 6.3 in \citep{koskela2020computing}. We find $t_3$ such that $\frac{f_{D_2,2}(t_3)}{f_{D_1,2}(t_3)}=\ee$, then $\delta_2 = F_{D_2,2}(t_3) - \ee * F_{D_1,2}(t_3)$. Using \texttt{R}, we have $t_3=-0.7830073$, and $\delta_2=0.369344$.   

{ 
Since we have $\delta_1 > \delta_2$, the claim by \citet{koskela2020computing} that ${f_{D_1,2}}$ and ${f_{D_2,2}}$ gives an upper bound of $\delta(\ep)$ is not true. %a valid one, i.e., the $(\ep,\delta)$-DP guarantee that it satisfies may not be satisfied for all the neighboring datasets.
}
}

\section{Details of the simulation and real-world experiment}\label{app:quantile_sensitivity}
In this section, we include NoisyVar (Algorithm \ref{alg:noisyvar}) used in our simulation, and we derive the sensitivity of the regularized ERM of quantile regression for using DP bootstrap with output perturbation in quantile regression. 
% In the end, we explain how we obtain the rule of thumb in Remark \ref{rmk:snr} by the results in the inference of the population mean and the slope parameters in logistic regression and quantile regression.

\begin{algorithm}[H]
	\caption{\texttt{NoisyVar} (DP CI for population mean \citep{du2020differentially})}\label{alg:noisyvar}
	\begin{algorithmic}[1]
		\STATE \textbf{Input} Dataset $D=\{x_1,x_2,\ldots,x_n\}\in\cX^n$, $x_i\in[0,1]$ for $i=1,\ldots,n$, $\mu>0$ for $\mu$-GDP, number of bootstrap samples $B$, confidence level $\alpha\in[0,1]$, simulation number $nsim$.
		\STATE $\cM_1(D)\leftarrow\frac{1}{n}\sum_{i=1}^n x_i + \xi_1,\quad \xi_1\sim\cN(0, \frac{2B}{(n\mu)^2})$.
		\STATE $\cM_2(D)\leftarrow\max\l\{0, \frac{1}{n-1}\sum_{i=1}^n \l(x_i - \frac{1}{n}\sum_{j=1}^n x_j \r)^2 + \xi_2\r\},\quad \xi_2\sim\cN(0, \frac{2B}{(n\mu)^2})$.
		\FOR {$i=1, 2, \ldots, nsim$}
    	\STATE $D' \leftarrow \{x_1', x_2', \ldots, x_n'\}$ where $\tilde{x}_i' \sim \cN(\cM_1(D), \cM_2(D))$ and $x_i' = \min(1, \max(0, \tilde{x}_i'))$.
            \STATE $\hat\theta_i \leftarrow \frac{1}{n}\sum_{i=1}^n x_i' + \xi_1,\quad \xi_1\sim\cN(0, \frac{2B}{(n\mu)^2})$
		\ENDFOR
        \STATE $MoE \leftarrow \frac{q(\hat\theta, 1-\frac{\alpha}{2}) - q(\hat\theta, \frac{\alpha}{2})}{2}$ where $\hat\theta = (\hat\theta_1, \ldots, \hat\theta_{nsim})$ and $q(x, \alpha)$ is a non-private empirical quantile function that outputs the $\alpha$ quantile of $x$.
		\STATE \textbf{Return} $(\cM_1(D) - MoE, \cM_1(D) + MoE)$ which satisfies $\mu$-GDP.
        
	\end{algorithmic}
\end{algorithm}

Let $D=\{(x_1,y_1),(x_2,y_2),\ldots,(x_n,y_n)\}$, $l(z_i) = (\tau-\mathbbm{1}(z_i\leq 0)) z_i$ and $z_i=y_i-x_i^\intercal \theta$.
The regularized empirical risk function $R_D^c(\theta) = \frac{1}{n}\sum_{i=1}^n l(z_i) + c\|\theta\|_2^2$ is $2c$-strongly convex. We consider a neighboring dataset to $D$: $D=\{(x_1',y_1'),(x_2,y_2),\ldots,(x_n,y_n)\}$, where $x_i=(1,w_i)$, $x_1'=(1,w_1')$, and $w_i\in[0,1]$, $w_1'\in[0,1]$. Let $\hat\theta_D^c = \mathrm{argmin}_\theta R_D^c(\theta)$. Denote $a=\mathbbm{1}(y_1\leq \theta^\intercal x_1), b=\mathbbm{1}(y_1'\leq \theta^\intercal x_1')$. By \citep[Lemma 7]{chaudhuri2011differentially}, $\|\hat\theta_D^c - \hat\theta_{D'}^c\| \leq \frac{1}{2c}\max_\theta \|\nabla (l(z_1) - l(z_1'))/n\| = \frac{1}{2nc}\|-\tau(x_1-x_1') + a x_1 - b x_1'\| = \frac{1}{2nc}\|{a-b \choose -\tau(w_1-w_1')+aw_1-bw_1'}\|$.\\
If $a=b=1$, $\|{a-b \choose -\tau(w_1-w_1')+aw_1-bw_1'}\|=\|{0\choose (1-\tau)(w_1-w_1')}\|\leq 2(1-\tau)$;\\
If $a=b=0$, $\|{a-b \choose -\tau(w_1-w_1')+aw_1-bw_1'}\|=\|{0\choose -\tau(w_1-w_1')}\|\leq 2\tau$; \\
If $a=0,b=1$, $\|{a-b \choose -\tau(w_1-w_1')+aw_1-bw_1'}\|=\|{-1\choose -\tau w_1 - (1-\tau)w_1'}\|\leq \sqrt{2}$; \\
If $a=1,b=0$, $\|{a-b \choose -\tau(w_1-w_1')+aw_1-bw_1'}\|=\|{1\choose (1-\tau)w_1 + \tau w_1'}\|\leq \sqrt{2}$.\\
Therefore,  $\|\hat\theta_D^c - \hat\theta_{D'}^c\| \leq \frac{1}{2nc} \max\{2\tau, 2(1-\tau), \sqrt{2}\}$ which is an upper bound of the sensitivity of the regularized ERM.

\bibliography{dp_boot}
\end{document}